\documentclass{article}

\usepackage[nonatbib,preprint]{neurips_2021}
\usepackage[utf8]{inputenc} % allow utf-8 input
\usepackage[T1]{fontenc}    % use 8-bit T1 fonts
\usepackage{hyperref}       % hyperlinks
\usepackage{url}            % simple URL typesetting
\usepackage{booktabs}       % professional-quality tables
\usepackage{amsfonts}       % blackboard math symbols
\usepackage{nicefrac}       % compact symbols for 1/2, etc.
\usepackage{microtype}      % microtypography
\usepackage{xcolor}         % colors
\usepackage{algorithm,algorithmic} 		%for algorithm's boxes
\usepackage{bbm} 						%to make 1 of the indicator function
\usepackage{amsmath,amsthm}
\usepackage{caption}
\usepackage{subcaption}
\usepackage{graphicx}
\usepackage{xspace}

\usepackage{tikz}
\usetikzlibrary{matrix,shapes,arrows,fit,positioning,calc,backgrounds}

\newtheorem{theorem}{Theorem}
\newtheorem{lemma}{Lemma}

\theoremstyle{definition}
\newtheorem{definition*}{Definition}
\newtheorem{remark}{Remark}

\newcommand{\MMM}{\mathcal{M}} %model
\newcommand{\SSS}{\mathcal{S}} %states
\newcommand{\AAA}{\mathcal{A}} %actions
\newcommand{\PPP}{\mathcal{P}} %transition prob
\newcommand{\KKK}{\mathcal{K}} %low frequently visited state set
\newcommand{\DDD}{\mathcal{D}} %state visitation distribution of the algorithm
\newcommand{\XXX}{\mathcal{X}} %insieme generico
\newcommand{\bbP}{\mathbb{P}}
\newcommand{\bbE}{\mathbb{E}}

\newcommand{\bbN}{\mathbb{N}}
\newcommand{\ind}[1]{\mathbbm{1}_{#1}} %indicator function
\newcommand{\norm}[1]{\left\lVert#1\right\rVert} %norm simbols
\newcommand{\REINFORCE}{\texttt{REINFORCE}\xspace}

\newcounter{mau}
\newcounter{mio}
\newcounter{mr}
\title{Curious Explorer: a provable exploration strategy in Policy Learning}

\author{%
	Marco Miani\\
	University of Pisa\\
	\texttt{marco.miani@sns.it}\\
	\And
	Maurizio Parton\\
	University of Chieti-Pescara\\
	\texttt{maurizio.parton@unich.it}\\
	\And
	Marco Romito\\
	University of Pisa\\
	\texttt{marco.romito@unipi.it}
}

\begin{document}

\maketitle

\begin{abstract}
    Having access to an exploring restart distribution (the so-called \emph{wide coverage} assumption) is critical with policy gradient methods. This is due to the fact that, while the objective function is insensitive to updates in unlikely states, the agent may still need improvements in those states in order to reach a nearly optimal payoff.
    For this reason, wide coverage is used in some form when analyzing theoretical properties of practical policy gradient methods. However, this assumption can be unfeasible in certain environments, for instance when learning is online, or when restarts are possible only from a fixed initial state. In these cases, classical policy gradient algorithms can have very poor convergence properties and sample efficiency.
    In this paper, we develop \emph{Curious Explorer}, a novel and simple iterative state space exploration strategy that can be used with any starting distribution $\rho$. Curious Explorer starts from $\rho$, then using intrinsic rewards assigned to the set of poorly visited states produces a sequence of policies, each one more exploratory than the previous one in an informed way, and finally outputs a restart model $\mu$ based on the state visitation distribution of the exploratory policies.
    Curious Explorer is provable, in the sense that we provide theoretical upper bounds on how often an optimal policy visits poorly visited states. These bounds can be used to prove PAC convergence and sample efficiency results when a PAC optimizer is plugged in Curious Explorer.
    %We plug Curious Explorer in REINFORCE, and prove 
	%\mau{an anytime sub-linear high probability regret bound as well as almost sure global convergence of the average regret with an asymptotically sub-linear rate.}{cosa è vero di questa frase? cmq è copiata, da cambiare}
    This allows to achieve global convergence and sample efficiency results \emph{without any coverage assumption} for \texttt{REINFORCE}, and potentially for any other policy gradient method ensuring PAC convergence \emph{with} wide coverage.
    Finally, we plug (the output of) Curious Explorer into \REINFORCE and \texttt{TRPO}, and show empirically that it can improve performance in MDPs with challenging exploration.
\end{abstract}

\section{Introduction}

The term Policy Gradient (PG) includes a family of RL methods that parameterize the policy with a $C^\infty$ parameterization and use an estimate of the gradient to maximize the expected long-term reward of the MDP. PG methods belong to a larger family of policy optimization methods which learn the optimal policy, as opposed to the so-called value-based methods which learn a value function first.

PG methods already appeared in germ at the dawn of modern RL in \cite{WitAOC,BSANAE}, and then in Sutton's PhD thesis \cite{SutTCA}, but it was Williams in \cite{WilSSG} who first presented a vanilla PG method called \texttt{REINFORCE}. After this very early appearance, PG methods were largely ignored in the 1990s, and focus remained to value-based methods \cite[historical remarks at page 337]{SuBRLI}. Only several years later PG methods reemerged and were elaborated in different declinations \cite[natural PG]{AmaNGW,KakNPG}, \cite[average reward setting]{BaBDGB}, \cite[PG Theorem with baseline]{SMSPGM}, \cite[online PG algorithm]{MaTSBO}.

%There are many reasons why PG methods may be preferred to value-based methods. First, PG methods can learn stochastic policies, whereas the classical epsilon-greedy exploration learns essentially deterministic policies. Second, they allow to use a prior knowledge of the shape of the optimal policy, and \mau{for this reason they also work well with continuous action spaces}{È vero che la ragione è questa?}, by learning statistics of the prior. Third, whenever the policy is simpler than the value function they have a better performance \cite{SAKWMD}. Fourth, they have better convergence properties. This is because actions change in a smooth fashion, and \mau{this allows to use the gradient information}{Secondo voi è vero che la ragione per cui possiamo usare metodi gradient-based è questa, più o meno?} to guide the optimization; in particular, being instances of gradient ascent, PG methods are guaranteed to converge to a stationary point even with function approximation, which is not true for value-based methods \cite{BaiRAR}.
There are several reasons to prefer PG methods to value-based methods \cite[Section 13.1]{SuBRLI}. This has led to a huge success of PG methods in deep RL \cite[high-dimensional action spaces]{SLHDPG}, \cite[AlphaGo, AlphaZero and MuZero]{SHMMGG,SSSMGG,SAHMAG}, \cite[Training highly directed latent variable models on large datasets]{MnGNVI}, \cite[Backpropagation for deep stochastic neural networks]{GLSMUB}, \cite[SeqGAN]{YZWSSG}, \cite[\texttt{TRPO} and derivatives]{SLATRP,SWDPPO,WHTTRG}, \cite[Atari games]{MKSHLC}, \cite[AlphaStar]{VBCGLS}.

Such success has been largely driven by experiments, without theoretical guarantees. The main issue is that the objective function is non-convex, and so the only guarantees of convergence are to local optima. Moreover, these guarantees come from stochastic approximation theory, and require the very strong assumption of an oracle providing the exact gradient, or at least the possibility of obtaining unbiased  estimates of the gradient.

%This is now rapidly changing. Indeed, the last few years have seen a proliferation of results answering fundamental questions for several important PG algorithms.
%When the parameterization contains the optimal policy $\pi_*$, does the algorithm converge to it? At what speed? When the parameterization is incomplete and does not contain $\pi_*$, does the algorithm converge to something close to $\pi_*$ with respect to some measure of error? Is it possible to quantify/limit the \emph{regret} (how much reward we are losing by not following $\pi_*$) as a function of the number of samples used (\emph{finite sample behaviour})? Finite sample behaviour is a particularly relevant topic for PG methods, which are notoriously not sample efficient due to the high variance in the gradient estimation (exceptions are very refined PG methods optimized in this sense like TRPO \cite{SLATRP}, PPO \cite{SWDPPO} and derivatives \cite{WHTTRG}).
This is now rapidly changing. Indeed, the last few years have seen a proliferation of results answering fundamental questions for several important PG algorithms, using one of the several theoretical performance measures \cite[Section 2]{DLBUPR} for Reinforcement Learning.
Bhandari and Russo in \cite{BhRGOG} describe MDP structural properties that gives convergence to the global optimum. This is true in particular for finite MDP with a complete parameterization. Liu et al.\ in \cite{LCYNTR} prove that a variant of \texttt{PPO} \cite{SWDPPO} equipped with overparameterized neural networks converges to the global optimum at a rate sublinear in the number of iterations, and Wang et al.\ in \cite{WCYNPG} prove a similar result for actor-critic schemes \cite{KoTACA} based on natural PG \cite{KakNPG}. Provable characterizations of computational, approximation, and sample size issues are given by Agarwal et al.\ in the incredibly comprehensive paper \cite{AKLTPG,AKLOAP}. Provable regret bounds are found in \cite{CHSRBP}, \cite{WJLMFR} and \cite{OrtRBR}.

All the papers listed above either assume \emph{wide coverage} or \emph{ergodicity}. These two hypothesis are as strong as widely accepted in the literature. Wide coverage assumes that a $\mu$-reset model through a soft $\mu$ be available, that is, trajectories can be restarted from states sampled by a distribution $\mu$ everywhere $>0$. Ergodicity assumes that for any stationary policy the induced Markov chain is irreducible and aperiodic. These assumptions completely eliminates the classic RL exploration-exploitation balancing problem, but rarely holds in real problems.

\section{Original contribution}\label{sec:contrib}

Our main contributions are an original exploration strategy that we call \emph{Curious Explorer}, and a theoretical analysis framework that does not require wide coverage or ergodicity. We only need the possibility to break the experience flow at will, letting the MDP restarting from its own starting state distribution $\rho$. We do not require any control on $\rho$.

Curious Explorer is built with a clear separation between exploration and optimization. It needs an optimizer \texttt{opt}, but any optimizer can be used during the exploration phase, see Algorithm \ref{alg:CE}.
After that, the output of Curious Explorer can be used to simulate a $\mu$-reset model with any chosen optimizer, where $\mu$ has more coverage than $\rho$.
This separation is something that, in our opinion, should be pursued when researching a novel exploration framework.
%\mau{If not, the exploration component of the optimizer could hinder the exploration efficiency of the novel framework}{davvero non so se mettere o no questa frase, aiutatemi}.

Curious Explorer is provable, in the sense that we provide theoretical upper bounds on how often an optimal policy visits poorly visited states, see Theorem \ref{thm:ceprovable}. These bounds can be used to prove PAC convergence when Curious Explorer, or better, the output of Curious Explorer, is plugged into a PAC optimizer.

%We provide an example of this statement. By pairing Curious Explorer with \texttt{REINFORCE} and using results from \cite{ZKOSER} we achieve global convergence and sample efficiency results \emph{without any coverage assumption} for \texttt{REINFORCE}, see Theorem \cite{thm:guarantee}, and potentially for any other policy gradient method ensuring PAC convergence \emph{with} wide coverage.
We provide an example of the above statement. Using results proved in \cite{ZKOSER} under the coverage hypothesis, we obtain an instance of \REINFORCE that is provably convergent without the coverage hypothesis. In fact, coverage is replaced by the MDP informed exploration given by Curious Explorer, see Theorems \ref{thm:reinforcece} and \ref{thm:cereinforce} (and the limitations of Theorem \ref{thm:cereinforce} in Section \ref{sec:limitations}).% To the best of our knowledge, a provable \texttt{REINFORCE} without any coverage hypothesis is an original contribution. \mau{However, as noted in Section \ref{sec:limitations}, the bound depends on the mismatch coefficient that can be\dots}{MM, da sistemare insieme}

In more details, Curious Explorer is an iterative procedure that takes a MDP with any starting state distribution $\rho$ (not necessarily soft) and any optimization algorithm \texttt{opt} as inputs, and outputs a sampling model \texttt{CE($\rho$,opt)}. At every step of the iteration, \texttt{CE($\rho$,opt)} improves exploration on previously poorly visited states as defined by \eqref{bad_states}. This improvement is \emph{informed by the MDP}, meaning that it is built on a mixture of $\rho$ and the exploration of previous step. See algorithm \ref{alg:CE} at page \pageref{alg:CE} for details.

The idea behind the theoretical analysis framework is simple. The state space $\SSS$ is splitted into poorly visited states $\KKK$ and its complement $\SSS-\KKK$. By simulating the coverage hypothesis on $\SSS-\KKK$, we can use any provable algorithm obtaining theoretical guarantees on $\SSS-\KKK$. Where the exploration goes wrong, that is, on poorly visited states $\KKK$, we can guarantee that those states cannot be visited more than a certain frequency by any policy, and therefore not even by the optimal policy for the original problem, see Theorem \ref{thm:ceprovable}. In short, as long as \texttt{opt} is provable nearly-optimal, the exploration framework is provable.
%This flexibility is in contrast with most other exploration strategies, where exploration and optimization are encoded together and an, and the output is the optimal policy instead of a sampling model. This flexibility of CE is in our opinion one of the main merits of CE.
%The theoretical analysis framework is applied to \texttt{REINFORCE}. Using the theoretical bounds proved \emph{under the coverage hypothesis} in \cite{ZKOSER}, we obtain a new provable algorithm \texttt{CE($\rho$,REINFORCE)} that doesn't need the coverage hypothesis and is provable for any $\rho$, see Theorem \ref{thm:cereinforce}. To the best of our knowledge, a provable \texttt{REINFORCE} without any coverage hypothesis is an original contribution. \mau{However, as noted in Section \ref{sec:limitations}, the bound depends on the mismatch coefficient that can be\dots}{MM, da sistemare insieme}

Finally, we empirically show that Curious Explorer improves exploration of \REINFORCE and \texttt{TRPO} in two challenging problems, the Consecutive Crossroad Traps
% \mio{known as DEEP20}{beh ma la deep ne proviamo diverse, 20 non è fissato}
\cite[Figure 2]{AKLTPG} and the Diabolical Combination Lock \cite[Figure 2]{AHKPPC}.

\section{Related Work}

The idea of promoting exploration by rewarding poorly visited states is not new. In the literature, these methods appear with various keywords, including \emph{intrinsic reward}, \emph{curiosity-driven exploration}, \emph{optimism in the face of uncertainty}. While the extrinsic reward is returned by the environment, an intrinsic reward is artificially assigned to poorly visited states. In Curious Explorer (hereafter, CE) the intrinsic reward is the main driver of exploration.

Upper Confidence Bound UCB methods are well-known methods that fall under the term intrinsic reward. They were originally developed for the exploration problem in Multi-Armed Bandits \cite{AueUCB} and then extended to MDPs \cite{JOANOR}. UCB methods are based on the estimation of a confidence interval for the extrinsic reward, and the upper bound of this interval is used to build an intrinsic reward. UCB methods are part of the larger family of Interval Estimation IE methods, whose precursor \cite{LaiATA} estimates the distribution model with Bayesian methods starting from Gaussian priors.
%The amount of scientific papers on methods described by the above terms is huge, but most of them are not comparable with CE, for instance because they use function approximation, or because contain only empirical results, and so on. Therefore, we choose papers with provable results and complete parametrization, such as CE.
%ELENCO di count-based methods: (con risultati teorici!) %chiaramente nell'articolo finale non vanno scritti a elenco così, adesso è solo per averli più chiari
%\begin{itemize}
%\item Interval Estimation (IE) --> Prima (credo) introduzione del metodo, però solo in multi-armed bandit e stimando solo distribuzioni gaussiane (paper 1987 - Adaptive treatment allocation and the multi-armed bandit problem)
%\item interval estimation (libro 1998 NON LETTO - Learning in embedded systems)

Model-based Interval Estimation MBIE \cite{StLTAM} is a method that uses statistics to estimate not only the reward, but also the distribution model of the MDP. Similarly to UCB, here the idea is to use the upper confidence bound to modify the estimated rewards, and additionally the transition probabilities, and backpropagate these updates through Bellman equations. This algorithm provides PAC-MDP bounds in the discounted setting. CE also uses the discount setting, but CE is model-free while MBIE is model-based.
% e usate per scegliere--> Estensione dei IE a MDP (paper 2005 - A Theoretical Analysis of Model-Based Interval Estimation)

MBIE with Exploratory Bonus MBIE-EB \cite{StLAMB} is a simplification of MBIE, with an intrinsic reward based on the number of visits: $1/\sqrt{n(s,a)}$. While MBIE-EB is model based as MBIE, and in this differs from CE, the count-based intrinsic reward here is related to how much $(s,a)$ has been visited. This is related to our notion of poorly visited states as states with small state visitation probability, see \eqref{bad_states}.

Bayesian Exploration Bonus BEB \cite{KoNNBE} uses a Bayesian approach, with count-based intrinsic reward $1/(1+n(s,a))$. Interestingly, in this paper it is shown that if the intrisic reward decays faster than $1/\sqrt{n(s,a)}$, then the method is not PAC. So in particular they show that neither BEB nor true Bayesian methods can be PAC. CE is not affected by this result, because its intrinsic reward is constant.

Explicit Explore or Exploit $\text{E}^3$ \cite{KeSNOR} is the first provably near-optimal polynomial time algorithm for sample complexity that does not assume a \emph{$\rho$-reset model}, that is, the ability to restart trajectories from the MDP starting state distribution at will. It is also the first to introduce the distinction between known and unknown states, which allows to use distinct periods of pure exploration and pure exploitation and to quantify at each step the relationship between exploration and exploitation. In CE, the notion of poorly visited states \eqref{bad_states} can be seen as a derivation of this idea, as well as the clear separation between exploration and improvement. Furthermore, CE requires a $\rho$-reset model, as opposed to $E^3$.
%However, it should be mentioned that $\text{E}^3$ is not model-free, which severely limits its applicability. --> non più vero, in lavori successivi è stato esteso anche al model-free.

\texttt{R-MAX} \cite{BrTRGP} is very close to CE. It is a generalization and simplification of $\text{E}^3$, with several very important achievements. For instance, in \texttt{R-MAX} the bias introduced by optimism under uncertainty is theoretically justified for the first time. In short, \texttt{R-MAX} changes the reward of poorly visited states to the best possible reward (hence the name), \emph{preserving the reward for the other states}. Similarly, CE assigns the best possible reward to poorly visited states, \emph{but changes to 0 every other reward}, making the exploration more effective. Another difference is that the threshold for poorly visited states is constant for \texttt{R-MAX}, while for CE follows a specific schedule, that can be used to tune the depth of the exploration. Moreover, \texttt{R-MAX} has a polynomial bound in states, actions, and mixing time $T$ (for every starting state, after $T=T(\epsilon)$ steps the policy is $\epsilon$-close to the optimal one), while CE has a bound that depends only on states, actions and $\gamma$. \texttt{R-MAX} is more general than CE, because it covers zero-sum stochastic games, not only MDP. Finally, as with $\text{E}^3$ and in contrast to CE, \texttt{R-MAX} does not require a $\rho$-reset model.% (but see the remark at the end of this section).

We remark that some of the above papers obtain theoretical results without any reset assumption, which is instead essential in CE. However, they replace it with the very strong assumption of ergodicity. Ergodicity implies that there is an $\epsilon$-mixing time $T$, thus the starting state is irrelevant because after $T$ steps the policy will become nearly-optimal. This mixing time is an unknown and potentially huge parameter, so any theoretical bound including $T$ is weaker than it seems. \texttt{R-MAX} contains $T$ in the final bound.

%-"optimism in the face of uncertainty" = the method defines a set of statistically plausible MDPs and chooses an optimistic MDP M (with respect to the achievable average reward) among these plausible MDPs\\
%-"bayesian" = the method keep track of a distribution over MDPs (exact process is infeasible in practice)
%
%
%Tra gli "Upper confidence bounds" ci sono: "interval estimation" (IE, MBIE, MBIE-MB), "optimistics" ($E^3$, R-Max, UCRL, UCRL2), "action elimination" (AE)\\
%I "curiosity driven" sono quelli che assegnano bonus in funzione di quando cambia il modello predittivo
%
%-->parole ricorrenti a cui va trovato un senso rigoroso: intrinsic motivation %,optimistics, curiosity driven

%ROBA NON TEORICA:

%\item Altro Bayesiano che usa come bonus la "sorpresa", cioè quanto una transizione cambia il modello. In particolare per high-dimensional state space (paper 2016 - VIME: Variational information maximizing exploration)

%\item Metodo che usa come bonus l'errore nella predizione. In particolare per high-dimensional state space (paper 2015 - Incentivizing exploration in reinforcement learning with deep predictive models)

%\item "Motivation": intern vs extern, intrinsic vs extrinsic, homeostatic vs heterostatics, fixed vs adaptive (articolo 2009 - What is intrinsic motivation? A typology of computational approaches)

\section{Background}

%\subsection{Notation}
%(definizione PAC) Performance Metrics ---> \mio{A reasonable notion of learning efficiency in an MDP is to require an efficient algorithm to achieve nearoptimal (expected) performance with high probability. An algorithm that satisfies such a condition can be said to be Probably Approximately Correct or PAC.}{questo è copiato ma la definizione è importante da scrivere} (cioè errore < $\epsilon$ con prob > 1-$\delta$)
%In supervised learning the \textbf{sample complexity} $N(\epsilon,\delta)$ of an algorithm $\mathfrak{A}$ is the minimum \textit{number of samples} $n\in\bbN$ such that
%\[\bbP\Big(\text{error} \big(\mathfrak{A} \text{ with } n \text{ samples} \big) \geq \epsilon\Big) < \delta,\]
%where the error is usually measured as the expectation of some appropriate loss function.

Given a set $\XXX$, we denote by $|\XXX|$ the cardinality of $\XXX$, by $\ind{\XXX}$ the indicator function of $\XXX$, and by $\Delta(\XXX)$ the set of all probability distributions over $\XXX$. %, and $\textit{unif}(\XXX)\in\Delta(\XXX)$ the uniform distribution over $\XXX$
Unless otherwise stated, notations and terminologies in this section are taken from \cite{KaLAOA}.

A Markov Decision Process (MDP) is a tuple $\MMM:=(\SSS,\rho,\AAA,r,\PPP)$ where $\SSS$ is the state space, $\rho$ is the \emph{starting state distribution}, $\AAA$ is the action space, $r:\SSS\times\AAA\to[0,1]$ is the reward function and $\PPP:\SSS\times\AAA\to\Delta(\SSS)$ is the transition model. Action choices are modeled by stationary policies, that is, conditional probability distributions $\pi:\SSS\to\Delta(\AAA)$ over actions.
The set of all stationary policies of the MDP will be denoted by $\Pi$.
As usual, we overload notation by writing $\PPP(s'|s,a)$ for the probability of transitioning to $s'$ after performing action $a$ in state $s$, and $\pi(a|s)$ for the probability of choosing action $a$ in state $s$.

Differently from other sources like for instance \cite{SuBRLI}, we consider the starting distribution $\rho$ as part of the MDP, to underline that $\rho$ represents how the decision problems corresponding to the MDP \emph{naturally} start. The assumption of a bounded reward function as in \cite{KakSCR} is necessary to obtain finite time convergence results with sampling based methods, while the choice of deterministic rewards in $[0,1]$ is for the sake of clarity and does not affect the generality of results.

A \emph{trajectory} $\tau$ of length $H$ is a state-action sequence $\tau = (s_0,a_0,s_1,a_1,s_2,a_2,...,s_{H-1},a_{H-1})$.
and accordingly the \emph{trajectory space} is
\[
\mathcal{T} := \bigcup_{H>0} (\SSS\times\AAA)^H.
\]
The trajectory space $\mathcal{T}$ is naturally made into a probability space $(\mathcal{T},\bbP)$ by joining policy and transition model at each step of the sequence,  using a \emph{discount factor $\gamma\in [0,1)$} for a geometric average on the trajectory length:
\begin{equation}\label{distribuzione_tau}
\bbP(\tau|\pi,\PPP,s_0\sim\rho):= (1-\gamma)\gamma^H\rho(s_0)\pi(a_0|s_0)\prod_{t=1}^{H-1} \PPP(s_t|s_{t-1},a_{t-1}) \pi(a_t|s_t).
\end{equation}

Using the probability $\bbP$ on trajectories the \emph{value of a state $s$} is the average discounted total reward that can be obtained from that state:
\begin{equation}\label{def_value_function}
V_{\pi,r,\PPP,\gamma}(s):= (1-\gamma) \bbE_{\tau \sim \bbP(\cdot|\pi,\PPP,s_0=s)}\left[\sum_{t=0}^\infty \gamma^t r(s_t,a_t)\right].
\end{equation}
Note that we are unconventionally using a \emph{normalized} value function, so that $V_{\pi,r,\PPP,\gamma}(s)\in [0,1]$. The reason is that otherwise the value function would be bounded by $(1-\gamma)^{-1}$, and so the $\epsilon$-accuracy required when estimating the value function should depend on $\gamma$ to make them comparable \cite[2.2.3]{KakSCR}.

Since the starting distribution $\rho$ represents how trajectories in the real problem start, policy learning methods look for a policy maximizing the value function averaged on starting states \cite{AKLTPG}:
\begin{equation}\label{argmax_problem}
\pi_*:=\text{argmax}_{\pi\in\Pi} V_{\pi,r,\PPP,\gamma}(\rho),\quad\text{where}\quad V_{\pi,r,\PPP,\gamma}(\rho):=\bbE_{s\sim\rho}\left[V_{\pi,r,\PPP,\gamma}(s)\right].
\end{equation}

The \emph{state visitation distribution} is a measure of how often states are visited under a certain policy. Formally:
\[
d^{\pi,\PPP,\gamma}_{\rho}(s):=(1-\gamma)\bbE_{\tau\sim\bbP(\cdot|\pi,\PPP,s_0\sim\rho)}\left[ \sum_{H=0}^\infty \gamma^H \ind{\{s_H=s\}}\right].
\]
Intuitively, the state visitation distribution conveys a notion of importance of states: if a state is rarely visited in trajectories, why should it be considered important? This is somehow reflected by the following distributional characterization of the value function:
\begin{equation}\label{def_value_function_distro}
V_{\pi,r,\PPP,\gamma}(\rho) = \sum_{s\in\SSS,a\in\AAA}d^{\pi,\PPP,\gamma}_{\rho}(s)\pi(a|s)r(s,a).
\end{equation}

Solving the argmax problem \eqref{argmax_problem} by directly improving the value function in \eqref{def_value_function} is not a good idea. Indeed, \eqref{def_value_function_distro} shows that the value function is insensitive to improvements at states where $d^{\pi,\PPP,\gamma}_{\rho}(s)$ is small, and this in turn means that \emph{we need to explore the unlikely states more}, not less.

This is an instance of the classic exploration-exploitation dilemma, particularly relevant with on-policy training, where states are visited with proportions given by the state visitation distribution $d^\pi$ of a suboptimal policy $\pi$. A common solution to this problem is to assume that $\rho$ is soft, that is, for every state $s$ one has $\rho(s)>0$. This is known as \emph{coverage assumption}. We stress that this is a very limiting assumption, because it is the problem itself that defines how trajectories start.
If $\rho$ is not soft, but we have access to a $\mu$-reset model \cite{KakSCR} for a soft $\mu$, we can solve the following problem:
\begin{equation}\label{argmax_problem_mu}
\pi_*:=\text{argmax}_\pi V_{\pi,r,\PPP,\gamma}(\mu),\quad\text{where}\quad V_{\pi,r,\PPP,\gamma}(\mu):=\bbE_{s\sim\mu}\left[V_{\pi,r,\PPP,\gamma}(s)\right].
\end{equation}

% VEDERE SE QUESTA FRASE È VERA
%The argmax problems \eqref{argmax_problem} and \eqref{argmax_problem_mu} \mau{are equivalent}{lo dico a istinto, ma sarebbe utile avere una ref o una dimostrazione.} when $\mu=d^{\pi_*,\PPP,\gamma}_{\rho}$, but when $\mu$ is not related to the MDP they can give very different solutions.

When $\mu$ is not related to the MDP, the argmax problems \eqref{argmax_problem} and \eqref{argmax_problem_mu} can give very different solutions. CE addresses this issue by an iterative procedure that at every step of the iteration improves exploration on previously poorly visited states. This improvement is related to the MDP, because it is built on a mixture of the starting state distribution and the exploration in the previous step.

%In this paper, we study a theoretical framework that, given a MDP with \emph{any starting state distribution $\rho$} and an optimization algorithm \texttt{opt}, produces a new algorithm \texttt{CE($\rho$,opt)} with improved exploration. If \texttt{opt} provides provable guarantees, the behaviour of \texttt{CE($\rho$,opt)} can be controlled to obtain provable guarantees of similar type. In particular, this framework can be used to pair an optimization algorithm that is provable only under the coverage assumption with a MDP where the starting distribution $\rho$ is not soft, producing a new provable optimization algorithm.

\begin{remark}
	Hereafter, for clarity of exposition and without any loss of generality, we assume that the MDP has a fixed initial state $s_0$, that is, $\rho$ is concentrated on one single state. Notice that from the coverage point of view, this is the worst possible case.
\end{remark}

The most commonly accepted notion of \emph{sample complexity} $N(\epsilon,\delta)$ of a Reinforcement Learning solving algorithm $\mathfrak{A}$ is the minimum number of samples $n\in\bbN$ such that
\begin{equation}\label{PAC}
\bbP\Big(\texttt{error} \big(\mathfrak{A} \texttt{ with } n \texttt{ samples} \big) \geq \epsilon\Big) < \delta,
\end{equation}
where \texttt{error} is usually measured as the expectation of some appropriate loss function, and the samples can be either the number of episodes or the number of timesteps. An algorithm that satisfies such a condition with arbitrarily small $\epsilon$ and $\delta$ is said to be Probably Approximately Correct or \emph{PAC}.

The notion of \emph{learning efficiency} is how $N(\epsilon,\delta)$ scales with decreasing $\epsilon$ and $\delta$. Given that exponential behaviour is trivial to achieve, considerable theoretical results guarantees polynomial dependance, sublinear in some special cases. 

We finish this section with the algorithm \texttt{visit$(\pi)$}, see algorithm \ref{alg:visit}. Given a policy $\pi$, this simple stochastic stopping algorithm simulates a $d^{\pi,\PPP,\gamma}_{s_0}(\cdot)$-reset model. This means that \texttt{visit$(\pi)$} returns a state $s\sim d^{\pi,\PPP,\gamma}_{s_0}(\cdot)$.

\begin{algorithm}
\caption{Visit$(\pi)$}\label{alg:visit}
\begin{algorithmic}[1]
  \STATE{Set $s=s_0$.}
  %\WHILE{$random\_real\_in(0,1)\leq\gamma$ } 
  \WHILE{\TRUE}
  	\STATE{With probability $1-\gamma$:}
  	\STATE{    \qquad\textbf{break}}
  	\STATE{Choice action $a$ according to $\pi(\cdot|s)$.}
  	\STATE{Perform $a$ in $\MMM$ and go to state $s'$ according to the unknown $\PPP(\cdot|s,a)$.}
  	\STATE{Set $s=s'$.} 
  \ENDWHILE
  \RETURN $s$
\end{algorithmic}
\end{algorithm}

\section{Results}

Proofs of theorems in this section are deferred to Appendix \ref{appendix:proofs}.

\subsection{The exploration phase}\label{sec:exploration}

The key idea of our strategy is to use the visit algorithm \ref{alg:visit} as a spread function over the distribution of the sampling model, in a convolutional smoothing fashion. Given a policy $\pi$ and a $\mu$-reset model, a simulated reset model $\tilde{\mu}=\texttt{visit}(\pi,\mu)$ is defined as follows:
\begin{equation}\label{visit_spread_distribution}
\tilde{\mu}(s') = \mu * d_{\cdot}^{\pi,\PPP,\gamma}(s') = \sum_{s\in\SSS} \mu(s) d_{s}^{\pi,\PPP,\gamma}(s') 
  \quad\Bigg( =: d_{\mu}^{\pi,\PPP,\gamma}(s')\Bigg).
\end{equation}
That is, moving from $\mu$ to $\tilde{\mu}$, the reset probability
of each state $s$ is ``spread'' across all other states according
to the state visitation distribution $d_s^{\pi,\PPP,\gamma}(\cdot)$.
Iterating this argument, concentrated distributions are diffused
to softer models.

Intuition is helped through an analogy with image blurring.
Represent a sampling model with an image where each pixel
is a state and the color depends on the reset probability
of that state, ranging continuously from black if 0 to white
if 1. Thus an $s_0$-reset model corresponds to an image that
is all black but for one white pixel.
The \texttt{visit} algorithm
spreads the $s_0$ pixel light across
other pixels according to the ``point spread function''
$d_{s_0}^{\pi,\PPP,\gamma}(\cdot)$. If more than one initial
pixel is not black, the corresponding spread lights will
merge linearly, each weighted by the respective light
intensity, in the same way as visit distributions in
\eqref{visit_spread_distribution}.

As every analogy, similarities are limited.
In image blurring the spread function is the same
for each pixel, thus blurring a one-white-pixel
image towards a uniform gray image. In our case
the dynamical reset blurs the model according
to the state visitation distribution.

We now describe Algorithm \ref{alg:CE}. At each step $n=0,1,\dots$ the policy $\pi_{n+1}$ is the result
of a curiosity-driven optimization strategy:
\begin{equation}\label{strategy}
\mu_{-1}=s_0 (=\rho)
\quad \underbrace{\longrightarrow}_{\pi_0} \quad 
\mu_{-1} * d_{\cdot}^{\pi_0,\PPP,\gamma} =: \mu_0
\quad \underbrace{\longrightarrow}_{\pi_1} \quad 
\dots
\quad \underbrace{\longrightarrow}_{\pi_{n-1}} \quad 
\mu_{n-1} * d_{\cdot}^{\pi_{n-1},\PPP,\gamma} =: \mu_{n}
\quad \underbrace{\longrightarrow}_{\pi_{n}} \quad
...
\end{equation}

The above strategy \eqref{strategy} is initialized with a uniform policy $\pi_0=\pi^U$. Then, a subprocess is started to compute an ``optimally exploring'' policy $\pi_{n+1}$.
%, but in general an informed inizialization of $\mu_{-1}$ can be used to inject a domain knowledge of the MDP, and speed up the exploration phase.
This subprocess is indeed a standard value optimization run
on a virtual MDP $\MMM_n$ and can thus be made through
any RL optimization method \texttt{opt}.
%Any RL algorithm that look for optimal value policies can thus be used for this step.
During the subprocess, each trajectory restarts according
to a mixture of $s_0$ (or $\rho$) and the simulated reset model $\mu_{n}$. We shall explicitly state the dependence
of the optimization method from the reset model, namely,
\[
%\pi_{n} = \texttt{opt}(\MMM_n,\pi_{n-1},\mu_{n-1})
\pi_{n+1} = \texttt{opt}(\MMM_n,\mu_{n}),
\]
where $\MMM_n=(\SSS,s_0,\AAA,r_n,\PPP)$ is a MDP which has the same states, starting distribution, actions and transitions as the original MDP $\MMM$, but different reward function. The reward function $r_n$ is an intrinsic reward motivating the optimizer in exploring less explored states. To measure how much states are visited, we define the set of \emph{poorly visited states} as
\begin{equation}\label{bad_states}
\KKK_n = \left\{ s\in\SSS \mathrel{\Big|} \sum_{i=0}^n d^{\pi_n,\PPP,\gamma}_{\mu_{n-1}}(s) \leq \beta_n \right\},
\end{equation}
where a sequence of $\beta_n\ge 0$ is suitably chosen as a hyperparameter of CE. In the extreme case $\beta_n=0$, states will be poorly visited only up to the first visit.
At the other end there is $\beta_n$ going to infinite faster than $n$, and in this case every state will be poorly visited forever.
The intrinsic reward is then:
\begin{equation}
r_n(s,a):=\ind{\{s\in\KKK_n\}}.
\end{equation}
Notice that the original reward is completely forgot, in contrast with \texttt{R-MAX}. In other words, the process on each trajectory is ``lured'' to visit
those states whose overall visit distribution in previous steps
falls below a threshold.
%Visit distributions, value function and reward function are strongly related by the following Lemma.
%\begin{lemma}
%Given an MDP $\MMM=(\SSS,\AAA,\PPP,r)$, a policy $\pi\in\Pi_\MMM$, a state distribution $\mu$ and a discount factor $\gamma\in[0,1)$ it holds that
%\begin{equation}
%V_{\MMM,\pi,\gamma}(\mu) = \sum_{s\in\SSS,a\in\AAA} d^{\pi,\gamma}_{\mu}(s) \pi(a|s) r(s,a).
%\end{equation}
%\end{lemma}
Each reward function $r_n$ is action-independent, therefore
for every policy $\pi$ we get
\[
  V_{\pi,r_n,\PPP,\gamma}(\mu_{n-1})
    = \sum_{s\in\SSS,a\in\AAA} d^{\pi,\PPP,\gamma}_{\mu_{n-1}} (s) \pi(a|s) r_n(s,a)
    = \sum_{s\in\SSS} d^{\pi,\PPP,\gamma}_{\mu_{n-1}} (s) \ind{\{s\in\KKK_n\}}
    = \sum_{s\in\KKK_n} d^{\pi,\PPP,\gamma}_{\mu_{n-1}} (s),
\]
where we have used the distributional characterization of the
value function \eqref{def_value_function_distro}.
Thus, maximizing the value function in $\MMM_n=(\SSS,s_0,\AAA,r_n,\PPP)$
corresponds to maximizing the sum of visitation distributions
on poorly visited states.

\begin{algorithm}[H]
\caption{Curious Explorer}\label{alg:CE}
\begin{algorithmic}[1]
  %\scriptsize
  \STATE \textbf{Input:} Markov Decision Process $\MMM=(\SSS,s_0,\AAA,\PPP,r)$, RL optimization algorithm \texttt{opt}$(\cdot,\cdot)$.\\
  \STATE {Schedule the visit threesholds $\{\beta_n\}_{n\in\bbN}$}.
  \STATE{If \texttt{opt} is PAC schedule the error bounds $\{\epsilon_n\}_{n\in\bbN}$ and affidabilities $\{\delta_n\}_{n\in\bbN}$.}
  \STATE{Set $\mu_{-1}:=\ind{\{s=s_0\}}$.}
  \STATE{Set $\pi_0:=\pi^U$.}
  \FOR{$n=0,1,2,\dots$} 
  	\STATE {Define $\DDD_n=d^{\pi_n,\PPP,\gamma}_{\mu_{n-1}}$.}
  	\STATE {Set $\KKK_n = \left\{ s\in\SSS \big| \sum_{i=0}^n \DDD_i(s) \leq \beta_n \right\}$ and $r_n(s,a)=\ind{\{s\in\KKK_n\}}$. }
  	\STATE {Set $\mu_n(s) = \frac{1}{2}\DDD_n(s) + \frac{1}{2}\ind{\{s=s_0\}}$.}
  	\STATE {Find $\pi_{n+1} = (\epsilon_n,\delta_n)\texttt{-opt}(\underbrace{\MMM_n}_{(\SSS,\rho,\AAA,\PPP,r_n)},\mu_n)$.}
  \ENDFOR
\end{algorithmic}
\end{algorithm}

A few remarks are in order. The first is that at each step the new reset model
is defined as a balance between the visit distribution collected with the current
policy and the original $s_0$-reset model. This has a twofold purpose.
On the one hand the ``true'' $s_0$-reset model is often meaningful for the problem.
%\mau{, or a constraint}{non capisco, toglierei}.
On the other hand it allows for a uniform control of the number of poorly visited states, as well as of the visit distribution on $\KKK_n$.
The second remark is over the choice of the sequence of thresholds $\beta_n$ that identify poorly visited states. The choice of a constant sequence $\beta_n$ makes poorly visited states in CE the same as unknown states in \texttt{R-MAX} \cite[\emph{Initialize:} paragraph at page 219]{BrTRGP}. Under the choice $\beta_n=\beta\cdot n$, for some number $\beta>0$, the sum in (\ref{bad_states}) can be re-interpreted as an average of the visit distributions over the ensemble of reset models and policies generated in the previous steps of the algorithm. Thus each \emph{non}-poorly visited state on average has accumulated a proportion of visit distribution which is larger than $\beta$. In particular, on $\SSS-\KKK_n$ the CE output $\mu_n$ of the exploration phase is soft, and this allows to use theoretical results requiring a coverage hypothesis.
If the subprocess \texttt{opt} is PAC and so some theoretical upper bounds $\epsilon_n$ on the error of the returned policy (possibily involving the reset model) holds, we can use them to provide theoretical guarantees on the exploration.

\begin{theorem}\label{thm:ceprovable}
% Assume access to an algorithm \texttt{opt} which takes
% as input an MDP, a sampling model, an $(\epsilon,\delta)$
% error bounds and return a policy that the value of which 
Let an MDP $\MMM=(\SSS,s_0,\AAA,r,\PPP)$, a $s_0$-reset
model, a discount factor $\gamma$ and a PAC
solver algorithm \texttt{opt} be given. Assume that
$(\epsilon,\delta)$-\texttt{opt}$(\MMM',\mu)$ is
guaranteed to return a policy whose value
in $\MMM'$, measured with respect to $\mu$, is
at least $\epsilon$-optimal with probability
at least $1-\delta$.
Suppose to follow Algorithm 2 with
any threshold parameters $\{\beta_n\}_{n\leq N}$.
\\
Then for every step $N$ with probability at least $1-\sum_{n=0}^N \delta_n$ we have
\begin{equation}\label{bound}
  \sum_{n=0}^N \max_{\pi\in\Pi} \sum_{s\in\KKK_n} d^{\pi,\PPP,\gamma}_{s_0}(s)
\leq
  2\sum_{n=0}^N \epsilon_n 
  +
  2\sum_{s\in\SSS}\beta_{\tilde{n}(s)}
  +
  2\sum_{s\in\SSS}\max_{\pi\in\Pi} d_{s_0}^{\pi,\PPP,\gamma}(s),
\end{equation}
\end{theorem}

%Using the $\epsilon_n$ bound on the error we can obtain a single-step lower bound of the performance in visiting bad states
%\begin{equation}\label{lower_bound} 
%\max_{\pi\in\Pi} \sum_{s\in\KKK_n} d^\pi_{s_0}(s)
%\leq
%2\epsilon_n + 2\sum_{s\in\KKK_n}d^{\pi_{n+1}}_{\mu_n}(s)
%\end{equation}
%Using the $\KKK_n$ definition we can obtain a single-state upper bound of the performance in visiting bad states
%\begin{equation}\label{upper_bound}
%\sum_{n=0}^{N} \ind{\{s\in\KKK_n\}} d^{\pi_{n+1}}_{\mu_n}(s)
%\leq
%\beta_{\tilde{n}(s)} + \max_{\pi\in\Pi} d^\pi_{s_0}(s)
%\end{equation}
where $\tilde{n}(s):=\max\{n\leq N|s\in\KKK_n\}$, and the maximums
above are taken over the set $\Pi$ of all stationary policies.

%Combining the two with a double counting trick we obtain
%Given a MDP $\MMM$ and a $s_0$-reset sampling model on $\MMM$ with fixed $s_0\in\SSS$. Assume to follow Algorithm 1. Then for every iteration $N\in\bbN$ it holds

The last term $\sum_{s\in\SSS}\max_{\pi\in\Pi} d_{s_0}^{\pi,\PPP,\gamma}(s) =:\mathfrak{e}$ (which is always $\leq |\SSS|$) is fixed and is indeed a characteristic parameter for every MDP, that we call \emph{exploitative factor}. The quantity $\mathfrak{e}$ measures, in terms of number of samples, the overall cost of gathering information about the structure of the MDP. Thus $\mathfrak{e}$ depends on the complexity of the graph geometry of the MDP induced by the transition model and by all possible policies. 
For instance, in a MDP with a few well-separated paths leading to final states (think of a tree structure), $\mathfrak{e}$ is of the order of the number of branches times the number of samples required to explore each path.
In conclusion, the previous theorem tells us that on average, the maximum possible visitability of $\KKK_n$ states is bounded from above by the average of the $\epsilon_n$ (as in the above theorem) and a constant divided by $N$, and thus is arbitrarily small for a larger and larger number of steps. This, combined with simulated reset model, with guaranteed coverage outside $\KKK_n$, can be used to prove theoretical convergence bounds with relaxed assumptions. We shall give an example of this procedure in the next subsection.

\subsection{An example of the improvement phase: \texttt{CE$(s_0,\REINFORCE)$}}\label{sec:improvement}

The exploration phase described in \ref{sec:exploration} can be briefly described as ``plug \texttt{opt} into CE''. This improvement phase is then briefly subsumed by ``plug CE into \REINFORCE''.

With a careful choice of parameters, \REINFORCE can be proved to be PAC. Indeed, after $i$ episodes, with probability at least $1-\delta$, Theorem~6 of \cite{ZKOSER} shows that \REINFORCE returns a policy $\pi$ such that
\begin{equation}\label{eq:boundreinforce}
V_{\pi^*,r,\PPP,\gamma}(\mu) - V_{\pi,r,\PPP,\gamma}(\mu)
\leq
C \frac{|\SSS|^2|\AAA|^2}{(1-\gamma)^2}\log(i/\delta)^{5/2} \frac{1}{i^{1/6}} \norm{\frac{d_\mu^{\pi^*,\PPP,\gamma}}{\mu}}_\infty^2,
\end{equation}
for a universal constant $C$, independent from the specific MDP.
The last term on the right-hand side, known as the \emph{mismatch coefficient},
clearly shows, at this level, the necessity of the coverage
assumption.

If we now plug into Theorem~\ref{thm:ceprovable} the bound \eqref{eq:boundreinforce} on error provided by \REINFORCE, we are able to measure the overall visitation of the set of poorly visited states in this context.

\begin{theorem}\label{thm:reinforcece}
Let an MDP $\MMM=(\SSS,s_0,\AAA,r,\PPP)$, a $s_0$-reset model and a discount factor $\gamma$ be given. Set \texttt{opt}=\REINFORCE, tuned according to \cite[Theorem~6]{ZKOSER}, and perform it for $i(n)$ episodes at each step $n$, with $\delta_n:=\delta/N$.
\\
Then for every step $N$, with probability at least $1-\delta$ we have
\begin{equation}%\label{bound2}
\begin{aligned}
    \sum_{n=0}^N \max_{\pi\in\Pi} \sum_{s\in\KKK_n} d^{\pi,\PPP,\gamma}_{s_0}(s)
&\leq
  2C\frac{|\SSS|^2|\AAA|^2}{(1-\gamma)^2}\log N
  \sum_{n=0}^N  \frac{\log(i(n)/\delta)^{5/2}}{i(n)^{1/6}} \max_{\pi\in\Pi}\norm{\frac{d_{s_0}^{\pi,\PPP,\gamma}}{\mu_n}}_\infty^2\\
&\quad
  +
  2\sum_{s\in\SSS}\beta_{\tilde{n}(s)}
  +
  2\sum_{s\in\SSS}\max_{\pi\in\Pi} d_{s_0}^{\pi,\PPP,\gamma}(s)
\end{aligned}
\end{equation}
%where $C$ is a constant that is indipendent on the problem.
\end{theorem}

\begin{theorem}\label{thm:cereinforce}
Let an MDP $\MMM=(\SSS,s_0,\AAA,r,\PPP)$, a $s_0$-reset model and a discount factor $\gamma$ be given. Perform CE to obtain a simulated $\mu_N$-reset model. Perform \REINFORCE on $\MMM=(\SSS,\mu_N,\AAA,r,\PPP)$ and call $\pi$ the policy returned after $i$ episodes. Then with probability at least $1-\delta$ it holds
\begin{equation}
V_{\pi^*,r,\PPP,\gamma}(s_0) - V_{\pi,r,\PPP,\gamma}(s_0)
 \leq  C \frac{|\SSS|^2|\AAA|^2}{(1-\gamma)^2}
 \frac{\log(i/\delta)^{5/2}}{i^{1/6}}\norm{\frac{d_{s_0}^{\pi^*,\PPP,\gamma}}{\mu_N}}_\infty^2
\end{equation}
%oppure se \textbf{proprio} vogliamo ladrare
%\begin{equation}
%V_{\pi^*,r,\PPP,\gamma}(s_0) - V_{\pi,r,\PPP,\gamma}(s_0)
% =  O\left( \frac{|\SSS|^2|\AAA|^2}{(1-\gamma)^2}
% \frac{\log(i/\delta)^{5/2}}{i^{1/6}} \right)
%\end{equation}
\end{theorem}
%ref articolo tesi is the first sublinear in the number of samples bound, with polynomial dependance in the classic parameter $|\SSS|$, $|\AAA|$ and $\gamma$. In order to guarantee exploration it require a strong coverage hypotesis on the sampling model (but only this!, not ergodicity or strange hidden terms). Thus if we set $pi=\epsilon$-\REINFORCE$(\MMM,\mu)$, performing \REINFORCE (with appropriate parameter and appropriate loss function (with reg term)) for $i$ episodes its guaranteed (da articolo tesi) that (with prob>1-$\delta$)
%Thus we can make $\epsilon$ arbitrarly small with sufficient high number $i$ of episodes.\\
%So the only not a priori boundable (and potentially exponential) term is the mismatch coefficient, from which is clear the need for the wide coverage hypothesis.

\section{Experiments}

Experiments can be found in Appendix \ref{appendix:experiments}.

\section{Limitations and future work}\label{sec:limitations}

To compute $d_{\mu_{n-1}}^{\pi_n,\gamma}$  we run \texttt{visit} for $i$ samples and count occurrences of each state ($i$ can be the same used for \texttt{opt} without influencing asymptotic convergence rate). Counting visits directly during the learning phase of \texttt{opt} would be more sample efficient but less precise, since the policy changes in the process.
It is unfeasible to know the exact values of visit distributions, an assumption of our theoretical bounds. The results still hold if we set a tolerance margin around $\beta$ and use Azuma-Hoeffding to upper bound the probability of being out of that interval. Since samples are independent, this bound is exponentially decreasing in the number of samples.

The solving procedure  maintains a strong separation between exploration and improvement. This is strongly pursued (Section \ref{sec:contrib}) however it is not sample efficient. Off-policy learning with the true rewards during the exploration phase could help. The two phases can be performed at once considering the maximum between intrinsic and extrinsic reward, in the spirit of \texttt{R-MAX}.

%These two stages may be made together, in the same spirit as policy evaluation and improvement in generalized policy iteration, but this is beyond the scope of this Thesis and may be a subject for future work.
Finally, PAC guarantees in \cite{ZKOSER} are built over \cite[Theorem 5.3]{AKLOAP}, stating:
\begin{equation}
V_{\pi^*,\gamma}(s_0) - V_{\pi,\gamma}(s_0)
\leq  \frac{1}{(1-\gamma)|\SSS|} \norm{\frac{d_{s_0}^{\pi^*,\gamma}}{\mu}}_\infty \norm{\nabla_\pi V_\pi}_\infty
\end{equation}
If the \emph{true} gradient is small, so is the error. It is indeed in this key step that the distribution mismatch appear, and with it the need of coverage. 
The logic of theoretical bounds in CE can be used to avoid this distribution mismatch dependance. In fact, CE provides two ``ortogonal'' guarantees on the visit distribution: a pointwise lower bound outside $\KKK$ (definition of $\KKK$), and an average upper bound inside $\KKK$, Theorem \ref{thm:ceprovable}.
%\begin{itemize}
%\item distribuzione pointwise > $\beta$ fuori da $\KKK$
%\item distribuzione average < bound dentro $\KKK$
%\end{itemize}
%Queste due si possono usare per boundare $V^*-V_\pi$ senza il \mio{mismatch coefficient}{va detto cos'Ã¨}
%\begin{align*}
%V_{\pi^*,\gamma}(s_0) - V_{\pi,\gamma}(s_0)
%& = \frac{1}{1-\gamma} \bbE_{s,a\sim d^{\MMM,\pi^*,\gamma}_{s_0}} [A_{\MMM,\pi,\gamma}(s,a)] \leq \\
%& \leq \frac{1}{1-\gamma} \sum_{s\in\SSS} d^{\pi^*,\gamma}_{s_0}(s) \max_{a\in\AAA}A_{\MMM,\pi,\gamma}(s,a) \leq \\
%& \leq \frac{1}{(1-\gamma)^2} \underbrace{\sum_{s\in\KKK} d^{\pi^*,\gamma}_{s_0}(s)}_{\leq bound (\ref{bound})}
%+ 
%\frac{1}{1-\gamma} \norm{\nabla_\pi V_\pi}_\infty \sum_{s\not\in\KKK} \underbrace{\frac{d^{\pi^*,\gamma}_{s_0}(s)}{d^{\pi,\gamma}_{s_0}(s)}}_{\leq 1/\beta}
%\end{align*}
\[
V_{\pi^*,\gamma}(s_0) - V_{\pi,\gamma}(s_0)\leq \frac{1}{(1-\gamma)^2} \underbrace{\sum_{s\in\KKK} d^{\pi^*,\gamma}_{s_0}(s)}_{\leq bound (\ref{bound})}
+ 
\frac{1}{1-\gamma} \norm{\nabla_\pi V_\pi}_\infty \sum_{s\not\in\KKK} \underbrace{\frac{d^{\pi^*,\gamma}_{s_0}(s)}{d^{\pi,\gamma}_{s_0}(s)}}_{\leq 1/\beta}
\]

%These two stages may be made together, in the same spirit as policy evaluation and improvement in generalized policy iteration, but this is beyond the scope of this Thesis and may be a subject for future work.
%Estendere a average reward setting? dovrebbe essere straightforward. \mio{mmmm non so mi sembra un casino, nella fase di esplorazione si perde l'equivalenza valuefunction=sommavisitesuK e nella fase di improvement si perde la sfruttabilitÃ  del sampling model per migliorare le prestazioni}
%da \cite{KaLAOA} "though this has a similar solution to maximizing the average reward for [...] see [4]"
%guardare \url{https://lilianweng.github.io/lil-log/2020/06/07/exploration-strategies-in-deep-reinforcement-learning.html} sezione Memory-based Exploration (Reward-based exploration suffers from several drawbacks:)
In perspective, the next step in the analysis of CE will be to derive an upper bound on the distribution mismatch coefficient using the theoretical bounds of CE itself. Indeed so far we only know that $\mu_N$ is soft and the mismatch is finite, without explicit bounds. Experiments show though that the mismatch is usually not too big. 

Finally, exploration is not always a good choice, see the noisy-TV example \cite{BEPLSS}. However we believe that a strong theoretical ground on pure exploration will foster innovative results.

%\begin{ack}
%Use unnumbered first level headings for the acknowledgments. All acknowledgments
%go at the end of the paper before the list of references. Moreover, you are required to declare
%funding (financial activities supporting the submitted work) and competing interests (related financial activities outside the submitted work).
%More information about this disclosure can be found at: \url{https://neurips.cc/Conferences/2021/PaperInformation/FundingDisclosure}.
%
%Do {\bf not} include this section in the anonymized submission, only in the final paper. You can use the \texttt{ack} environment provided in the style file to autmoatically hide this section in the anonymized submission.
%\end{ack}

\paragraph{Acknowledgments}
Many thanks to Davide Bacciu and Rosa Gini for their very useful comments, and to Leonardo Robol for his very efficient support in the use of the computational resources of the laboratory of the department of mathematics.

\appendix

\section*{Appendix}

\section{Proofs of Theorems}\label{appendix:proofs}

We recall that we denote by $\SSS$ the state space, by $\AAA$ the action space,
by $\PPP$ the transition model, by $r$ the reward function, and by
$\Pi$ the set of all stationary policies of the MDP $\MMM$.
The state visitation distribution, with discount factor $\gamma$,
given the policy $\pi$ and the starting distribution $\rho$, is
\begin{equation}\label{eq:state_visitation}
d_\rho^{\pi,\PPP,\gamma}
= (1-\gamma)\bbE_{\tau\sim\bbP(\cdot\mid\pi,\PPP,s_0\sim\rho)}
\Bigl[\sum_{H=0}^\infty \gamma^H\ind{\{s_H=s\}}\Bigr].
\end{equation}
First we prove that, \emph{given the policy}, the exploration step
does not change the state visitation distribution.
\begin{lemma}\label{l:on_visit}
	For every $n\geq1$, let $\MMM_n$ be the virtual MDP which
	has the same states, actions and transitions as $\MMM$,
	but with reward function $r_n(s,a)=\ind{\{s\in\KKK_n\}}$.
	Then for every starting distribution $\mu\in\Delta(\SSS)$
	and every policy $\pi$,
	\[
	d^{\MMM_n,\pi,\PPP,\gamma}_\mu(s)
	= d^{\MMM,\pi,\PPP,\gamma}_\mu(s),
	\qquad s\in\SSS.
	\]
\end{lemma}
\begin{proof}
	The equality is immediate, once one notices that the set of policies for
	$\MMM_n$ is the same as the set $\Pi$ of policies of $\MMM$, and therefore
	the two MDPs differ only on the reward function. In particular, the
	distribution $\bbP(\cdot\mid\pi,\PPP,s_0\sim\rho)$ of trajectories
	is the same for the two MDPs, since it depends only on the transition
	probabilities, the starting distribution and the policy.
\end{proof}
In the rest of the section, for simplicity, we will drop the
indication of the dependence of the MDP $\MMM$ and of the
transition model $\PPP$ from the state visitation distribution
(so $d^{\pi,\gamma}$ will stand for $d^{\MMM,\pi,\PPP,\gamma}$,
or $d^{\MMM_n,\pi,\PPP,\gamma}$, which is the same).

The next lemma shows that the exploration phase spreads the
visit distribution.
\begin{lemma}\label{l:boh}
	Let $\mu$ be a simulated reset distribution obtained, starting from $s_0$,
	by applying the \texttt{visit} algorithm $n$ times, for some $n\geq1$, with
	a sequence of policies. Then for any $\pi\in\Pi$,
	\[
	d^{\pi,\gamma}_{\mu}(s)
	\leq  \max_{\pi'\in\Pi} d^{\pi',\gamma}_{s_0}(s).
	\]
\end{lemma}
\begin{proof}
	Fix a state $s\in\SSS$. We can explicitly define a policy whose
	visit distribution is greater than the left-hand side. Then it
	is sufficient to compute a local upper bound and use the
	Markov property to conclude.
\end{proof}

\setcounter{theorem}{0}
\begin{theorem}%\label{thm:ceprovable}
	Let an MDP $\MMM=(\SSS,s_0,\AAA,r,\PPP)$, a $s_0$-reset
	model, a discount factor $\gamma$ and a PAC
	solver algorithm \texttt{opt} be given. Assume that
	$(\epsilon,\delta)$-\texttt{opt}$(\MMM',\mu)$ is
	guaranteed to return a policy whose value
	in $\MMM'$, measured with respect to $\mu$, is
	at least $\epsilon$-optimal with probability
	at least $1-\delta$.
	Suppose to follow Algorithm 2 with
	any threshold parameters $\{\beta_n\}_{n\leq N}$.
	
	Then for every step $N$ with probability at least
	$1-\sum_{n=0}^N \delta_n$ we have
	\begin{equation}\label{boundappendix}
	\sum_{n=0}^N \max_{\pi\in\Pi} \sum_{s\in\KKK_n} d^{\pi,\gamma}_{s_0}(s)
	\leq
	2\sum_{n=0}^N \epsilon_n 
	+
	2\sum_{s\in\SSS}\beta_{\tilde{n}(s)}
	+
	2\sum_{s\in\SSS}\max_{\pi\in\Pi} d_{s_0}^{\pi,\gamma}(s),
	\end{equation}
	where $\tilde{n}(s):=\max\{n\leq N|s\in\KKK_n\}$.
\end{theorem}
\begin{proof}
	Fix a step $n\geq1$.
	First, by Lemma~\ref{l:on_visit} we know that the state
	visitation distributions of $\MMM$ and of the virtual
	MDP $\MMM_n$ which has the same states,
	actions and transitions as $\MMM$, but with reward function
	$r_n(s,a)=\ind{\{s\in\KKK_n\}}$, are the same. We will
	use this fact without further notice in the proof.
	Recall the characterization of the value function
	through the state visitation,
	\[
	V_{\MMM_n,\pi,\gamma}(\mu_n)
	= \sum_{s\in\SSS,a\in\AAA} d^{\pi,\gamma}_{\mu_n}(s) \pi(a\mid s) r_n(s,a).
	\]
	Thus, for every policy $\pi\in\Pi$,
	\begin{equation}\label{eq:valuetobad}
	V_{\MMM_n,\pi,\gamma}(\mu_n)
	= \sum_{s\in \KKK_n}  d^{\pi,\gamma}_{\mu_n}(s)
	\Bigl(\sum_{a\in\AAA}\pi(a\mid s)\Bigr)
	= \sum_{s\in \KKK_n}  d^{\pi,\gamma}_{\mu_n}(s).
	\end{equation}
	By the guarantees on \texttt{opt} and our algorithm
	(Algorithm 2 - Curious Explorer), the optimal policy
	at step $n+1$ is
	\[
	\pi_{n+1}
	= (\epsilon_n,\delta_n)-\texttt{opt}(\MMM_n,\mu_n),
	\]
	we have that with probability at least $1-\delta_n$,
	%\begin{align*}
	%\max_{\pi\in\Pi} \sum_{s\in \KKK_n} d^{\pi,\gamma}_{\mu_n}(s) 
	% & = \max_{\pi\in\Pi} V_{\MMM_n,\pi,\gamma}(\mu_n) \leq\\
	% & \leq V_{\MMM_n,\pi_{n+1},\gamma}(\mu_n) + \epsilon_n= \\
	% & = \sum_{s\in \KKK_n} d^{\pi_{n+1},\gamma}_{\mu_n}(s)  +\epsilon_n
	%\end{align*}
	\begin{equation}\label{eq:eps_optimality_step_n}
	\max_{\pi\in\Pi} V_{\MMM_n,\pi,\gamma}(\mu_n)
	\leq V_{\MMM_n,\pi_{n+1},\gamma}(\mu_n) + \epsilon_n.
	\end{equation}
	Again by our algorithm (Algorithm 2), the reset distribution
	is given by
	$\mu_n(s) = \frac12 d^{\pi_n,\gamma}_{\mu_{n-1}}(s) + \frac12\ind{\{s=s_0\}}$,
	therefore for every policy $\pi\in\Pi$,
	\[
	\frac12 V_{\MMM_n,\pi,\gamma}(s_0)
	\leq \frac12 V_{\MMM_n,\pi,\gamma}(s_0) 
	+ \frac12 V_{\MMM_n,\pi,\gamma}(d^{\pi_n,\gamma}_{\mu_{n-1}})
	= V_{\MMM_n,\pi,\gamma}(\mu_n),
	\]
	and thus taking the maximum on both sides leads to
	\begin{equation}\label{eq:uniform_to_s0}
	\frac12 \max_{\pi\in\Pi} V_{\MMM_n,\pi,\gamma}(s_0)
	\leq \max_{\pi\in\Pi} V_{\MMM_n,\pi,\gamma}(\mu_n).
	\end{equation}
	Formulae \eqref{eq:eps_optimality_step_n} and \eqref{eq:uniform_to_s0},
	together with the characterization \eqref{eq:valuetobad} of the value
	function in terms of the bad states yield,
	\begin{equation}\label{eq:th1_left_bound_presum}
	\frac12 \max_{\pi\in\Pi} \sum_{s\in \KKK_n} d^{\pi,\gamma}_{s_0}(s)
	\leq \sum_{s\in \KKK_n}  d^{\pi_{n+1},\gamma}_{\mu_n}(s) + \epsilon_n,
	\end{equation}
	with probability at least $1-\delta_n$.
	
	We fix a final step $N>0$ and analyze steps $n=0,\dots,N$.
	For a state $s\in\SSS$, if $s\not\in\KKK_n$  we have that
	\[
	\sum_{n=0}^N \ind{\{s\in \KKK_n\}} d^{\pi_{n+1},\gamma}_{\mu_n}(s)
	= 0.
	\]
	Otherwise, set $\widetilde{n}=\max\{n\leq N:s\in\KKK_n\}$, then
	by the definition of $\KKK_n$ we get,
	\[
	\sum_{n=0}^N \ind{\{s\in \KKK_n\}} d^{pi_{n+1},\gamma}_{\mu_n}(s)
	\leq \sum_{n=0}^{\widetilde{n}} d^{\pi_n,\gamma}_{\mu_n}(s)
	+ d^{\pi_{\widetilde{n}+1},\gamma}_{\mu_n}(s)
	\leq \beta_{\widetilde{n}} + d^{\pi_{\widetilde{n}+1},\gamma}_{\mu_n}(s).
	\]
	In either case, the inequality
	\[
	\sum_{n=0}^N \ind{\{s\in \KKK_n\}} d^{\pi_{n+1},\gamma}_{\mu_n}(s)
	\leq \beta_{\widetilde{n}} + d^{\pi_{\widetilde{n}+1},\gamma}_{\mu_n}(s)
	\]
	holds for all $s\in\SSS$. By Lemma~\ref{l:boh},
	\begin{equation}\label{eq:th1_right_bound_presum}
	\sum_{n=0}^N \ind{\{s\in \KKK_n\}} d^{\pi_{n+1},\gamma}_{\mu_n}(s)
	\leq \beta_{\widetilde{n}(s)} + \max_{\pi\in\Pi} d^{\pi,\gamma}_{s_0}(s),
	\end{equation}
	where we have made explicit the dependence of $\widetilde n$
	from the state $s$.
	
	We now use a double counting trick, over states and over
	algorithm's steps, of the state visitation distributions
	of our explorative policies on poorly visited states.
	Indeed, summing \eqref{eq:th1_right_bound_presum}
	over $s\in\SSS$ yields
	\begin{equation}\label{eq:th1_right_bound}
	\sum_{s\in\SSS} \sum_{n=0}^N \ind{\{s\in \KKK_n\}} d^{\pi_{n+1},\gamma}_{\mu_n}(s)
	\leq \sum_{s\in\SSS}\beta_{\widetilde{n}(s)}
	+ \sum_{s\in\SSS}\max_{\pi\in\Pi} d^{\pi,\gamma}_{s_0}(s),
	\end{equation}
	while summing \eqref{eq:th1_left_bound_presum} over
	$n\in\{0,\dots,N\}$ yields
	\begin{equation}\label{eq:th1_left_bound}
	\begin{aligned}
	\sum_{n=0}^N \sum_{s\in\SSS} \ind{\{s\in \KKK_n\}}
	d^{\pi_{n+1},\gamma}_{\mu_n}(s)
	&= \sum_{n=0}^N \sum_{s\in\KKK_n}
	d^{\pi_{n+1},\gamma}_{\mu_n}(s)\\
	&\geq - \sum_{n=0}^N\epsilon_n
	+ \frac12 \sum_{n=0}^N \max_{\pi\in\Pi} \sum_{s\in \KKK_n}  d^{\pi,\gamma}_{s_0}(s),
	\end{aligned}
	\end{equation}
	with probability at least $1-\sum_{n=0}^N \delta_n$.
	By putting together \eqref{eq:th1_left_bound} and
	\eqref{eq:th1_right_bound}, we deduce that
	\[
	\frac12 \sum_{n=0}^N \max_{\pi\in\Pi} \sum_{s\in \KKK_n}
	d^{\pi,\gamma}_{s_0}(s)
	\leq \sum_{n=0}^N \epsilon_n
	+ \sum_{s\in\SSS} \beta_{\widetilde{n}(s)}
	+ \sum_{s\in\SSS}\max_{\pi\in\Pi} d^{\pi,\gamma}_{s_0}(s),
	\]
	and this completes the proof.
\end{proof}
\begin{remark}
	It is also worth noting that, for every $n$
	and every $Q\subset\KKK_n$,
	\[
	\max_{\pi\in\Pi} \sum_{s\in Q}  d^{\pi,\gamma}_{s_0}(s) 
	\leq \max_{\pi\in\Pi} \sum_{s\in \KKK_n}  d^{\pi,\gamma}_{s_0}(s)
	\]
	Indeed, by its definition, the state visitation distribution
	$d^{\pi,\gamma}_{s_0}(s)$ is always non-negative.
	Thus formula \eqref{eq:th1_left_bound_presum} also holds true in the form
	\begin{equation}\label{eq:th1_left_bound_presum_Q}
	\frac12 \max_{\pi\in\Pi} \sum_{s\in Q}  d^{\pi,\gamma}_{s_0}(s) 
	\leq \sum_{s\in \KKK_n}  d^{\pi_{n+1},\gamma}_{\mu_n}(s) + \epsilon_n,
	\end{equation}
	for all $Q\subseteq \KKK_n$. Now, with the same procedure
	of proof of Theorem~\ref{thm:ceprovable}, but using
	\eqref{eq:th1_left_bound_presum_Q} instead of
	\eqref{eq:th1_left_bound_presum}, we get that if we
	choose $Q_n\subseteq\KKK_n$ for every $n$ then,
	\[
	\frac12 \sum_{n=0}^N \max_{\pi\in\Pi} \sum_{s\in Q^n}d^{\pi,\gamma}_{s_0}(s)
	\leq \sum_{n=0}^N \epsilon_n
	+ \sum_{s\in\SSS} \beta_{\widetilde{n}(s)}
	+ \sum_{s\in\SSS} \max_{\pi\in\Pi} d^{\pi,\gamma}_{s_0}(s).
	\]
	If we choose $\beta_n=\beta$ for some constant $\beta$ (in \texttt{R-MAX} style),
	the sets $\KKK_n$ are non-increasing by inclusion, namely
	$\KKK_{n+1}\subseteq \KKK_n$ for every $n$. We can than set
	$Q^n=\KKK_N$ for any $n\in\{0,\dots,N\}$ and this leads to
	\begin{equation}\label{crucial_bound_exp}
	\max_{\pi\in\Pi} \sum_{s\in \KKK_N}  d^{\pi,\gamma}_{s_0}(s)
	\leq\frac2N \Bigl(\sum_{n=0}^{N-1} \epsilon_n
	+ \beta|\SSS|
	+ \sum_{s\in\SSS}\max_{\pi\in\Pi} d^{\pi,\gamma}_{s_0}(s)\Bigr).
	\end{equation}
\end{remark}

\begin{theorem}%\label{thm:cereinforce}
	Let an MDP $\MMM=(\SSS,s_0,\AAA,r,\PPP)$, a $s_0$-reset model and a discount factor $\gamma$ be given. Set \texttt{opt}=\REINFORCE, tuned according to \cite[Theorem~6]{ZKOSER}, and perform it for $i(n)$ episodes at each step $n$, with $\delta_n:=\delta/N$.
	
	Then for every step $N$, with probability at least $1-\delta$ we have
	\begin{equation}%\label{bound2}
	\begin{aligned}
	\sum_{n=0}^N \max_{\pi\in\Pi} \sum_{s\in\KKK_n} d^{\pi,\gamma}_{s_0}(s)
	&\leq
	2
	C \frac{|\SSS|^2|\AAA|^2}{(1-\gamma)^2}\log N
	\sum_{n=0}^N  \frac{\log(i(n)/\delta)^{5/2}}{i(n)^{1/6}}
	%\norm{\frac{d_{s_0}^{\pi^*,\gamma}}{\mu_n}}_\infty^2\\
	\max_{\pi\in\Pi} \norm{\frac{d_{s_0}^{\pi,\gamma}}{\mu_n}}_\infty^2\\
	&\quad
	+
	2\sum_{s\in\SSS}\beta_{\widetilde{n}(s)}
	+
	2\sum_{s\in\SSS}\max_{\pi\in\Pi} d_{s_0}^{\pi,\gamma}(s).
	\end{aligned}
	\end{equation}
	%where $C$ is a constant that is indipendent on the problem.
\end{theorem}
\begin{proof}
	First recall that from Theorem~\ref{thm:ceprovable}, at
	step $n$ of the exploration algorithm, the approximation
	error $\epsilon_n$ of the PAC solver is an upper bound
	of the error measured with respect to $\mu_n$, namely, 
	\[
	%\label{eq:eps_optimality_step_n}
	\max_{\pi\in\Pi} V_{\MMM_n,\pi,\gamma}(\mu_n) - V_{\MMM_n,\pi_{n+1},\gamma}(\mu_n)
	\leq \epsilon_n
	\]
	Theorem~6 of \cite{ZKOSER} shows that \REINFORCE with
	a $\mu$ reset sampling model, after $i$ episodes,
	returns a policy $\pi$ such that, with probability
	at least $1-\delta$,
	\begin{equation}\label{eq:boundreinforceappendix}
	V_{\pi^*,r,\PPP,\gamma}(\mu) - V_{\pi,r,\PPP,\gamma}(\mu)
	\leq C \frac{|\SSS|^2|\AAA|^2}{(1-\gamma)^2}
	\log(i/\delta)^{5/2}\frac{1}{i^{1/6}}
	\norm{\frac{d_\mu^{\pi^*,\gamma}}{\mu}}_\infty^2,
	\end{equation}
	where $\pi^*$ is a policy that maximises the value
	function and $C$ is a universal constant,
	independent from the specific MDP.
	
	By setting $\delta_n:=\delta/N$, we have that,
	with probability at least $1-\sum\delta_n=1-\delta$,
	the bound \eqref{eq:boundreinforceappendix} holds at each
	step $n$ and provides an explicit value for $\epsilon_n$.
	Thus, by replacing the values of $\epsilon_n$
	given by \eqref{eq:boundreinforceappendix} in \eqref{boundappendix},
	\begin{align*}
	\sum_{n=0}^N \max_{\pi\in\Pi} \sum_{s\in\KKK_n} d^{\pi,\gamma}_{s_0}(s)
	&\leq 2 C \frac{|\SSS|^2|\AAA|^2}{(1-\gamma)^2}
	\sum_{n=0}^N  \frac{(\log(i(n)/\delta_n)^{5/2}}{i(n)^{1/6}} \norm{\frac{d_{\mu_n}^{\pi^*,\gamma}}{\mu_n}}_\infty^2\\
	&\quad + 2\sum_{s\in\SSS}\beta_{\widetilde{n}(s)}
	+ 2\sum_{s\in\SSS}\max_{\pi\in\Pi} d_{s_0}^{\pi,\gamma}(s)\\
	&\leq 2 C \frac{|\SSS|^2|\AAA|^2}{(1-\gamma)^2}
	\sum_{n=0}^N  \frac{\bigl(\log(i(n)/\delta)+\log(N)\bigr)^{5/2}}
	{i(n)^{1/6}} 
	\norm{\frac{d_{\mu_n}^{\pi^*,\gamma}}{\mu_n}}_\infty^2\\
	&\quad + 2\sum_{s\in\SSS}\beta_{\widetilde{n}(s)}
	+ 2\sum_{s\in\SSS}\max_{\pi\in\Pi} d_{s_0}^{\pi,\gamma}(s),
	\end{align*}
	and in general the term $\log N$ is negligible if $i(n)$
	is sufficiently large with respect to $N$.
	
	Finally, from Lemma~\ref{l:boh} we get
	\[
	\norm{\frac{d_{\mu_n}^{\pi^*,\gamma}}{\mu_n}}_\infty
	= \max_{s\in\SSS} \frac{d_{\mu_n}^{\pi^*,\gamma}(s)}{\mu_n(s)}
	\leq \max_{s\in\SSS} \frac{\displaystyle\max_{\pi\in\Pi}d_{s_0}^{\pi,\gamma}(s)}{\mu_n(s)}
	\leq \max_{\pi\in\Pi} \norm{\frac{d_{s_0}^{\pi,\gamma}}{\mu_n}}_\infty,
	\]
	and this completes the proof.
\end{proof}
\begin{theorem}%\label{thm:reinforcece}
	Let an MDP $\MMM=(\SSS,s_0,\AAA,r,\PPP)$, a $s_0$-reset model
	and a discount factor $\gamma$ be given. Perform \texttt{CE}
	to obtain a simulated $\mu_N$-reset model. Perform \REINFORCE
	on $\MMM=(\SSS,\mu_N,\AAA,r,\PPP)$ and call $\pi$ the policy
	returned after $i$ episodes. Then with probability at least
	$1-\delta$ it holds
	\begin{equation}
	V_{\pi^*,r,\PPP,\gamma}(s_0) - V_{\pi,r,\PPP,\gamma}(s_0)
	\leq  C \frac{|\SSS|^2|\AAA|^2}{(1-\gamma)^2}
	\frac{\log(i/\delta)^{5/2}}{i^{1/6}}
	\norm{\frac{d_{s_0}^{\pi^*,\gamma}}{\mu_N}}_\infty^2.
	\end{equation}
\end{theorem}
\begin{proof}
	The proof of Theorem~6 of \cite{ZKOSER} can be slightly modified, without
	too much efforts, so that by having access to a $\mu$-reset model,
	for every other distribution $\rho$,
	\[%\label{eq:boundreinforce_generalized}
	V_{\pi^*,r,\PPP,\gamma}(\rho) - V_{\pi,r,\PPP,\gamma}(\rho)
	\leq C \frac{|\SSS|^2|\AAA|^2}{(1-\gamma)^2}\log(i/\delta)^{5/2} \frac{1}{i^{1/6}}
	\norm{\frac{d_\rho^{\pi^*,\gamma}}{\mu}}_\infty^2,
	\]
	and this is sufficient to prove Theorem 3.
\end{proof}

\section{Experiments}\label{appendix:experiments}

We tested the exploration behaviour of CE with two different optmizers: \REINFORCE and \texttt{TRPO}. We implemented them as Python functions that take an MDP, a reset model and a number of timesteps as input, and return the nearly-optimal policy obtained. Those are two extremes of \texttt{opt} function that can be paired which CE. The first one satisfies good theoretical convergence guarantees \cite{ZKOSER}, the second is one of the best empirical RL solver as of today. The original \texttt{TRPO} implementation is taken from \cite{KosPIT}. The source code for experiments is available as a Jupyter Notebook in supplementary material, and will be soon available as a GitHub repository.

We chose two MDPs with a very challenging exploration because ot their sparse rewards: the \emph{Consecutive Crossroads Traps} CCT \cite[Figure 2]{AKLTPG} and the \emph{Diabolical Combination Lock} DCL \cite[Figure 2]{AHKPPC}. We implemented them as superclasses of OpenAI Gym \cite{OpeGYM}, using a Python framework \texttt{blackhc.mdp} for creating custom MDPs \cite{KirMDP}. For both of them, we used three different depths of 5, 10 and 20, and we refer to them as CCT5, CCT10, CCT20, DCL5, DCL10, DCL20. Notice that for exploration, depth 5 is very easy while depth 20 is very challenging. For a given depth, DCL is harder that CCT, because the locked path in DCL makes impossible to recover from mistakes, see Figures Figure \ref{fig:CCT} and Figure \ref{fig:DCL}.

In Section \ref{sec:expeff} we describe the exploration efficiency of \texttt{CE($s_0$,\REINFORCE)} and \texttt{CE($s_0$,\texttt{TRPO})}, and in Section \ref{sec:impeff} we describe the learning performance of \REINFORCE and \texttt{TRPO} when the corresponding CE output is used as a reset model. For each MDP among CCT5, CCT10, CCT20, DCL5, DCL10, DCL20, we provide graphs of the evolution of the visit distribution (pages \pageref{page:evolution_reinforcecct} through \pageref{page:scatterplots_trpodcl}) and of the learning performance (pages \pageref{page:COMP_reinforcecct} through \pageref{page:COMP_trpodcl}).

%In the submitted paper we already included results from \texttt{CE($s_0$,TRPO)} on DCL20. We restate them in this section for completeness.

%
%Per ciascuno dei 3 (d5, d10, d20) proviamo sia REINFORCE sia TRPO. Per ognuno di questi sei casi ci sono:\\
%-sequenza di grafi colorati verde-rosso con visit distr\\
%-scatterplot dell'andamento delle visit distr
%-comparison original problem with/without exploration

\subsection{MDPs and hyperparameters}

The MDP pictured in Figure \ref{fig:CCT} is called \emph{Consecutive Crossroad Traps} CCT. It is a challenging MDP for exploration because at every state only one out of four actions takes one step forward, closer to the reward appearing at the terminal state, while every other action takes one step back.
%\mau{For a random policy, the average termination time is exponential in the depth $d$.}{MM o MR, è corretta questa frase? Mi serve qualcosa del genere per dire che questo MDP è challenging per l'esplorazione}.
CCT is an example of MDP where sample based estimates of gradients will be zero under random exploration, because the probability of reaching the terminal state is exponentially small in the depth $d$ \cite[Remark 4.1]{AKLTPG}. This makes CCT a challenging MDP for large $d$. In the experiments we used CCT of depth 5 (very easy), 10 and 20 (hard).

The MDP pictured in Figure \ref{fig:DCL} is called \emph{Diabolical Combination Lock} DCL. There are 3 paths called A, B and locked. Trajectories have constant length $d+1$, because every action goes exactly one step forward. Reward is $1$ in the terminal states of path A and path B, and 0 otherwise. DCL is a stochastic MDP, with transitions shown in Figure \ref{fig:DCL}. Since two out of four actions take to the locked path, a random policy has probability $1/2$ of getting trapped at each step, so has probability $(1/2)^{d+1}$ of obtaining non-zero reward, where $d$ is the depth of DCL (usually the first step is not considered in the depth of DCL). When $d>>0$, the reward becomes very sparse and the exploration very hard. In the experiments we used DCL of depth 5 (easy), 10 and 20 (very hard).

We used a discount factor $\gamma=0.95$ and a linear schedule $\beta_n:=\beta\cdot n$, in order to simulate a distribution (the average of the $\mu_n$) that is pointwise greater than $\beta$. We fix $\beta:=\frac{1}{2|\SSS|}$, same order of magnitude as the uniform distribution. We iterate CE for $d$ steps, thus producing as many exploratory policies as the depth of the MDP. While clearly the more steps the better, choosing steps of the same order of magnitude as the number of states seems reasonable when looking at the inequality in Theorem 1 divided by $N$. In fact, the choice $N\sim|\SSS|$ spreads the exploitative factor error term among all states, and so the error that remains has the same order of magnitude as the the maximum visit of one state, and is therefore hopefully negligible.

\subsection{Exploration efficiency of \texttt{CE($s_0$,opt)}}\label{sec:expeff}

%The diabolical combination lock instance we used has 3 paths of depth $d$, one of which locked, and a single starting state. The actions are 4, with $a_0$, $a_1$ switching paths with probability $\alpha$ and $1-\alpha$, and $a_2$, $a_3$ taking to the locked path. See Figure \ref{fig:diablock}.

Exploration efficiency of \texttt{CE($s_0$,opt)} can be understood by looking at the evolution of the simulated $\mu_n$-reset models, for $n=0,\dots,d-1$, see figures on page \pageref{page:evolution_reinforcecct} through \pageref{page:evolution_trpodcl}. In the spirit of the image blurring analogy, we color the poorly visited states $\KKK$ from green for the most visited (lighter pixel) to red for the less visited (darker pixel), proportionally to the visit distribution value. States in $\SSS-\KKK$ are green.
The explorations are satisfactory, with the exception of the exploration of \texttt{CE($s_0$,\REINFORCE)} on CCT20 (end of page \pageref{page:evolution_reinforcecct}). Since CCT20 is simpler than DCL20, the expectation would be to observe a more successful exploration. However, the low performance may be explained by the higher simplicity: indeed, starting from states other than $s_0$ does not help much, while it makes learning from $s_0$ more challenging.

%While almost all the explorations are satisfactory, the exploration of \texttt{CE($s_0$,\REINFORCE)} on CCT20 (end of page \pageref{page:evolution_reinforcecct}) is particularly disappointing, especially considering that CCT20 is simpler than DCL20, where exploration is much more successful (page \pageref{page:evolution_reinforcedcl}).
%%However, a similar low performance is obtained by \REINFORCE without CE
%This simplicity can be the actual reason for the low performance: starting from states other than $s_0$ does not help much, while it makes learning from $s_0$ more challenging.

Pages \pageref{page:evolution_reinforcecct} through \pageref{page:evolution_trpodcl} show only few snapshots of the whole sequence of policies $\mu_0,\mu_1\dots,\mu_{d-1}$. We included in supplementary material \texttt{exploration.mp4}, a video showing all exploration snapshots, each every second, starting from the simple CCT5 up to DCL20 for both \REINFORCE and \texttt{TRPO}.

Pages \pageref{page:scatterplots_reinforcecct} through \pageref{page:scatterplots_trpodcl} show a different representation of exploration evolution for $\{\mu_n\}_{n=0,\dots,d-1}$: scatterplots with states $s$ on the $x$-axis and values of $\mu_n(s)$ on the $y$-axis. Steps $n$ are represented with colors, ranging from red ($n=0$) to blue ($n=d-1$), and a green line shows the uniform distribution for which $\mu(s)=1/|\SSS|$. Since $\beta=1/(2|\SSS|)$, poorly visited states are below half the green line. The evolution of each state, from red to blue, should approach half the green line from below, because exploration points towards poorly visited states, and from above, because in that case the average of the visits will decrease. Thus, ideally blue points should be as close to half the green line as possible. Again, exploration of \texttt{CE($s_0$,\REINFORCE)} on CCT20 appears to be challenging (end of page \pageref{page:scatterplots_reinforcecct}).

\subsection{Learning performance of \texttt{opt} with \texttt{CE($s_0$,opt)}}\label{sec:impeff}

The optimizer \texttt{opt} (\REINFORCE or \texttt{TRPO}) can be used on CCT or DCL with or without the simulated \texttt{CE($s_0$,opt)}-reset model. To measure whether the performance of \texttt{opt} improves when exploration is done via the \texttt{CE($s_0$,opt)}-reset model, we look at the return, averaged over episodes. Notice that we use a discounted return for CCT, where episodes can be of different lengths, while we use the undiscounted return for DCL, where episodes are of constant length $d+1$. Results shown on pages \pageref{page:COMP_reinforcecct} through \pageref{page:COMP_trpodcl} are an average made over 10 different independent runs. The solid line is the average, the light area is one standard deviation.

We point out that since we switched from discounted return to \emph{undiscounted} return for performance comparisons on DCL, Figure 3 at page 9 of the submitted paper is here replaced by the last figure at the end of page \pageref{page:COMP_trpodcl}.

On simple tasks like CCT5 and DCL5 the additional exploration provided by \texttt{CE($s_0$,opt)} does not appear to improve the learning phase. Increasing depth to $d=10$, we see a moderate improvement for \REINFORCE on CCT10 (page \pageref{page:COMP_reinforcecct}) and a definite improvement for the other pairings \REINFORCE-DCL10, \texttt{TRPO}-CCT10 and \texttt{TRPO}-DCL10. Further increasing the depth up to $d=20$ gives a big improvement in the learning phase of \texttt{TRPO} on CCT20 and DCL20 (pages \pageref{page:COMP_trpocct}, \pageref{page:COMP_trpodcl}). However, \REINFORCE appears to struggle with both CCT20 and DCL20: no reward is obtained, with or without the help of CE for the exploration (pages \pageref{page:COMP_reinforcecct}, \pageref{page:COMP_reinforcedcl}).

%Due to high computational time, it has not been possible to complete the scatterplot and the comparison for \texttt{TRPO} on CCT20 in time for this submission.

%\mau{A policy value improvement never happens on}{Stessa cosa, questa parte va cambiata quando abbiamo tutti i grafici} the first 1000 episodes when starting from $\texttt{start}$. On the opposite, a policy value improvement happens after less than 500 episodes when starting from $s\sim\texttt{CE(\texttt{start},TRPO)}$. This is not surprising: the probability of obtaining a positive reward from \texttt{start} is $1/2^20$, and 1000 episodes are not enough. No optimization algorithm can start learning before receiving a non-zero reward.

\clearpage

%\begin{figure}[ht]
%	\begin{subfigure}{.5\textwidth}
%		\includegraphics[width=\linewidth]{img/cross_D5.png}
%	\end{subfigure}%
%	\begin{subfigure}{.35\textwidth}
%		\includegraphics[width=\linewidth]{img/cross_D10.png}
%	\end{subfigure}%
%	\begin{subfigure}{.2\textwidth}
%		\includegraphics[width=\linewidth]{img/cross_D20.png}
%	\end{subfigure}
%	\caption{Consecutive Crossroad Traps with depth $d=5$ (left), $d=10$ (center), $d=20$ (right). From the exploration point of view, $d=5$ is very easy while $d=20$ is challenging.}
%	\label{fig:CCT}
%\end{figure}

\begin{figure}[ht]
	\centering
	\includegraphics[width=0.5\linewidth]{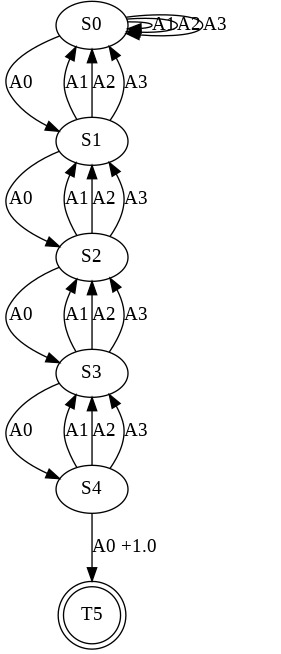}
	\caption{Consecutive Crossroad Traps CCT with depth $d=5$. Reward is 0 except when entering the terminal state \texttt{T5}, where reward $=1$. In all but the terminal state one action takes closer to \texttt{T5}, while every other action takes one step back,  farther from \texttt{T5}. The reward is sparse, making CCT a challenging MDP for exploration issues. We used $d=5,10$ and $20$ in the experiments. From the exploration point of view, $d=5$ is very easy while $d=20$ is challenging.}\label{fig:CCT}
\end{figure}

\clearpage

%\begin{figure}[ht] 
%	\begin{subfigure}{.6\textwidth}
%		\includegraphics[width=\linewidth]{img/lock_D5.png}
%	\end{subfigure}%
%	\begin{subfigure}{.5\textwidth}
%		\includegraphics[width=\linewidth]{img/lock_D10.png}
%	\end{subfigure}%
%	\begin{subfigure}{.4\textwidth}
%		\includegraphics[width=\linewidth]{img/lock_D20.png}
%	\end{subfigure}
%	\caption{Diabolical Combination Lock DCL with switching probability $0.2$ and depth $d=5$ (left), $d=10$ (center), $d=20$ (right).
%		%We used $d=20$ in the experiment.
%		The starting state distribution $\rho$ is $s_0=\texttt{start}$. Reward is $1$ in \texttt{end A}, \texttt{end B}, and 0 in \texttt{locked end}. In the experiments we used a discount factor $\gamma=0.95$. We used a linear schedule $\beta_n:=\beta\cdot n$, in order to simulate a distribution (the average of the $\mu_n$) that is pointwise greater than $\beta$. We fix $\beta:=\frac{1}{2|\SSS|}$, same order of magnitude as the uniform distribution. For a given depth, DCL is harder that CCT, because the locked path in DCL makes impossible to recover from mistakes}\label{fig:diablock}
%\end{figure}

\begin{figure}[ht]
	\centering 
	\includegraphics[width=\linewidth]{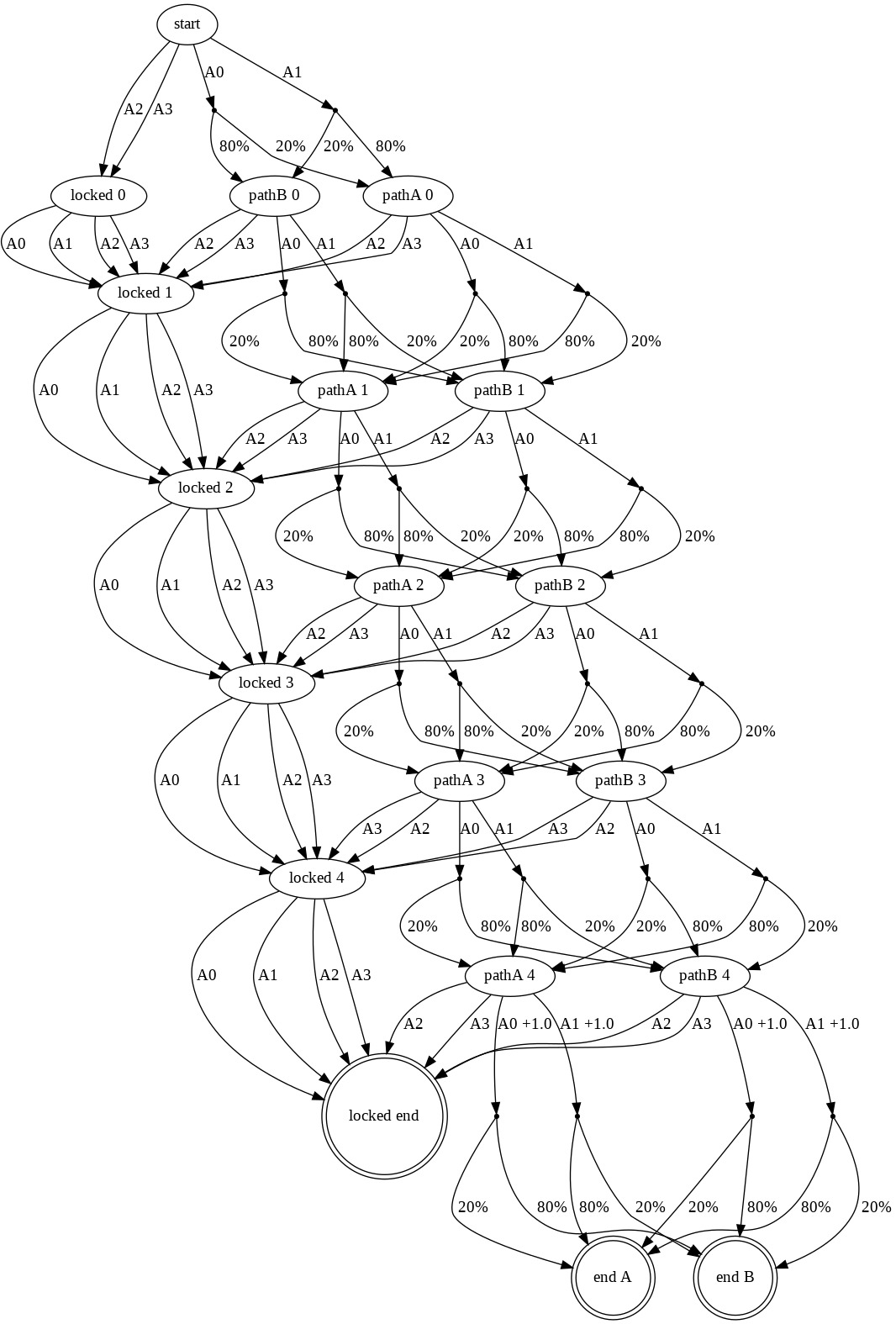}
	\caption{Diabolical Combination Lock DCL with switching probability $0.2$ and depth $d=5$. There are 3 paths called \texttt{pathA}, \texttt{pathB} and \texttt{locked}. The starting state distribution $\rho$ is $s_0=\texttt{start}$, and every episode has constant length $d+1$. Reward is $1$ in \texttt{end A} and \texttt{end B}, and is 0 everywhere else. In particular, there is no reward in the terminal state \texttt{locked end}. For a given depth, DCL is harder that CCT, because the locked path in DCL makes impossible to recover from mistakes. When $d>>0$, the reward becomes very sparse and the exploration very hard. In the experiments we used DCL of depth 5 (easy), 10 and 20 (very hard).}\label{fig:DCL}
\end{figure}

\clearpage

%\caption{Visit distribution $\mu_n$ simulated by CE for $n= 1, 3, 6, 10, 20$ steps respectively. Colors range from green for the most visited to red for the less visited, proportionally to the visit distribution value. In the supplementary material you can find a video showing the exploration evolution.}\label{fig:evolution}
%\caption{Visit distribution $\mu_n$ simulated by CE for $n= 1, 3, 6, 10, 20$ steps respectively. Colors range from green for the most visited to red for the less visited, proportionally to the visit distribution value. In the supplementary material you can find a video showing the exploration evolution.}\label{fig:evolution}

\newcommand{\dirfig}{./img/SFUM-jpg/} %dir for figures

\begin{tikzpicture}[overlay] %[node distance=3cm,auto]
\label{page:evolution_reinforcecct}
\tikzstyle{format} = [anchor = south]
\path (.5\columnwidth,-.5\textheight) coordinate (center);
\def\vstep{1.12cm}
\def\hstep{4cm}
\def\widthfig{10cm}
\def\vcaption{-0.2cm}
\path (-5cm,1.1cm) coordinate (UL);
\foreach\t in {0,...,10} \foreach\s in {0,...,20} \path (UL) + (\t*\hstep,-\s*\vstep) coordinate (C\t_\s);
%	\foreach\t in {0,...,10} \foreach\s in {0,...,6} \path (C\t_\s) node {(\t,-\s)};

\def\scale{0.23}
\path (C1_1) node[scale=\scale]{\includegraphics[width=\widthfig]{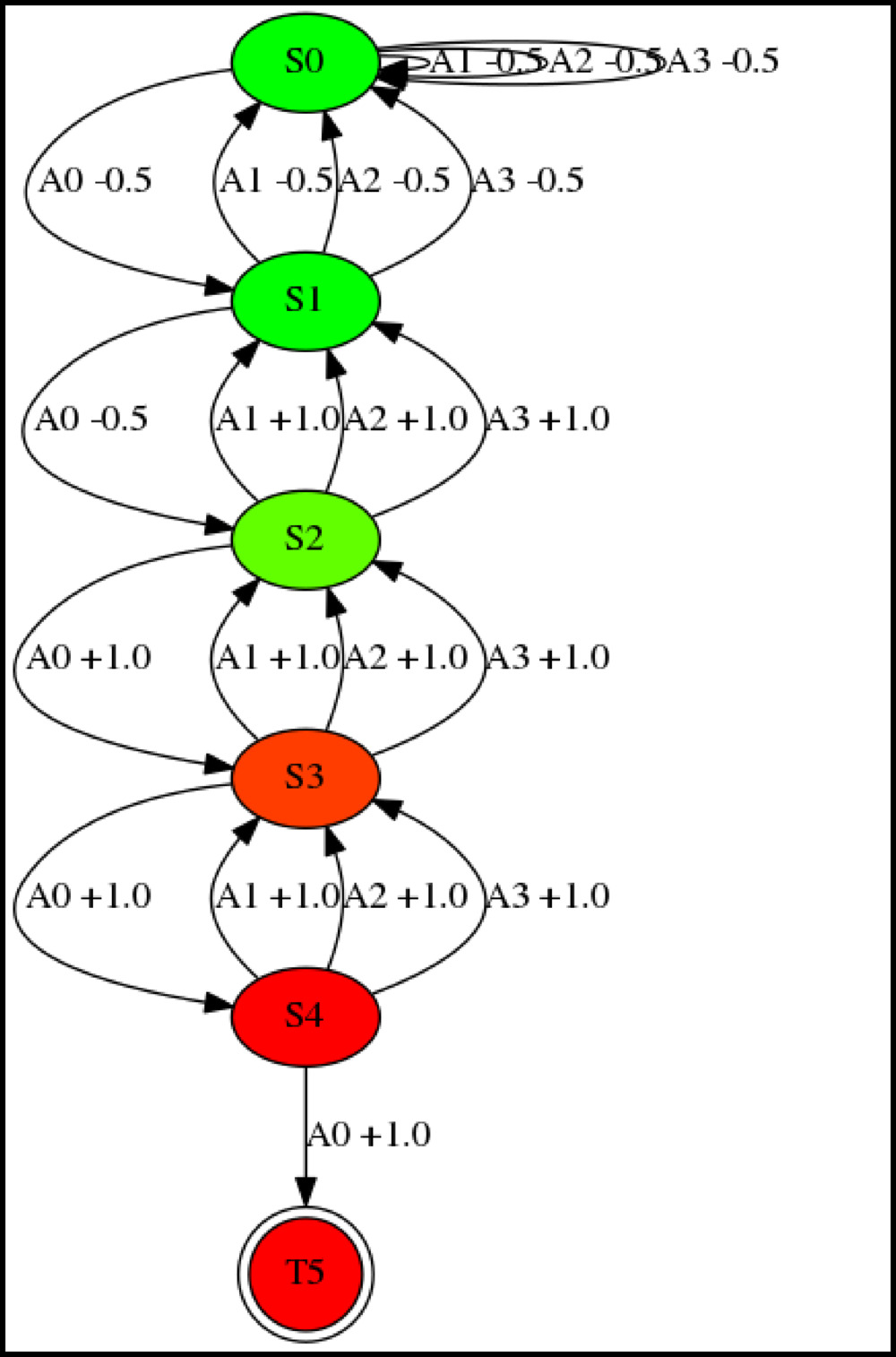}};
\path (C2_1) node[scale=\scale]{\includegraphics[width=\widthfig]{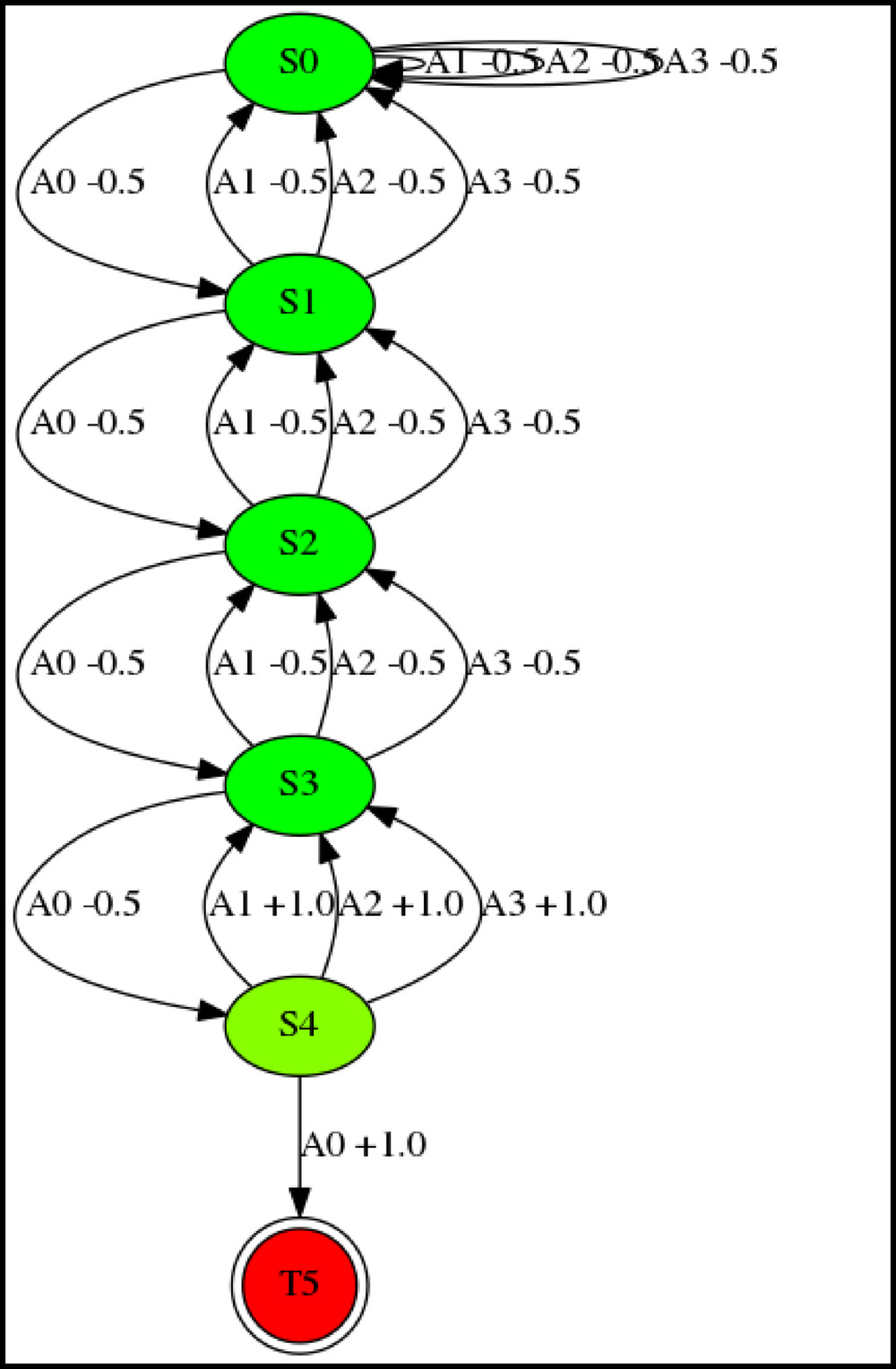}};
\path (C3_1) node[scale=\scale] (central1) {\includegraphics[width=\widthfig]{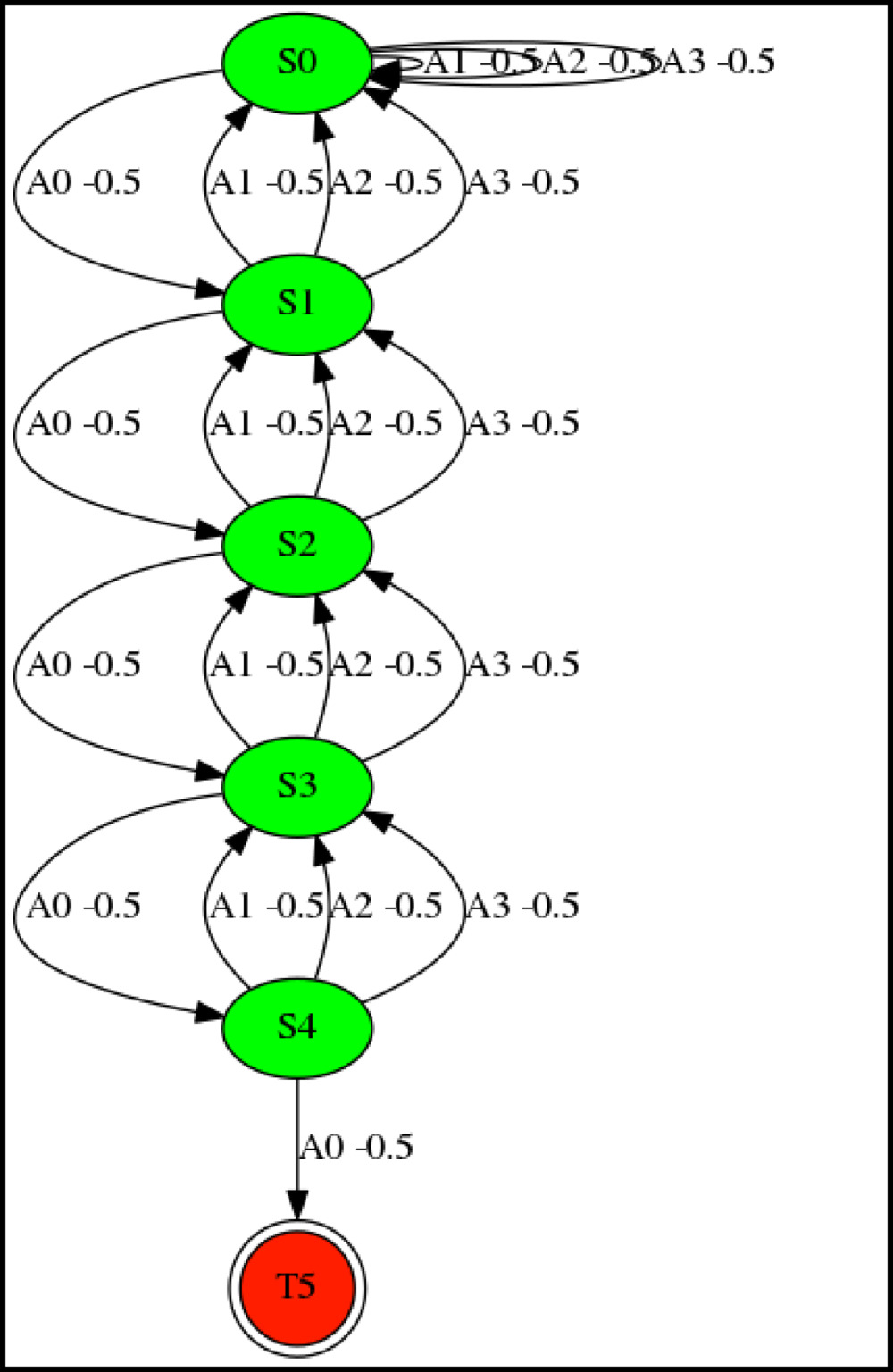}};
\path (C4_1) node[scale=\scale]{\includegraphics[width=\widthfig]{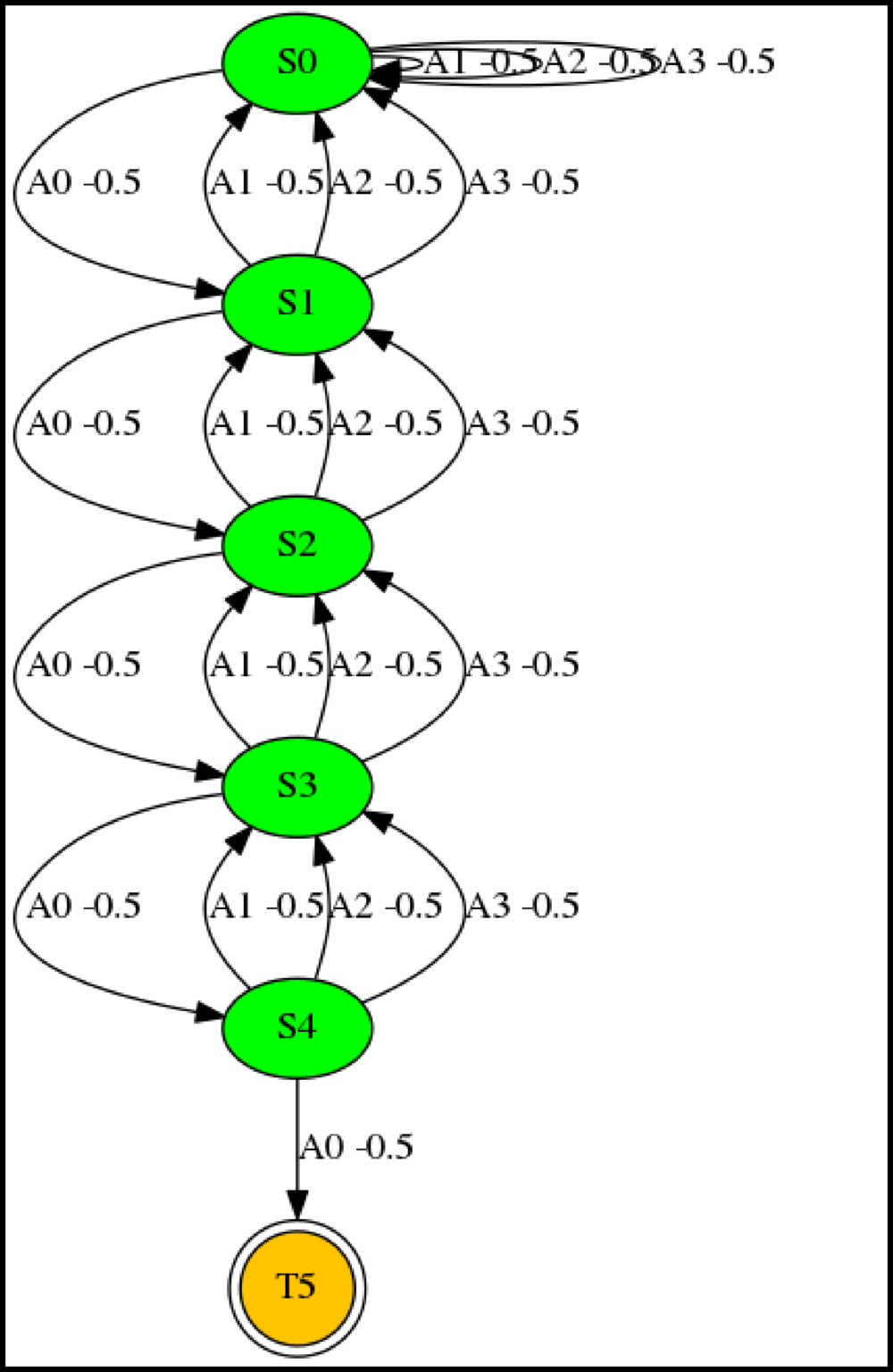}};
\path (C5_1) node[scale=\scale]{\includegraphics[width=\widthfig]{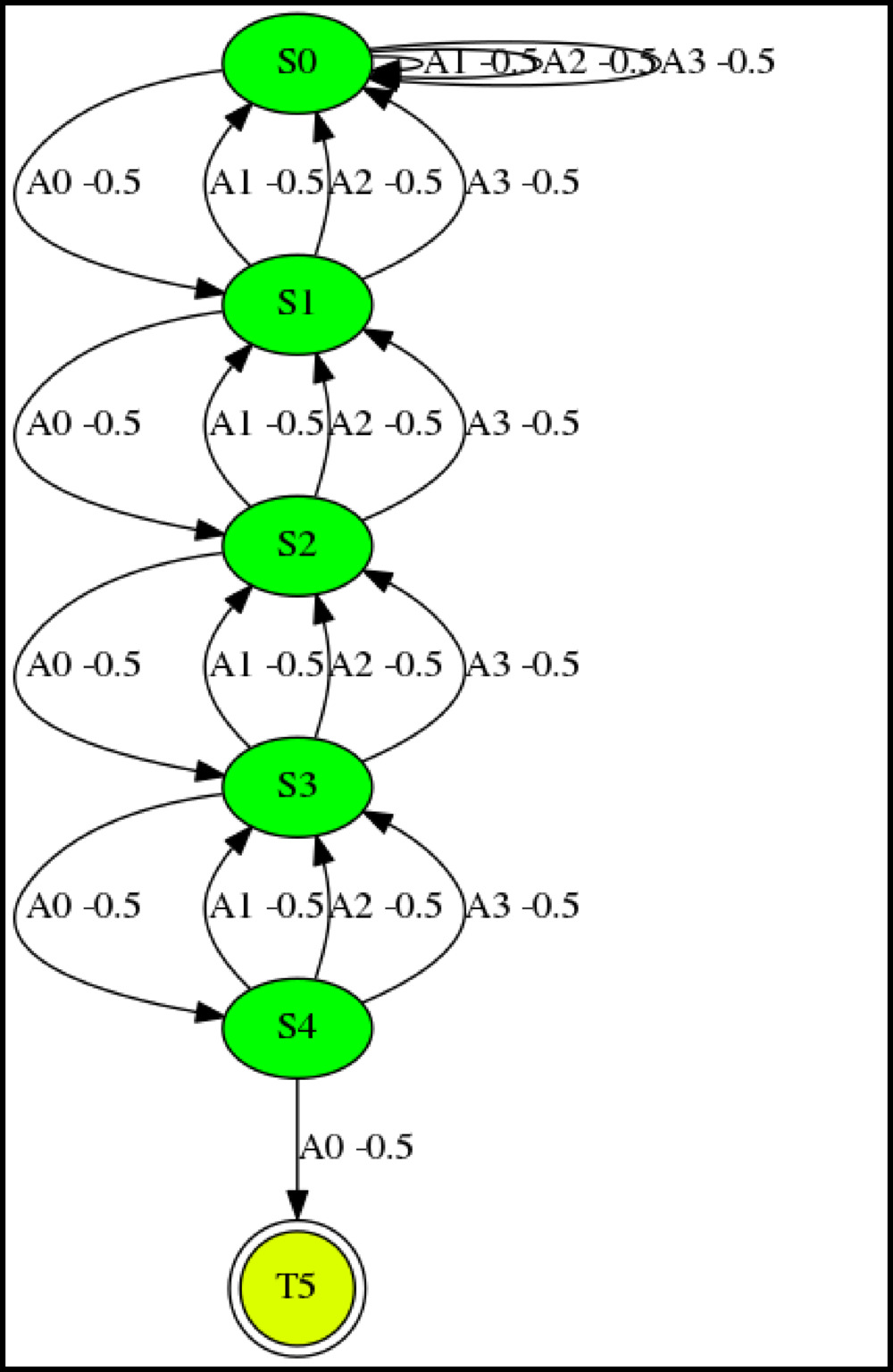}};

\path (C1_7) node[scale=\scale]{\includegraphics[width=\widthfig]{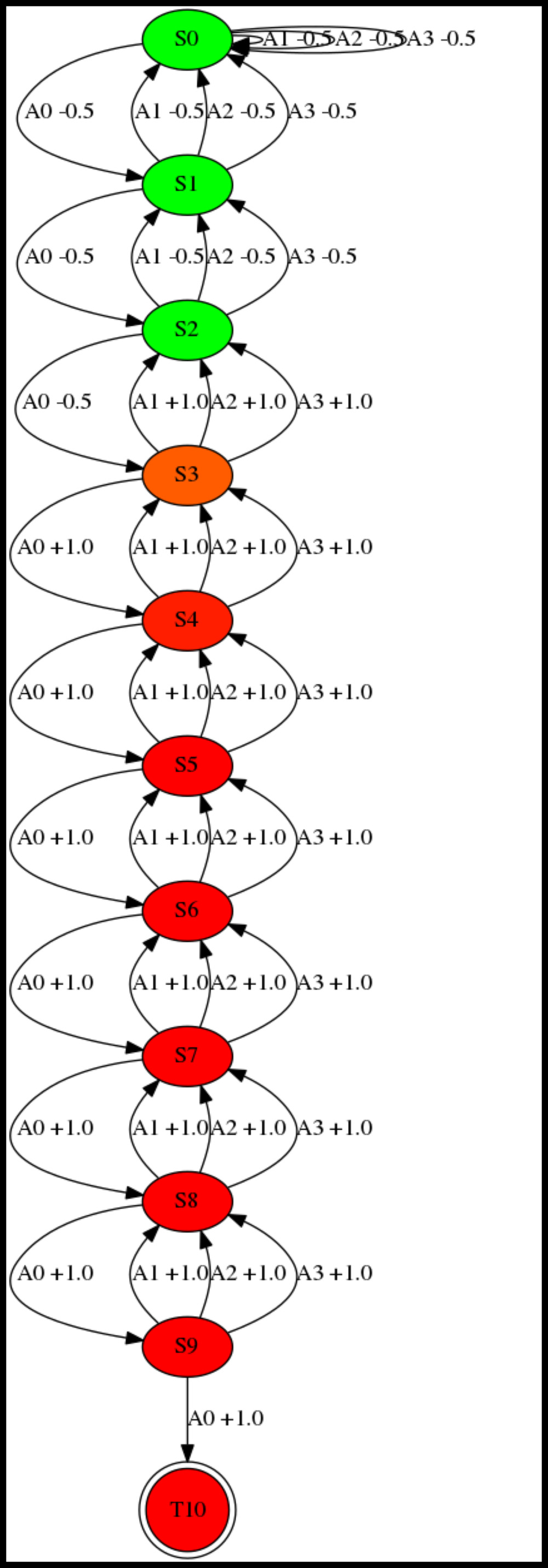}};
\path (C2_7) node[scale=\scale]{\includegraphics[width=\widthfig]{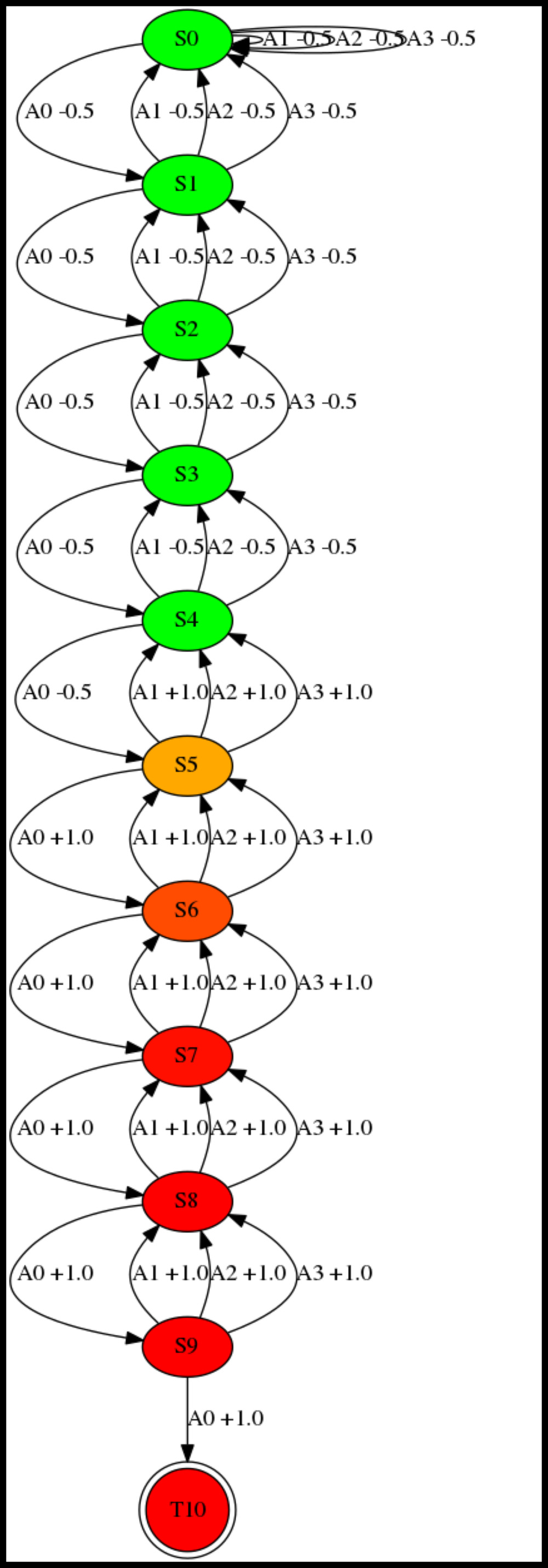}};
\path (C3_7) node[scale=\scale] (central2) {\includegraphics[width=\widthfig]{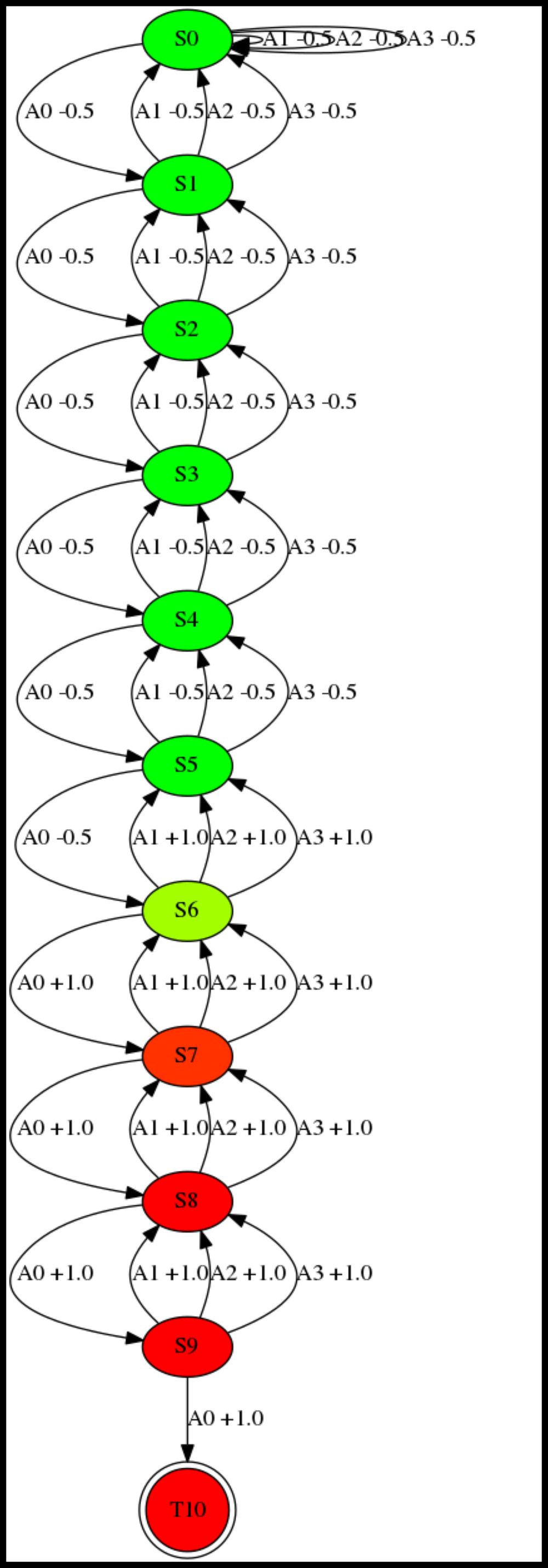}};
\path (C4_7) node[scale=\scale]{\includegraphics[width=\widthfig]{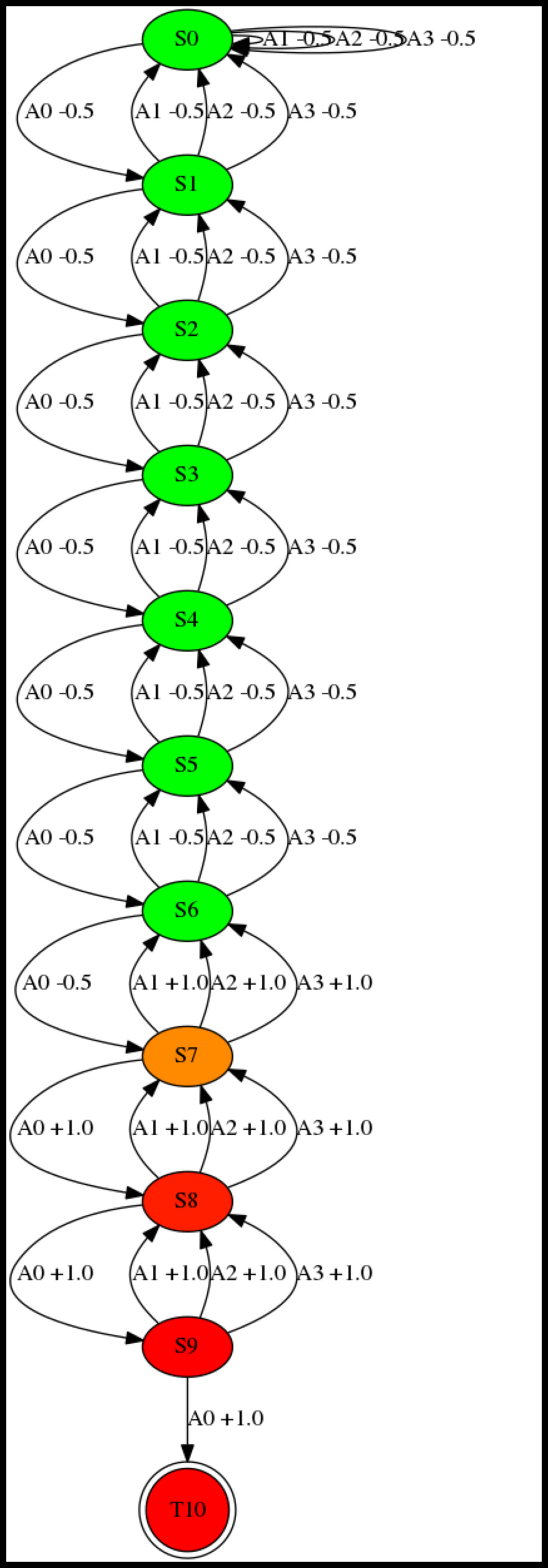}};
\path (C5_7) node[scale=\scale]{\includegraphics[width=\widthfig]{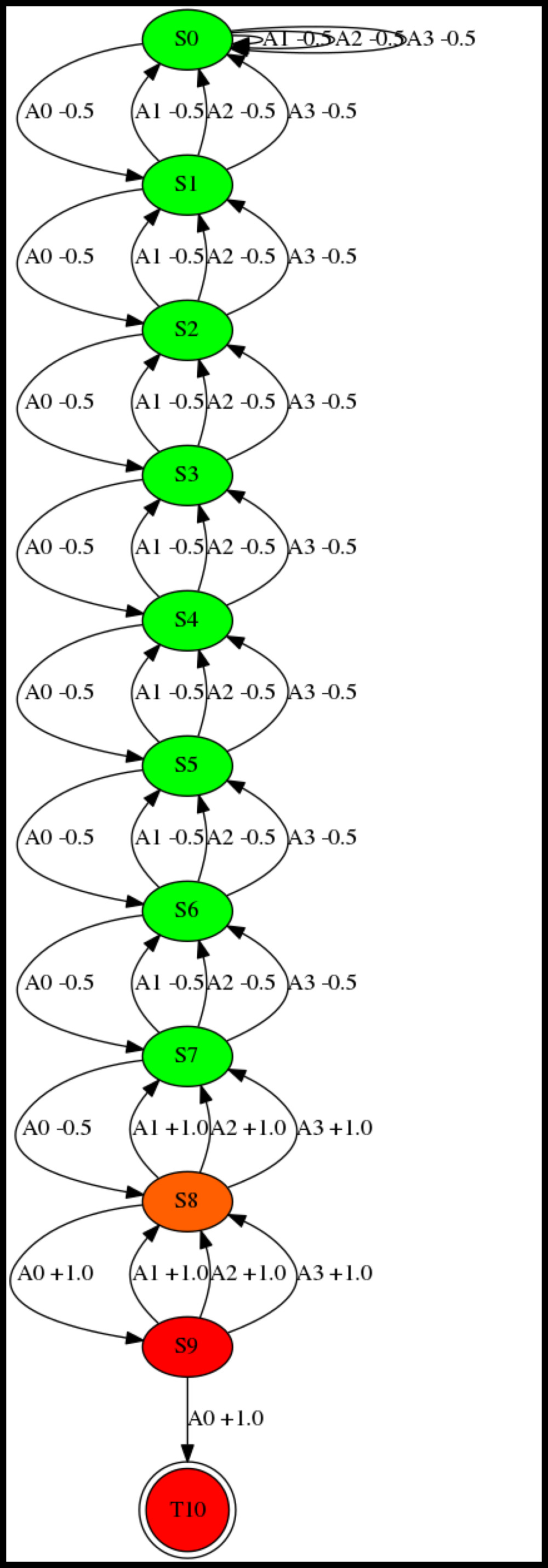}};

\path (C1_17) node[scale=\scale]{\includegraphics[width=\widthfig]{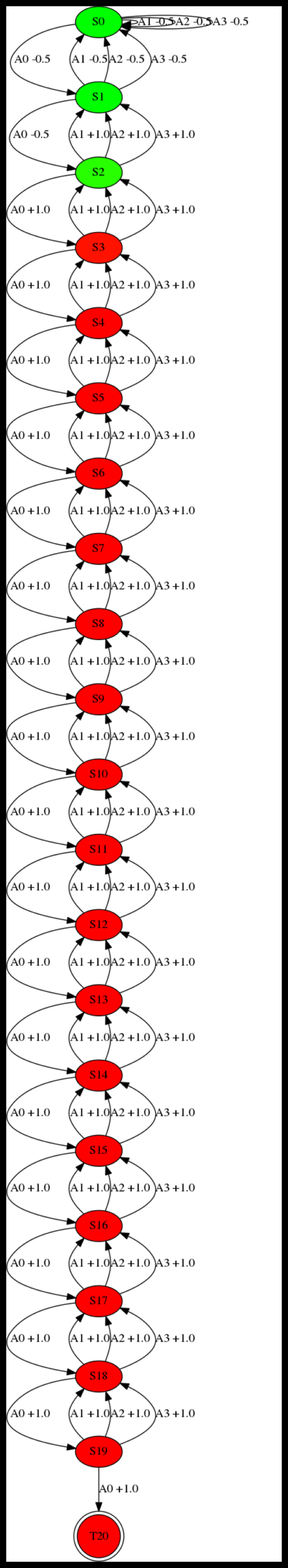}};
\path (C2_17) node[scale=\scale]{\includegraphics[width=\widthfig]{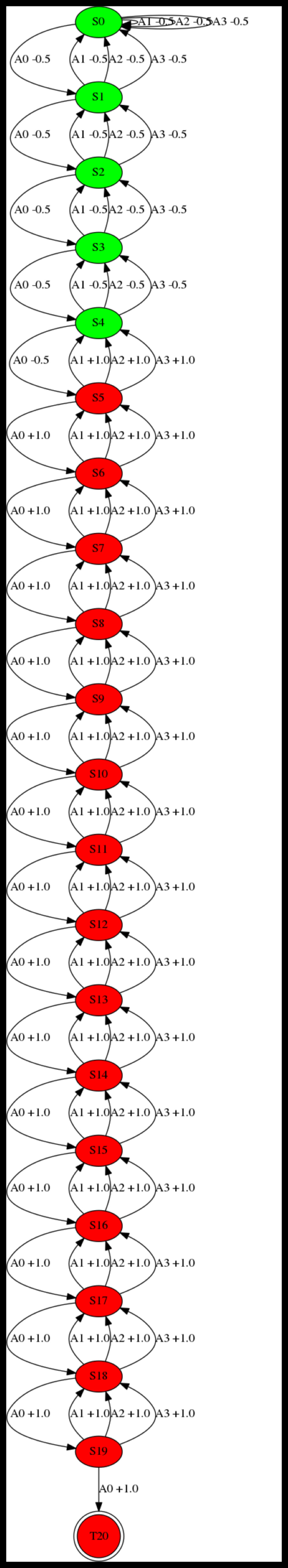}};
\path (C3_17) node[scale=\scale] (central3) {\includegraphics[width=\widthfig]{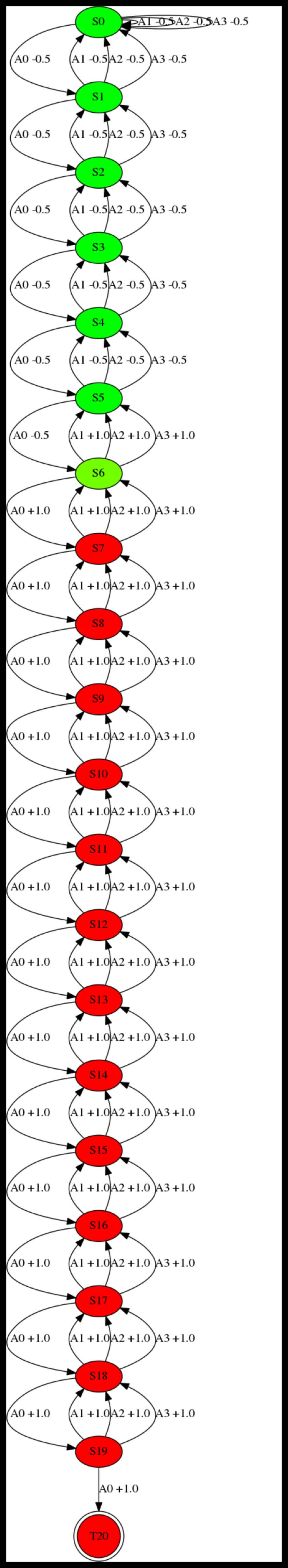}};
\path (C4_17) node[scale=\scale]{\includegraphics[width=\widthfig]{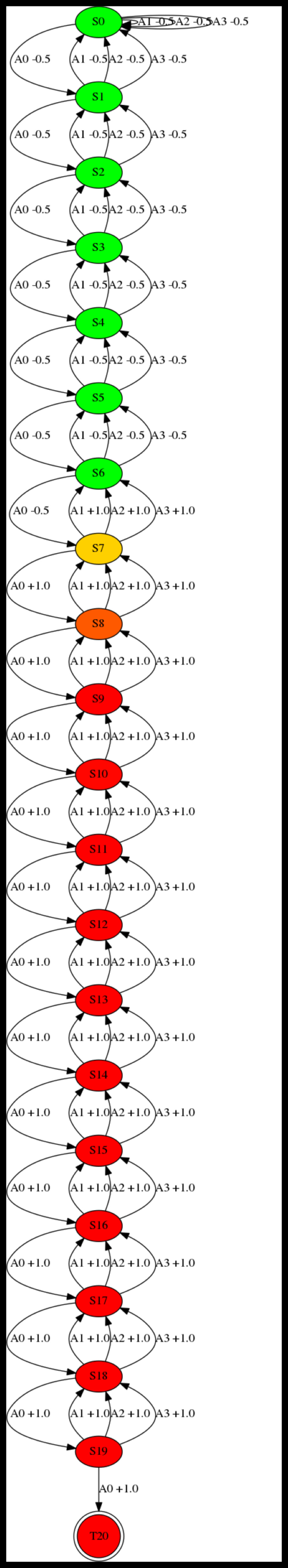}};
\path (C5_17) node[scale=\scale]{\includegraphics[width=\widthfig]{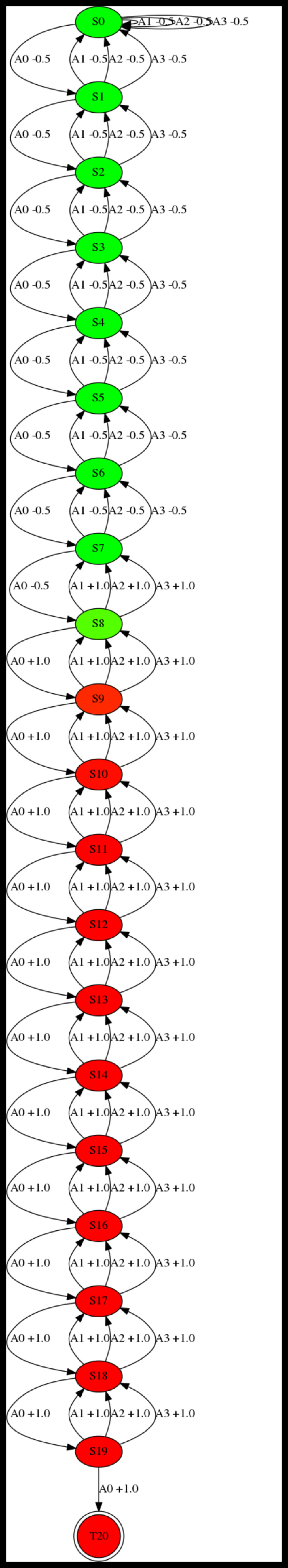}};

\path (central1.north) + (0,-\vcaption) node[format]{\texttt{CE($s_0$,\REINFORCE)} on Consecutive Crossroad Traps, depth $d=5$, iteration $n=1,2,3,4,5$ from left to right.};
\path (central2.north) + (0,-\vcaption) node[format]{\texttt{CE($s_0$,\REINFORCE)} on Consecutive Crossroad Traps, depth $d=10$, iteration $n=1,2,3,4,10$ from left to right.};
\path (central3.north) + (0,-\vcaption) node[format]{\texttt{CE($s_0$,\REINFORCE)} on Consecutive Crossroad Traps, depth $d=20$, iteration $n=1,2,3,6,20$ from left to right.};
\end{tikzpicture}

\clearpage

\begin{tikzpicture}[overlay] %[node distance=3cm,auto]
\label{page:evolution_reinforcedcl}
\tikzstyle{format} = [anchor = south]
\path (.5\columnwidth,-.5\textheight) coordinate (center);
\def\vstep{1.25cm}
\def\hstep{4cm}
\def\widthfig{10cm}
\def\vcaption{-0.1cm}
\path (-5cm,0cm) coordinate (UL);
\foreach\t in {0,...,10} \foreach\s in {0,...,20} \path (UL) + (\t*\hstep,-\s*\vstep) coordinate (C\t_\s);
%	\foreach\t in {0,...,10} \foreach\s in {0,...,6} \path (C\t_\s) node {(\t,-\s)};

\def\xscale{0.3}
\def\yscale{0.6}
\path (C1_1) node[xscale=\xscale,yscale=\yscale]{\includegraphics[width=\widthfig]{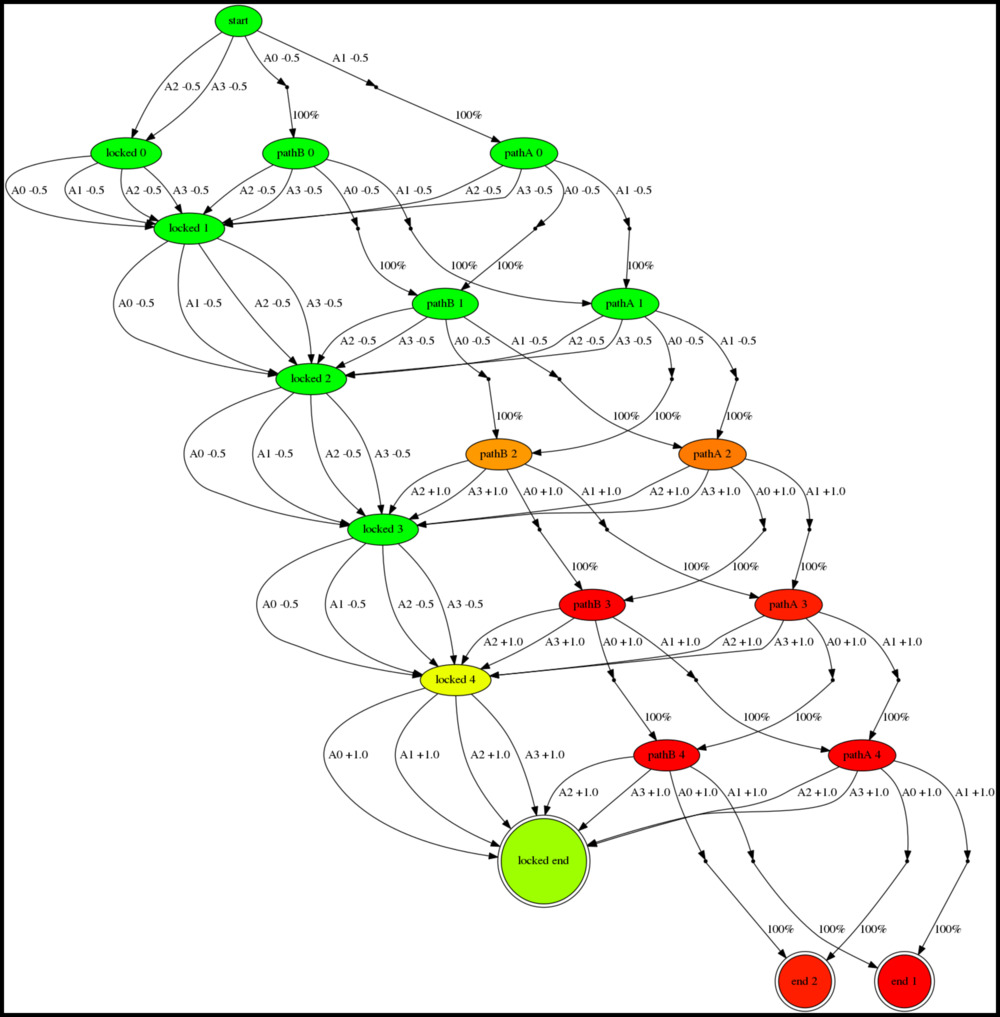}};
\path (C2_1) node[xscale=\xscale,yscale=\yscale]{\includegraphics[width=\widthfig]{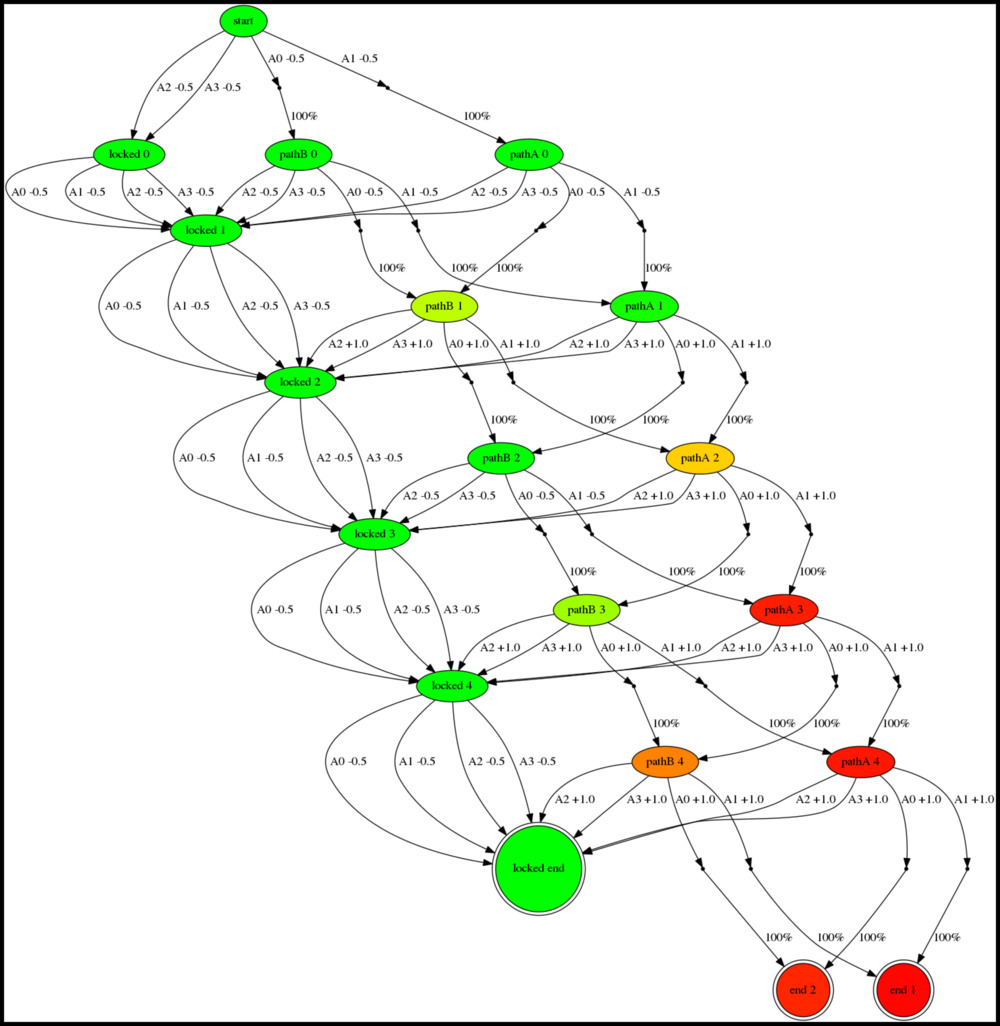}};
\path (C3_1) node[xscale=\xscale,yscale=\yscale] (central1) {\includegraphics[width=\widthfig]{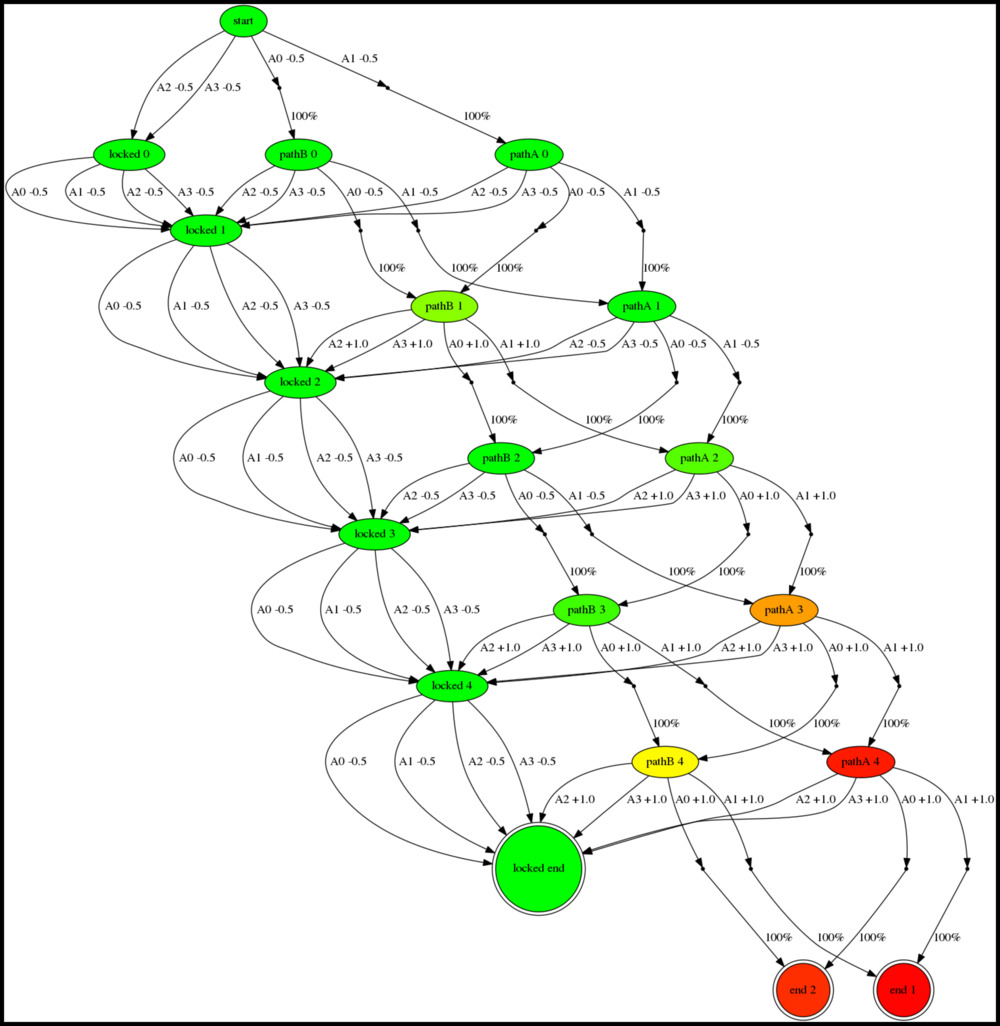}};
\path (C4_1) node[xscale=\xscale,yscale=\yscale]{\includegraphics[width=\widthfig]{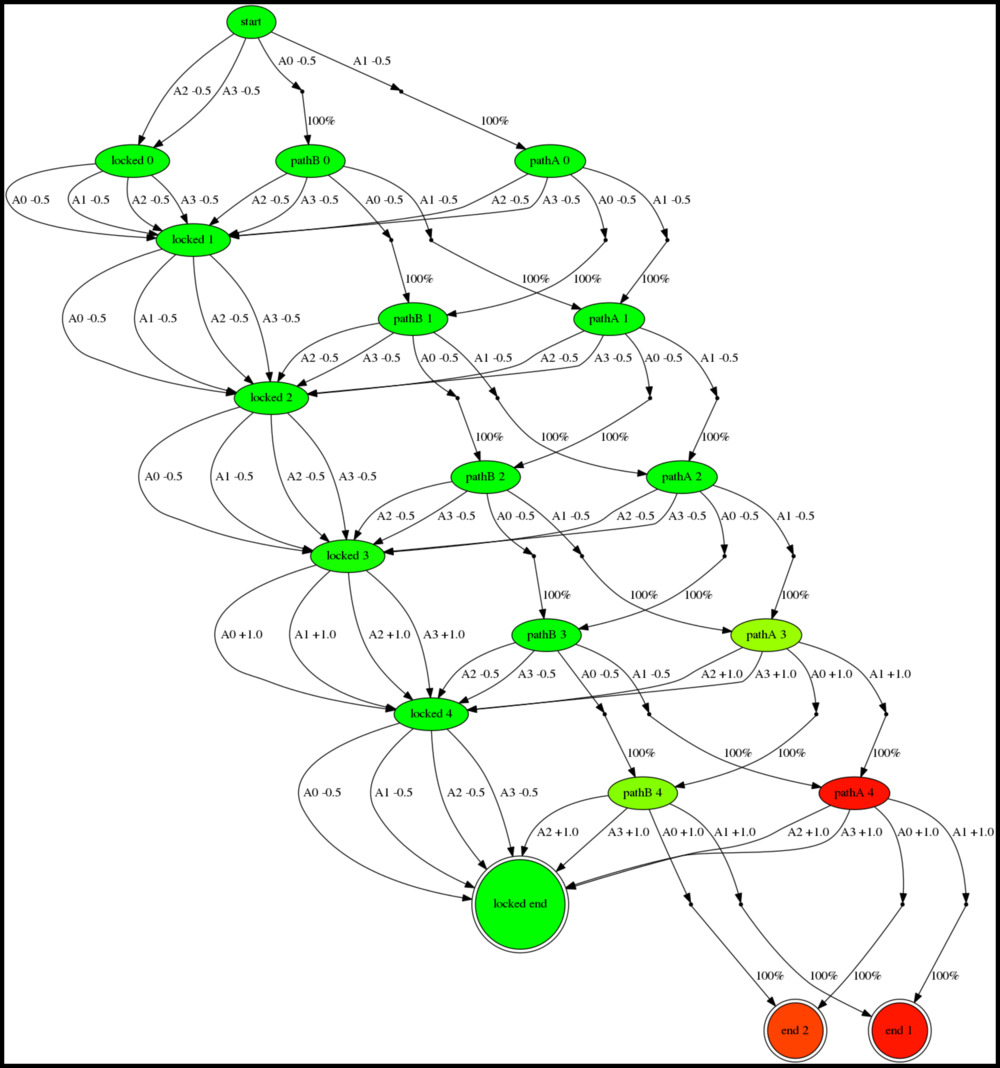}};
\path (C5_1) node[xscale=\xscale,yscale=\yscale]{\includegraphics[width=\widthfig]{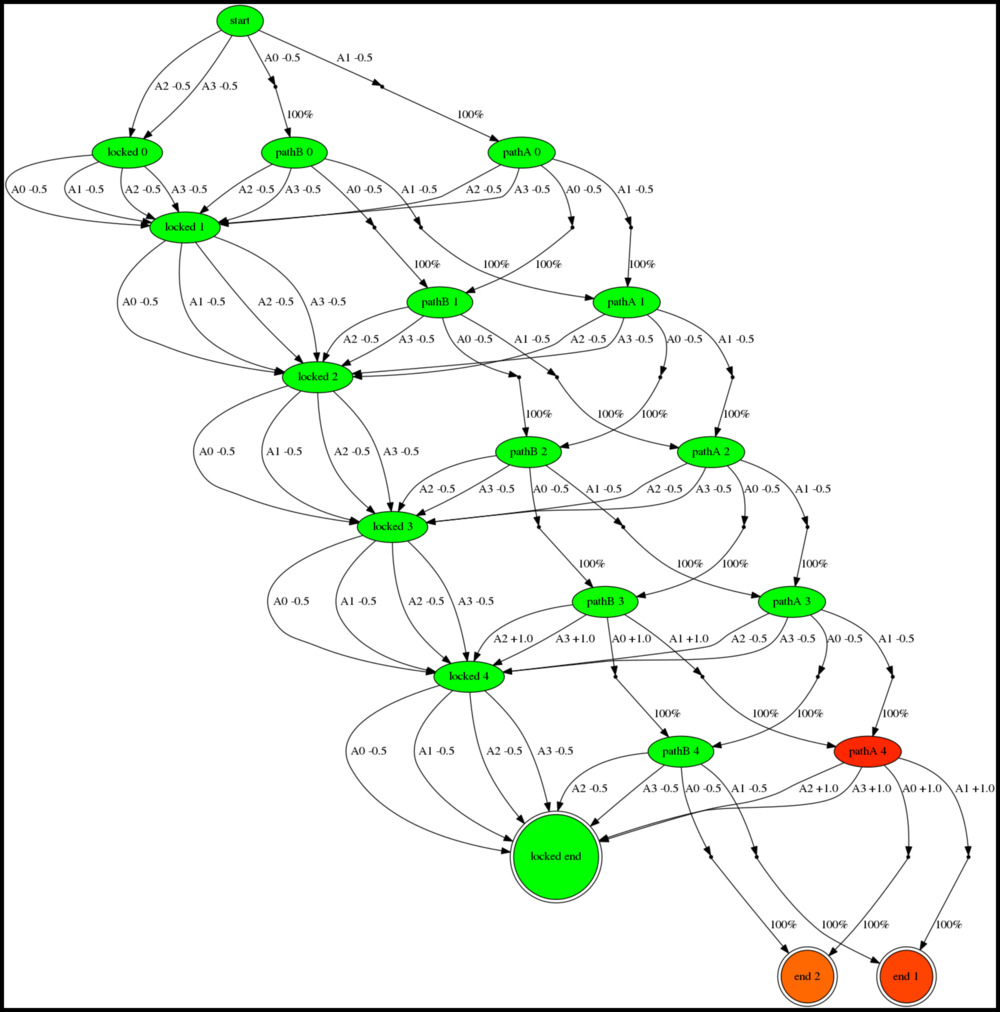}};

\path (C1_8) node[xscale=\xscale,yscale=\yscale]{\includegraphics[width=\widthfig]{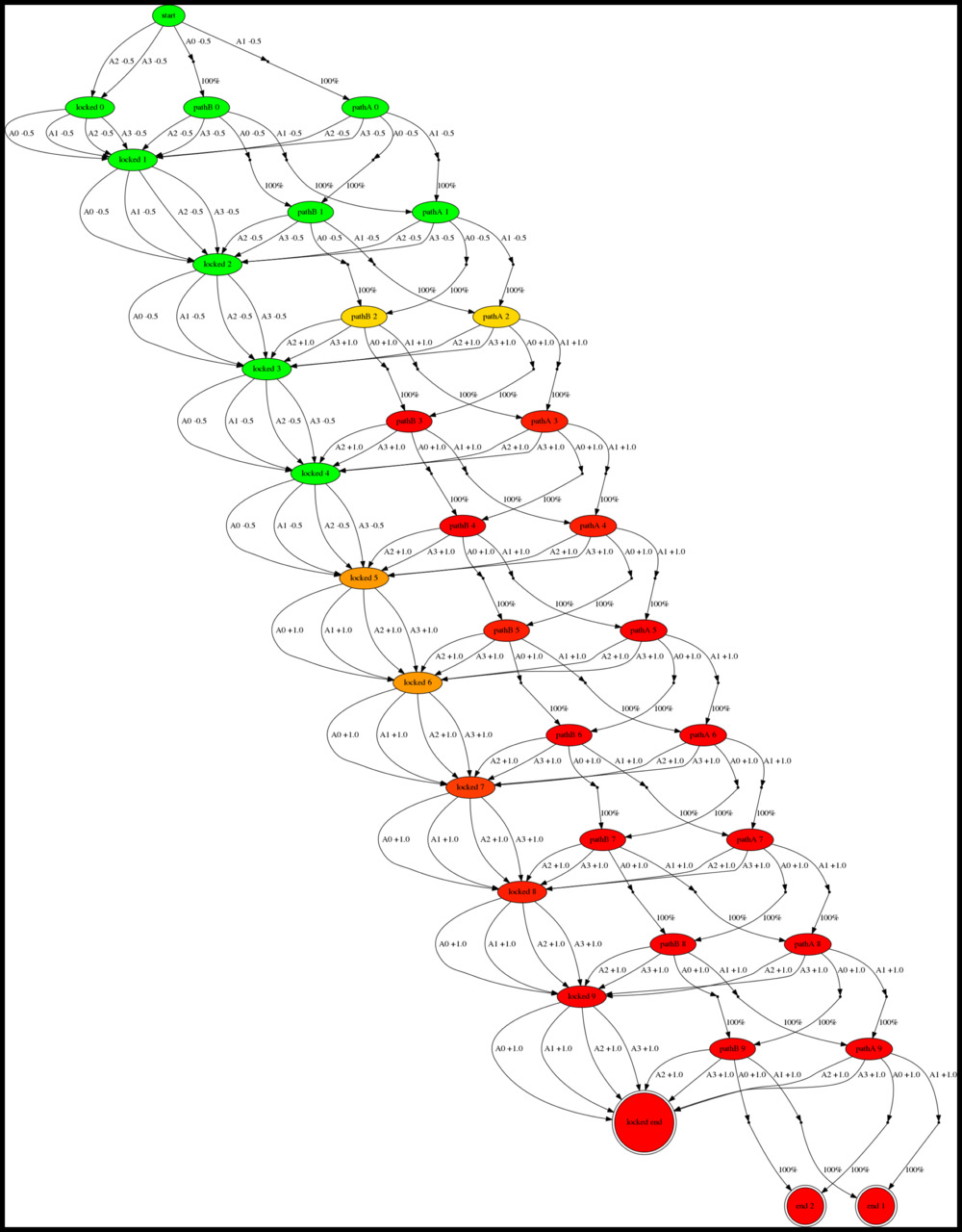}};
\path (C2_8) node[xscale=\xscale,yscale=\yscale]{\includegraphics[width=\widthfig]{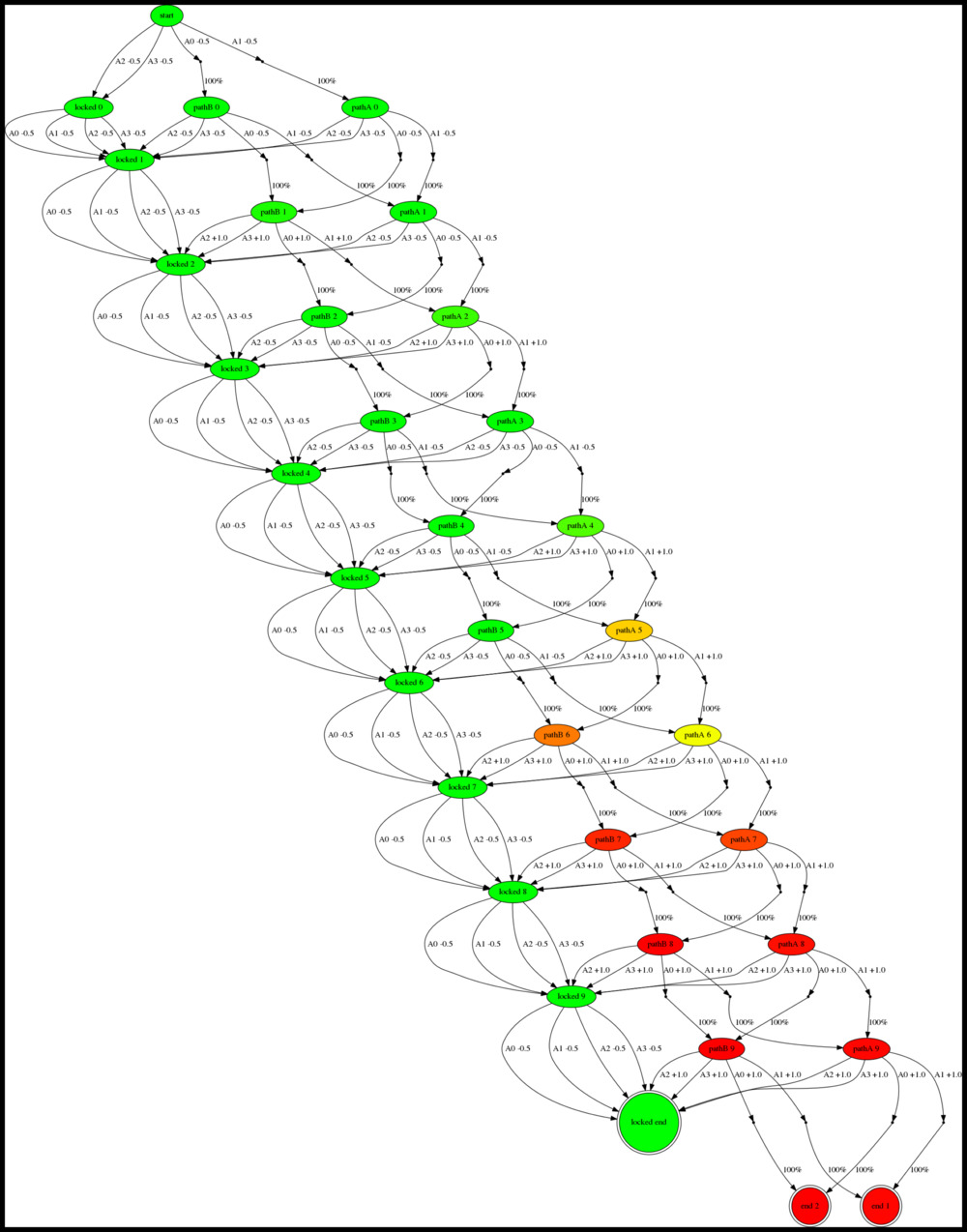}};
\path (C3_8) node[xscale=\xscale,yscale=\yscale] (central2) {\includegraphics[width=\widthfig]{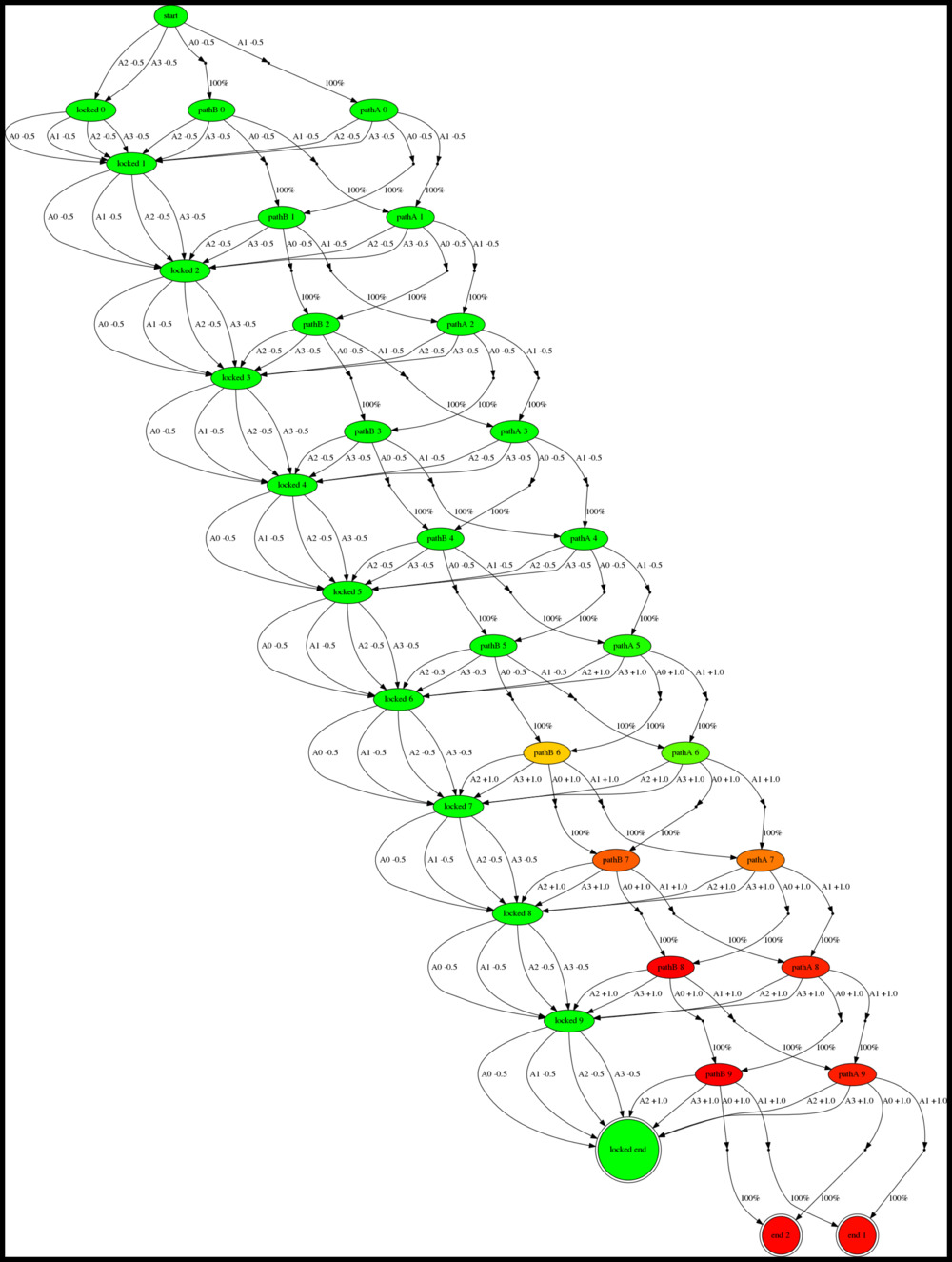}};
\path (C4_8) node[xscale=\xscale,yscale=\yscale]{\includegraphics[width=\widthfig]{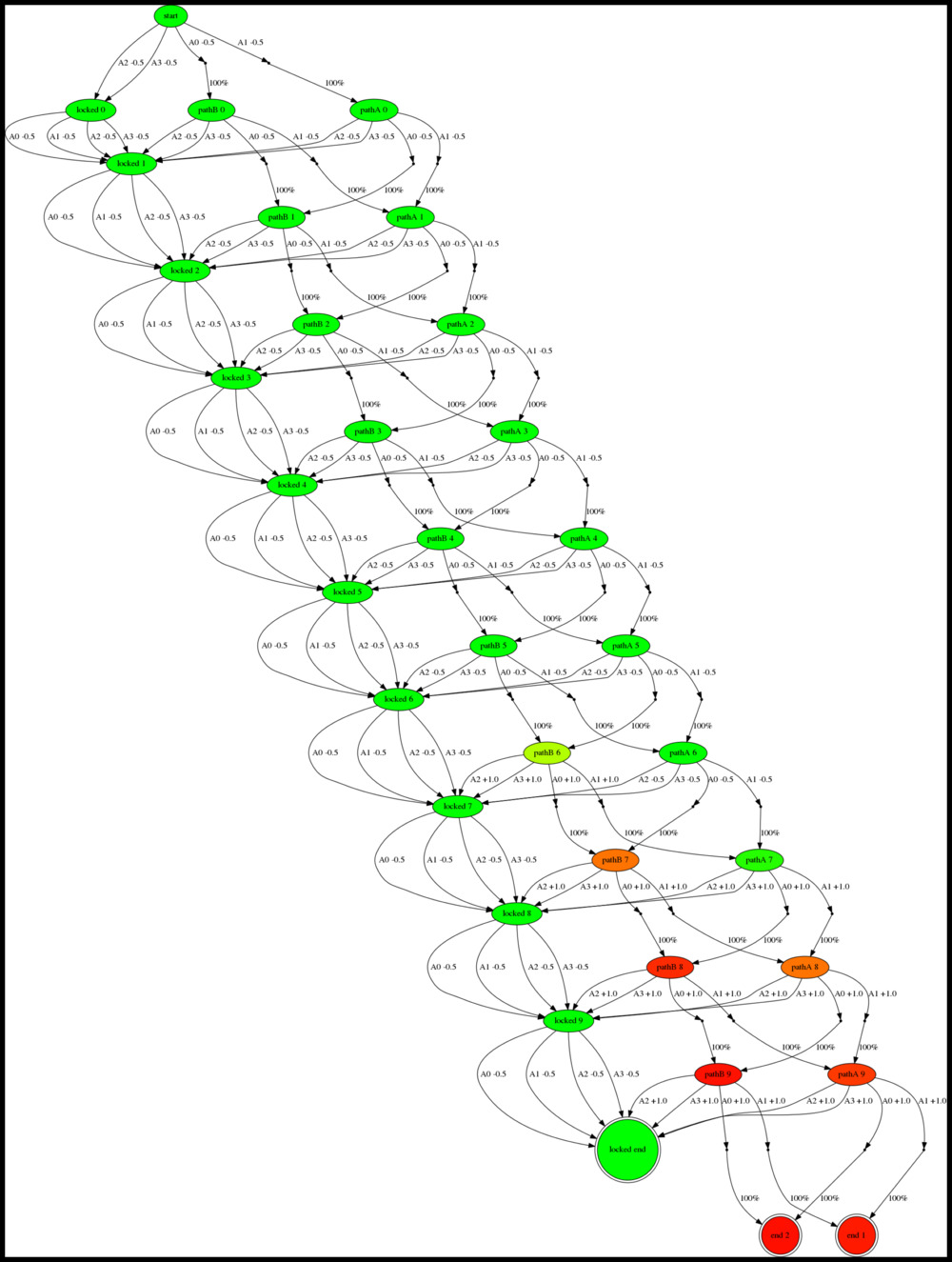}};
\path (C5_8) node[xscale=\xscale,yscale=\yscale]{\includegraphics[width=\widthfig]{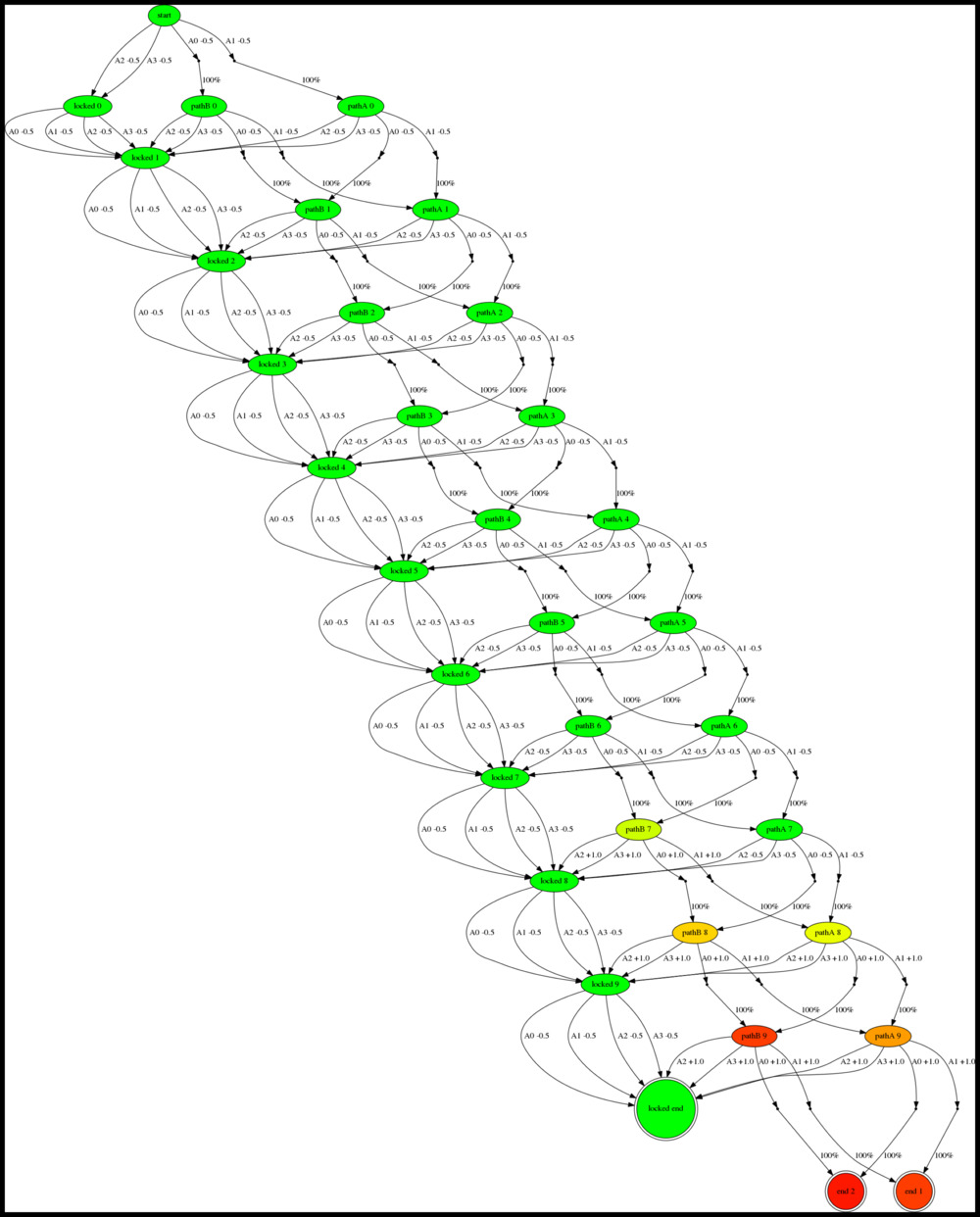}};

\path (C1_16) node[xscale=\xscale,yscale=\yscale]{\includegraphics[width=\widthfig]{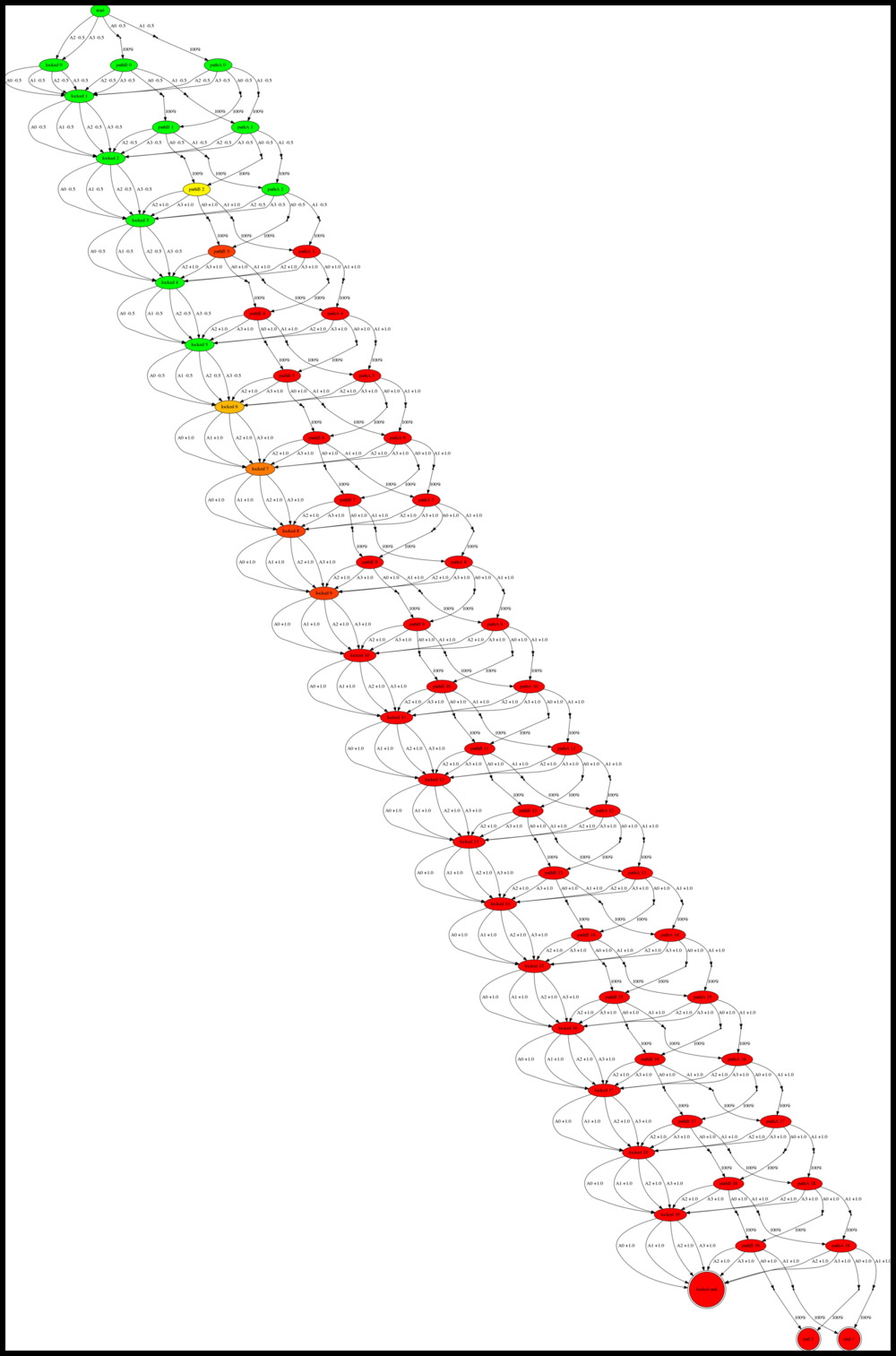}};
\path (C2_16) node[xscale=\xscale,yscale=\yscale]{\includegraphics[width=\widthfig]{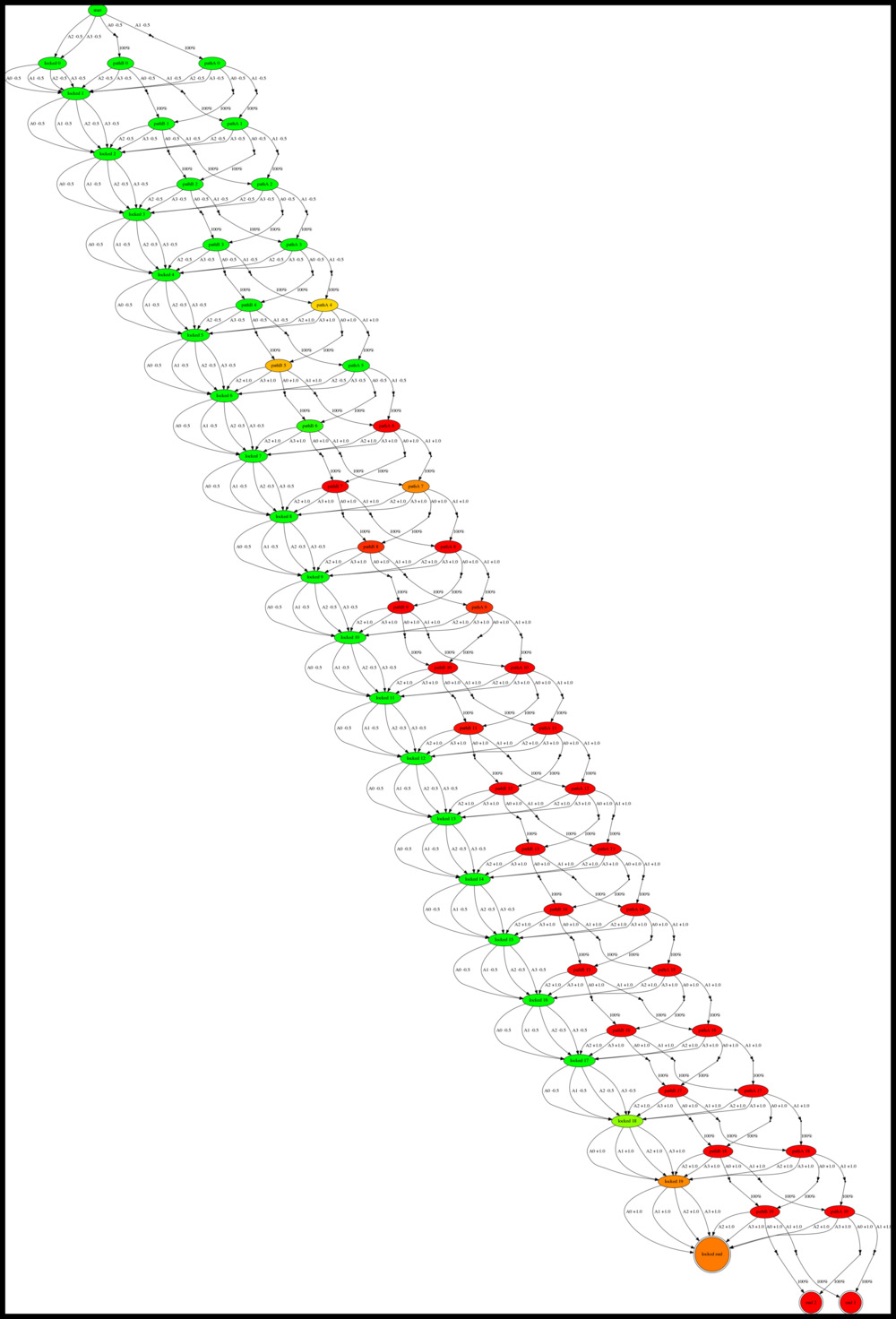}};
\path (C3_16) node[xscale=\xscale,yscale=\yscale] (central3) {\includegraphics[width=\widthfig]{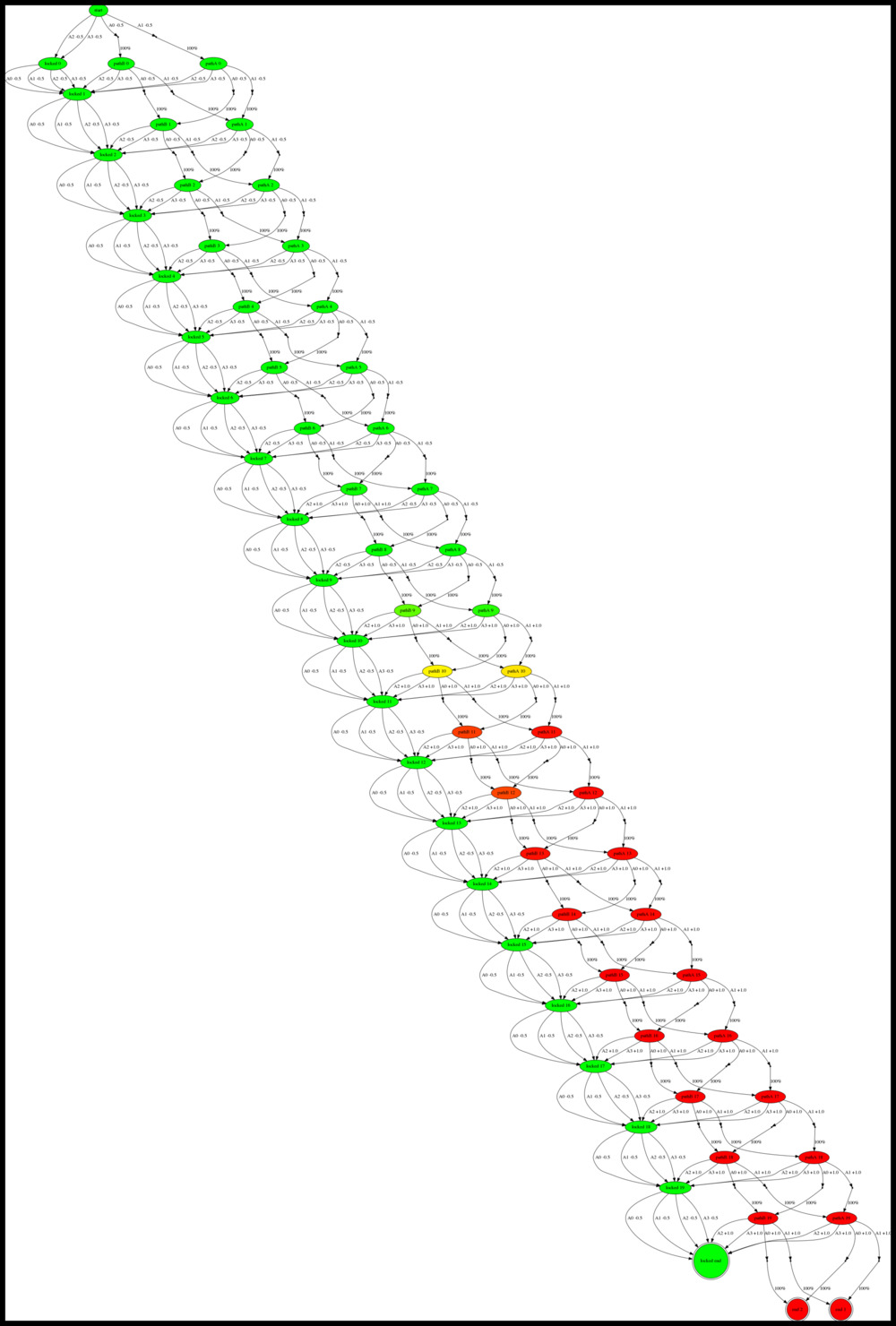}};
\path (C4_16) node[xscale=\xscale,yscale=\yscale]{\includegraphics[width=\widthfig]{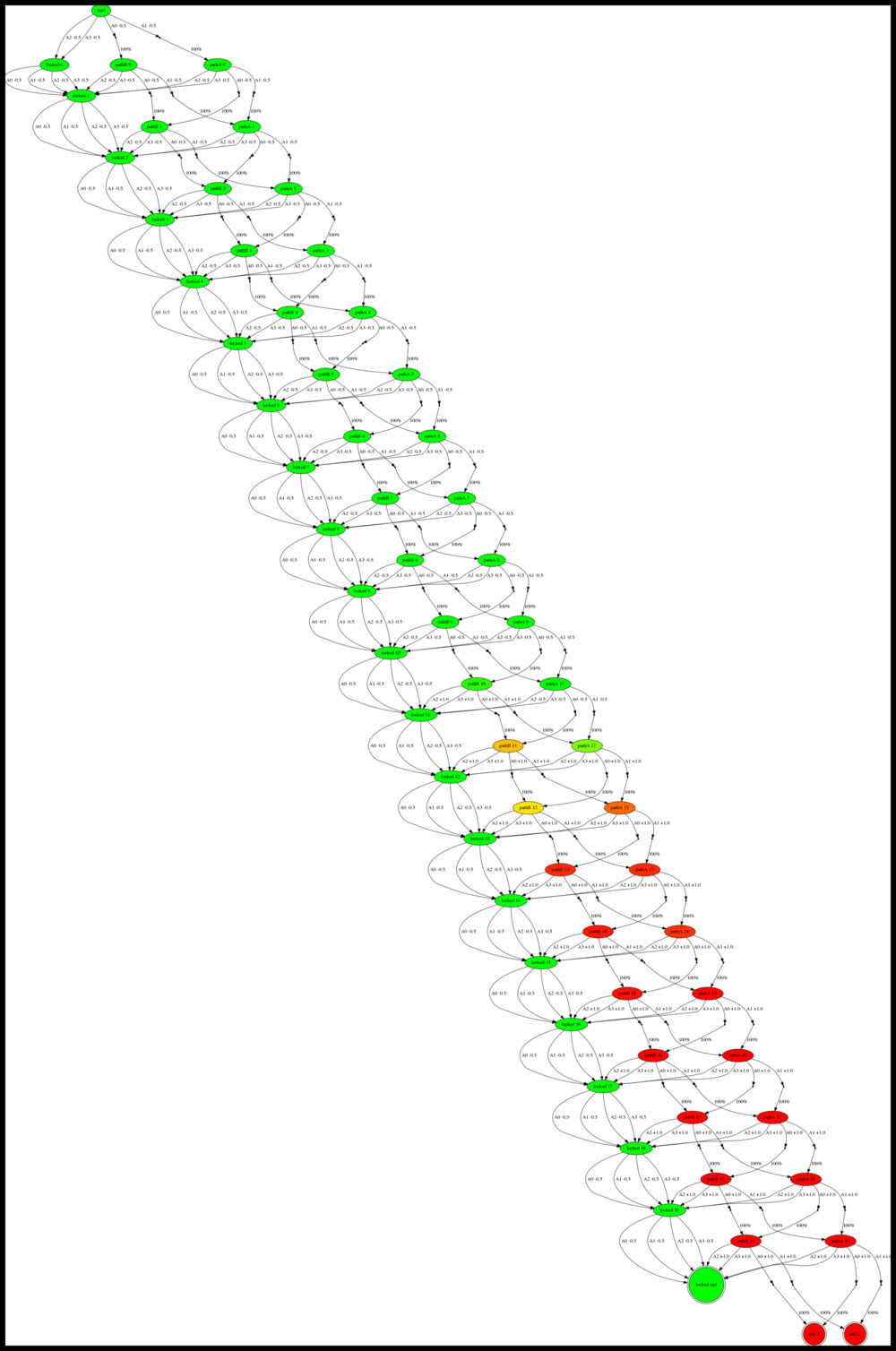}};
\path (C5_16) node[xscale=\xscale,yscale=\yscale]{\includegraphics[width=\widthfig]{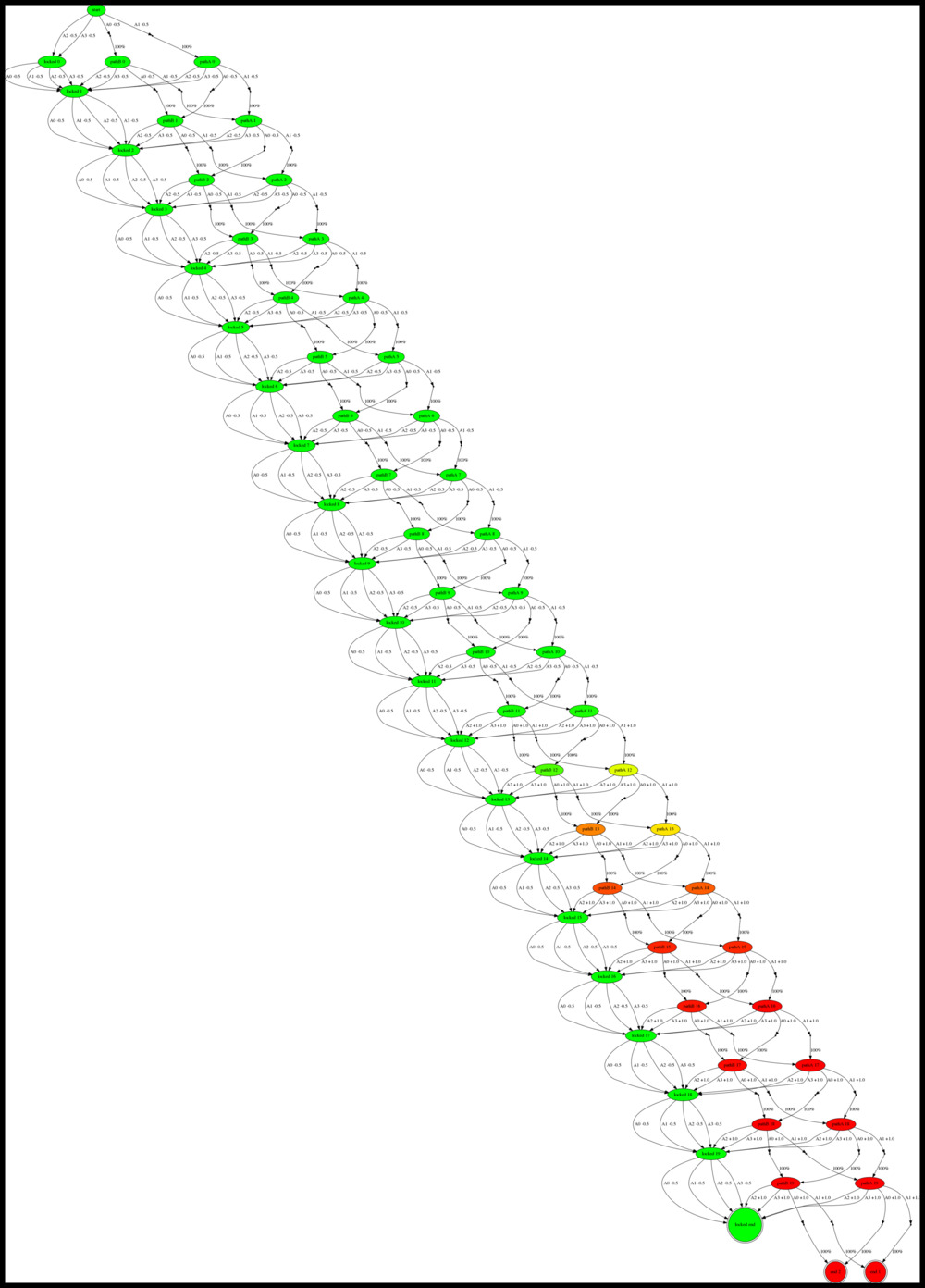}};

\path (central1.north) + (0,-\vcaption) node[format]{\texttt{CE($s_0$,\REINFORCE)} on Diabolical Combination Lock, depth $d=5$, iteration $n=1,2,3,4,5$ from left to right.};
\path (central2.north) + (0,-\vcaption) node[format]{\texttt{CE($s_0$,\REINFORCE)} on Diabolical Combination Lock, depth $d=10$, iteration $n=1,4,6,8,10$ from left to right.};
\path (central3.north) + (0,-\vcaption) node[format]{\texttt{CE($s_0$,\REINFORCE)} on Diabolical Combination Lock, depth $d=20$, iteration $n=1,4,7,10,14$ from left to right.};
\end{tikzpicture}

\clearpage

\begin{tikzpicture}[overlay] %[node distance=3cm,auto]
\label{page:evolution_trpocct}
\tikzstyle{format} = [anchor = south]
\path (.5\columnwidth,-.5\textheight) coordinate (center);
\def\vstep{1.12cm}
\def\hstep{4cm}
\def\widthfig{10cm}
\def\vcaption{-0.2cm}
\path (-5cm,1.1cm) coordinate (UL);
\foreach\t in {0,...,10} \foreach\s in {0,...,20} \path (UL) + (\t*\hstep,-\s*\vstep) coordinate (C\t_\s);
%	\foreach\t in {0,...,10} \foreach\s in {0,...,6} \path (C\t_\s) node {(\t,-\s)};

\def\scale{0.23}
\path (C1_1) node[scale=\scale]{\includegraphics[width=\widthfig]{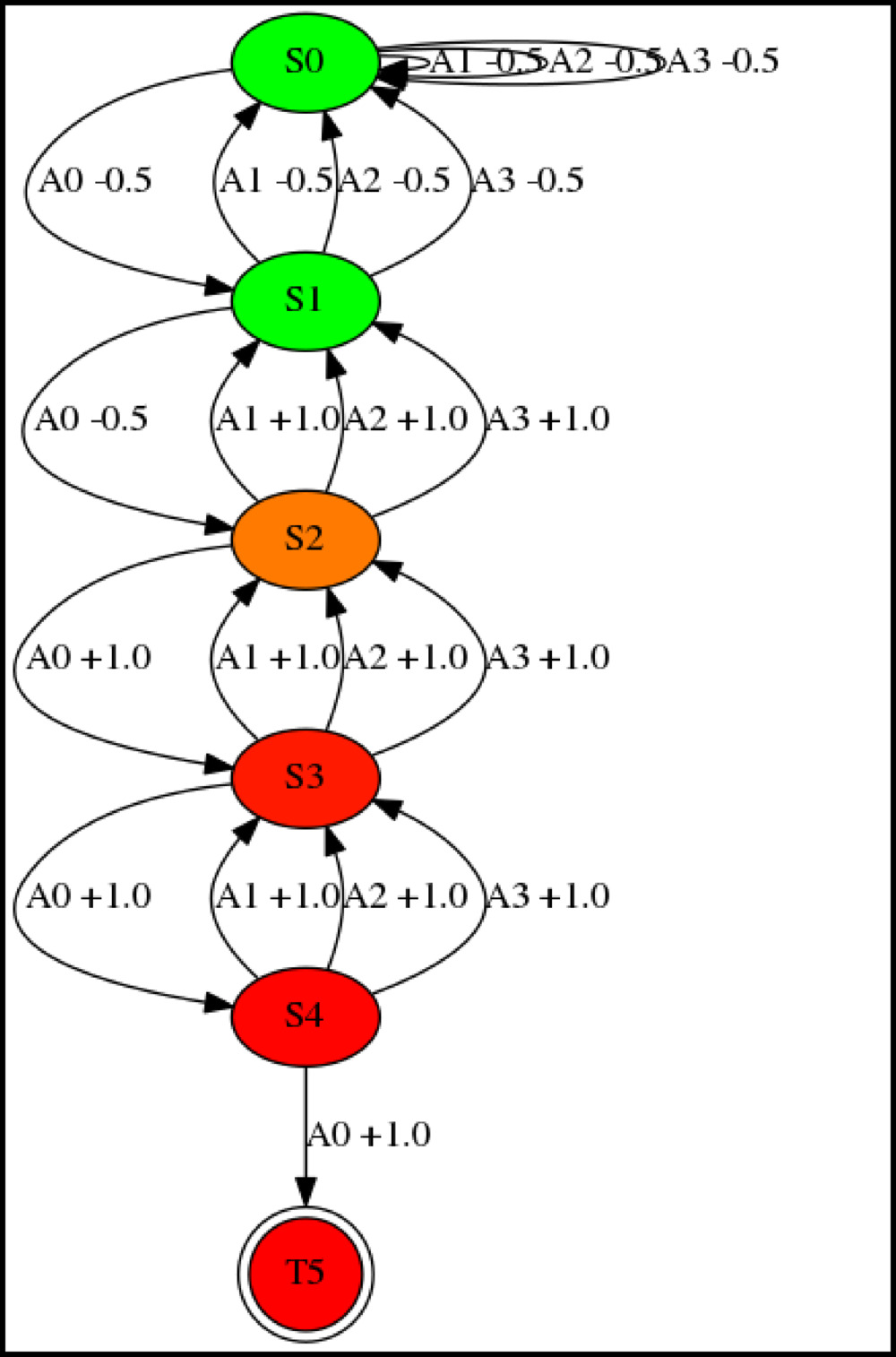}};
\path (C2_1) node[scale=\scale]{\includegraphics[width=\widthfig]{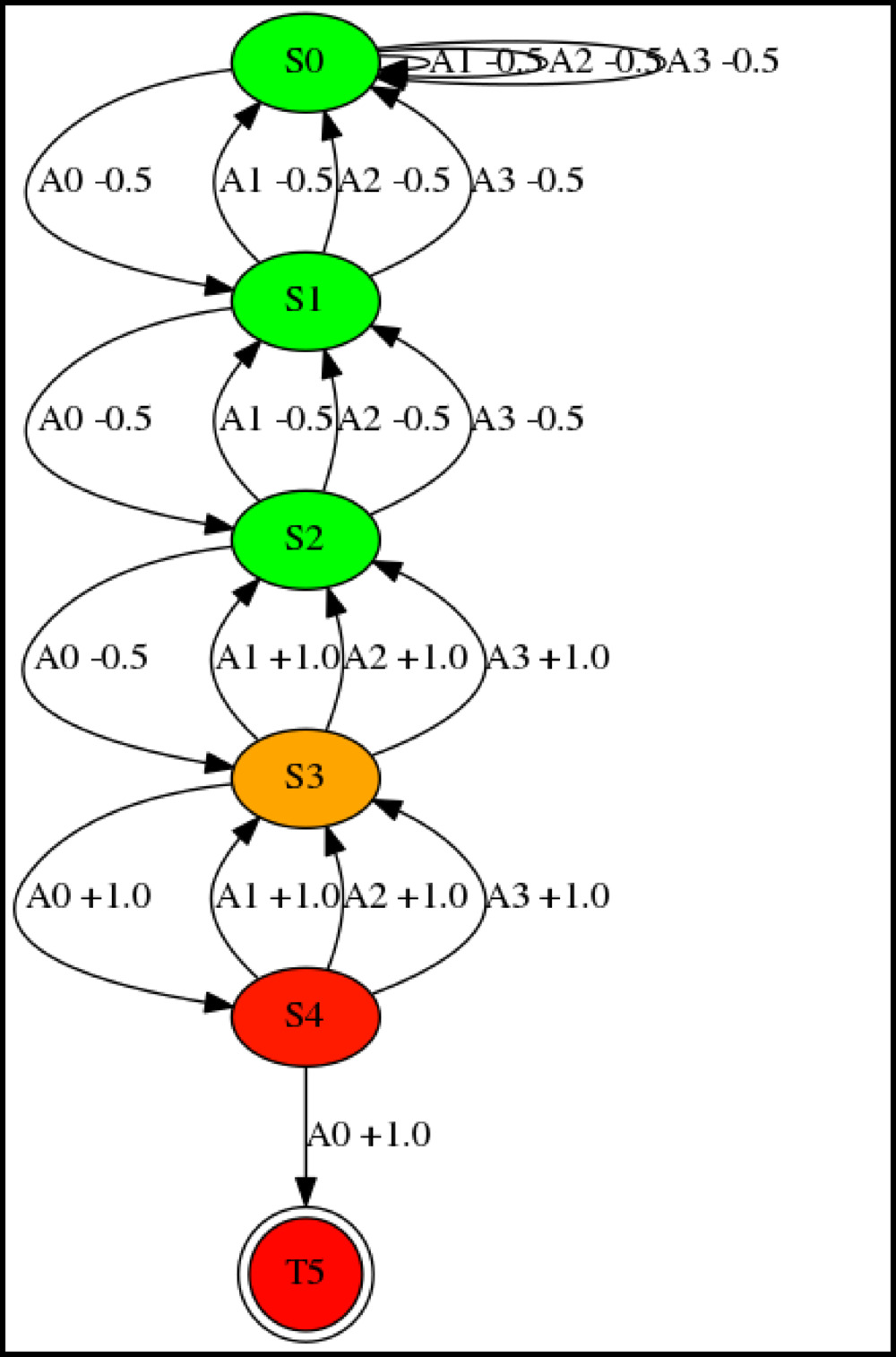}};
\path (C3_1) node[scale=\scale] (central1) {\includegraphics[width=\widthfig]{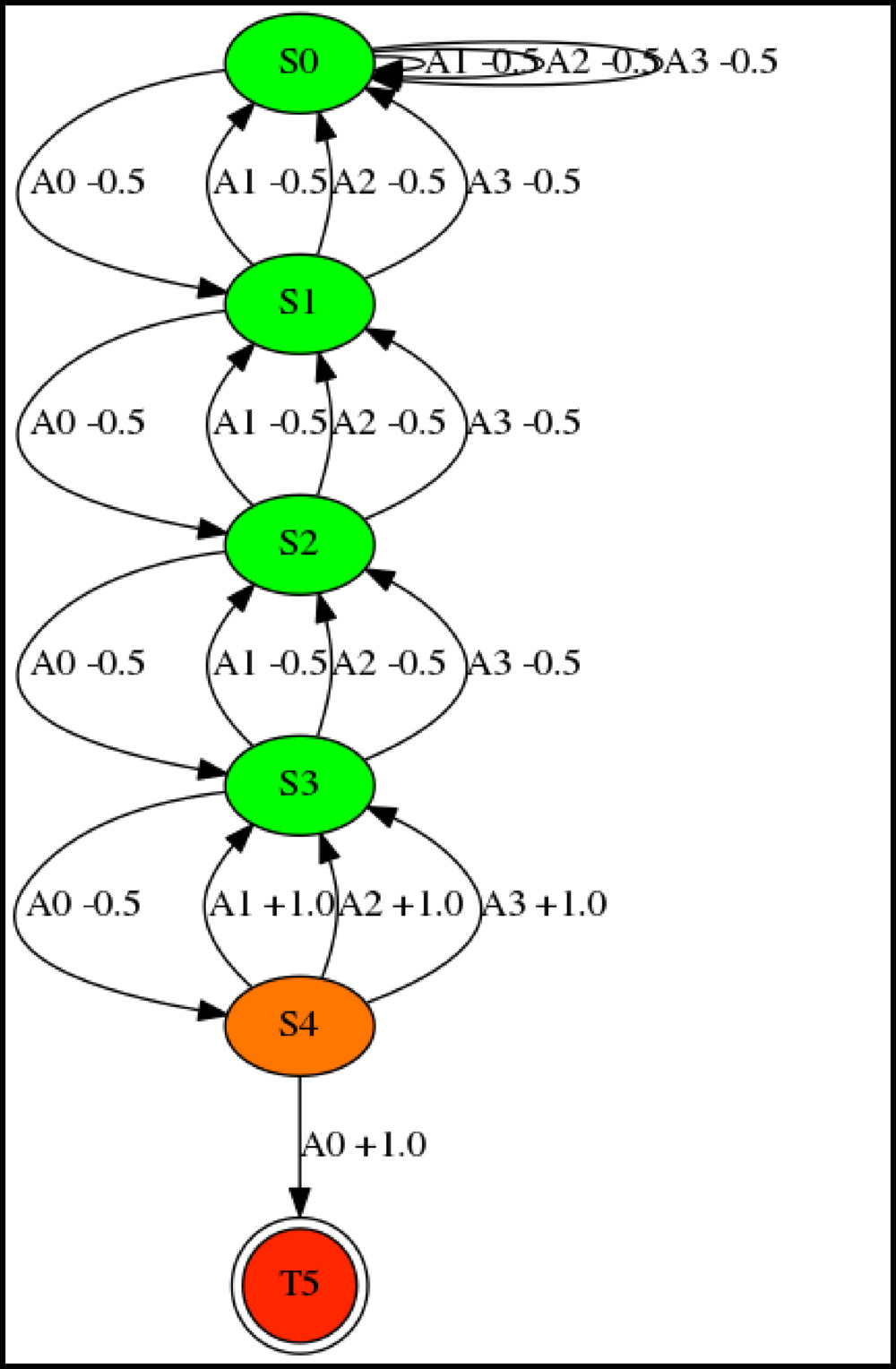}};
\path (C4_1) node[scale=\scale]{\includegraphics[width=\widthfig]{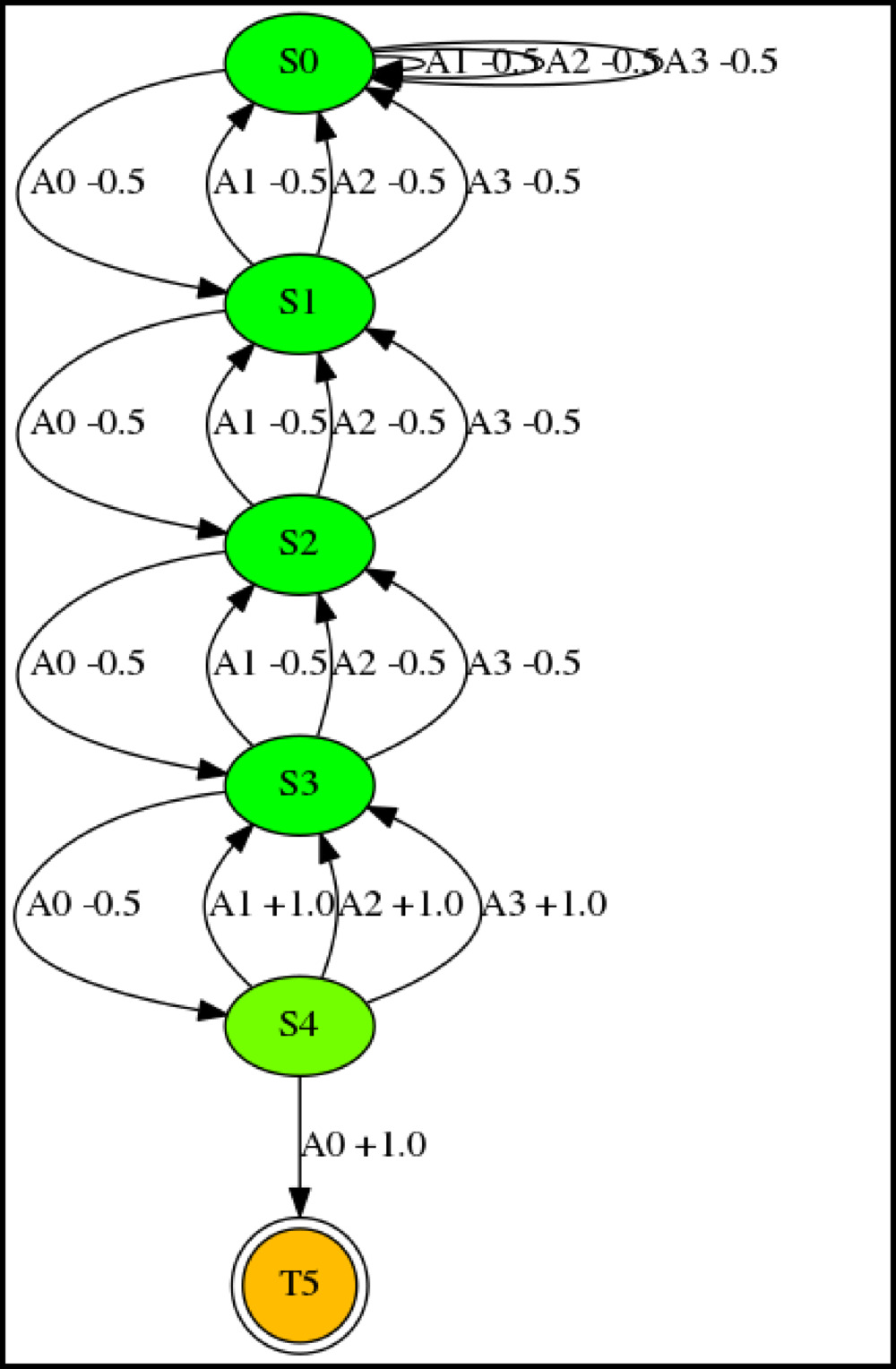}};
\path (C5_1) node[scale=\scale]{\includegraphics[width=\widthfig]{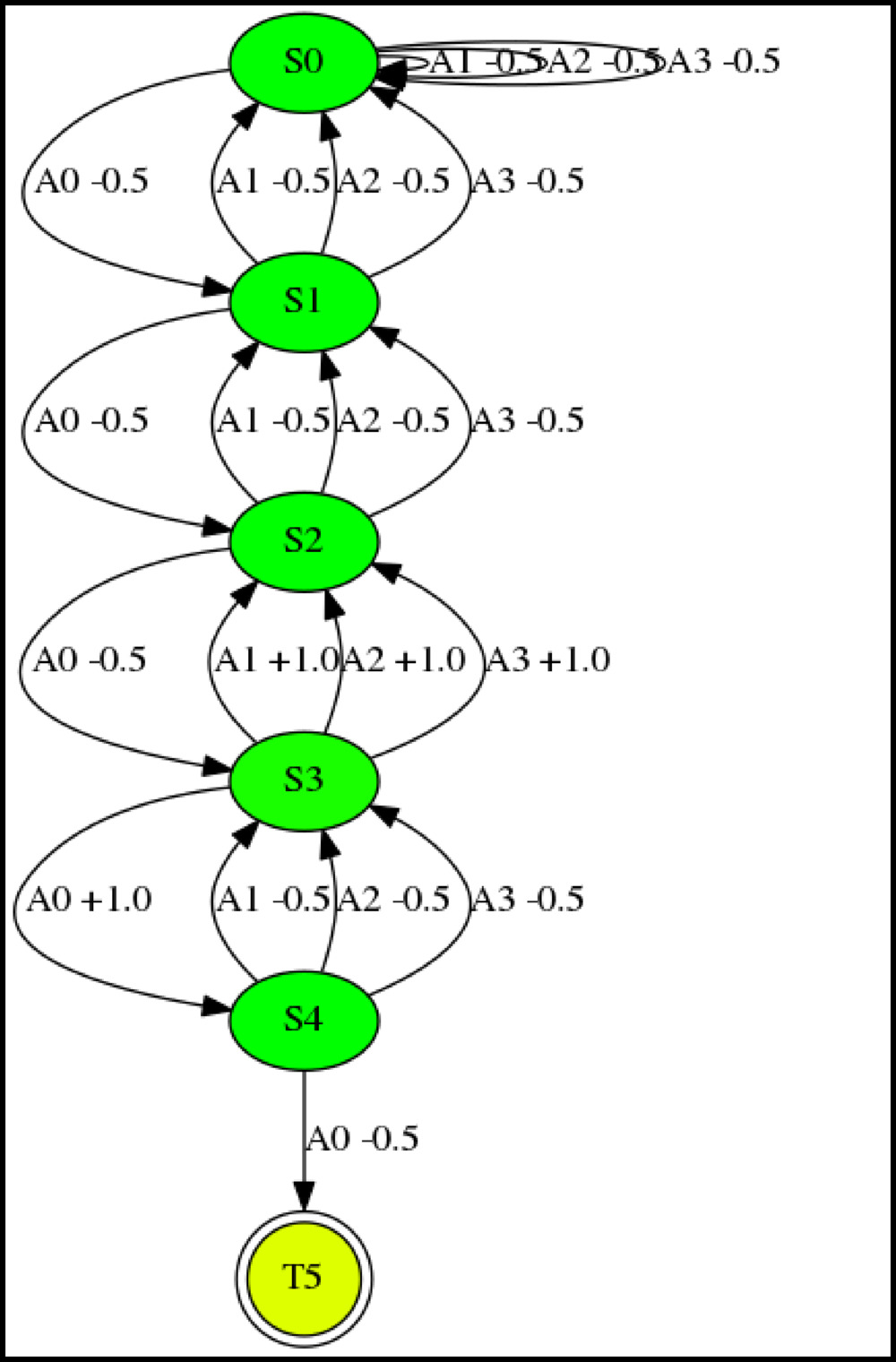}};

\path (C1_7) node[scale=\scale]{\includegraphics[width=\widthfig]{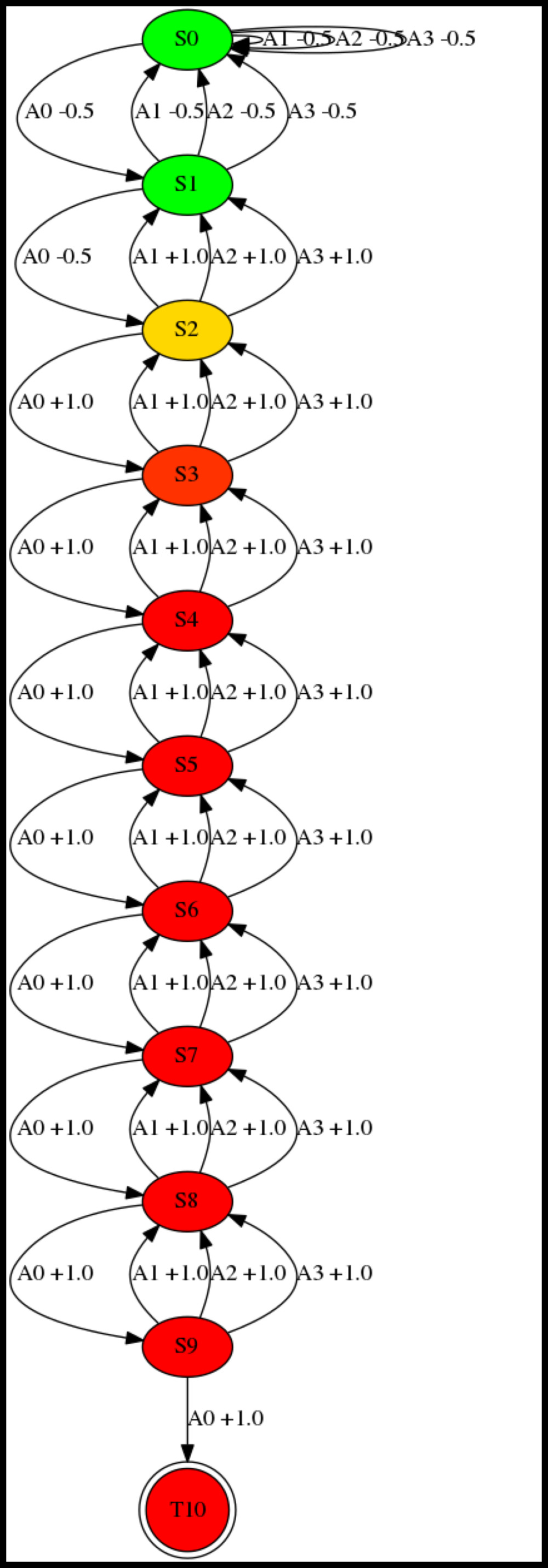}};
\path (C2_7) node[scale=\scale]{\includegraphics[width=\widthfig]{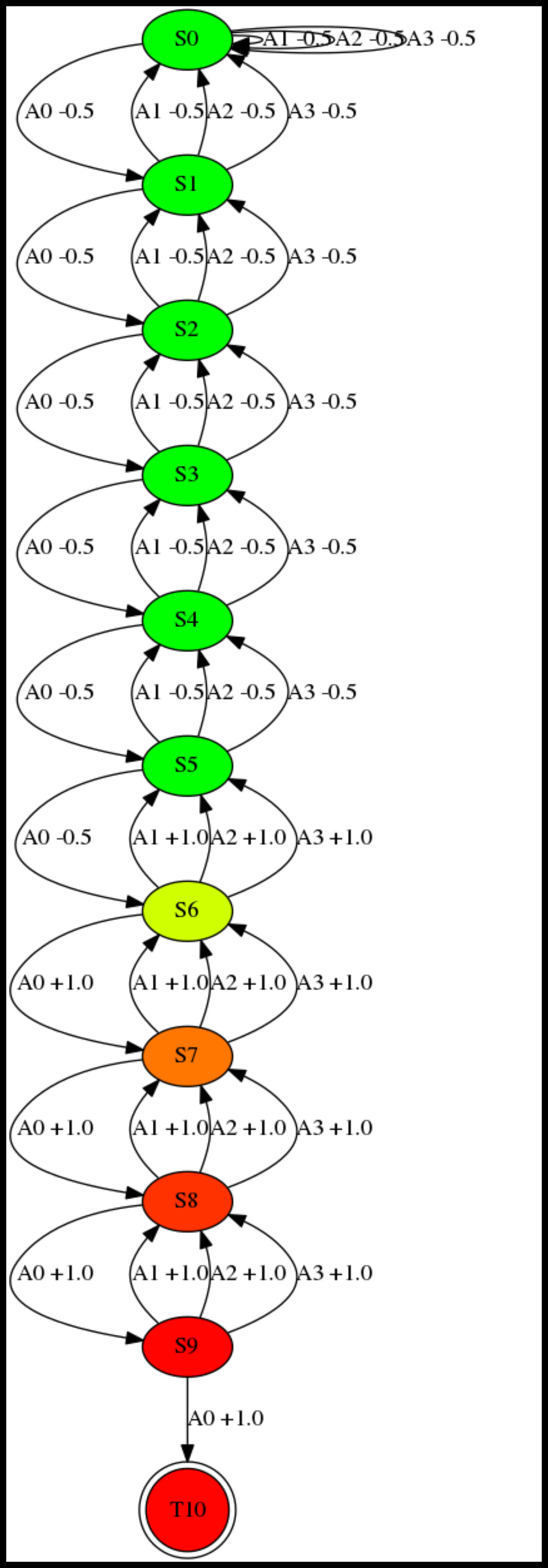}};
\path (C3_7) node[scale=\scale] (central2) {\includegraphics[width=\widthfig]{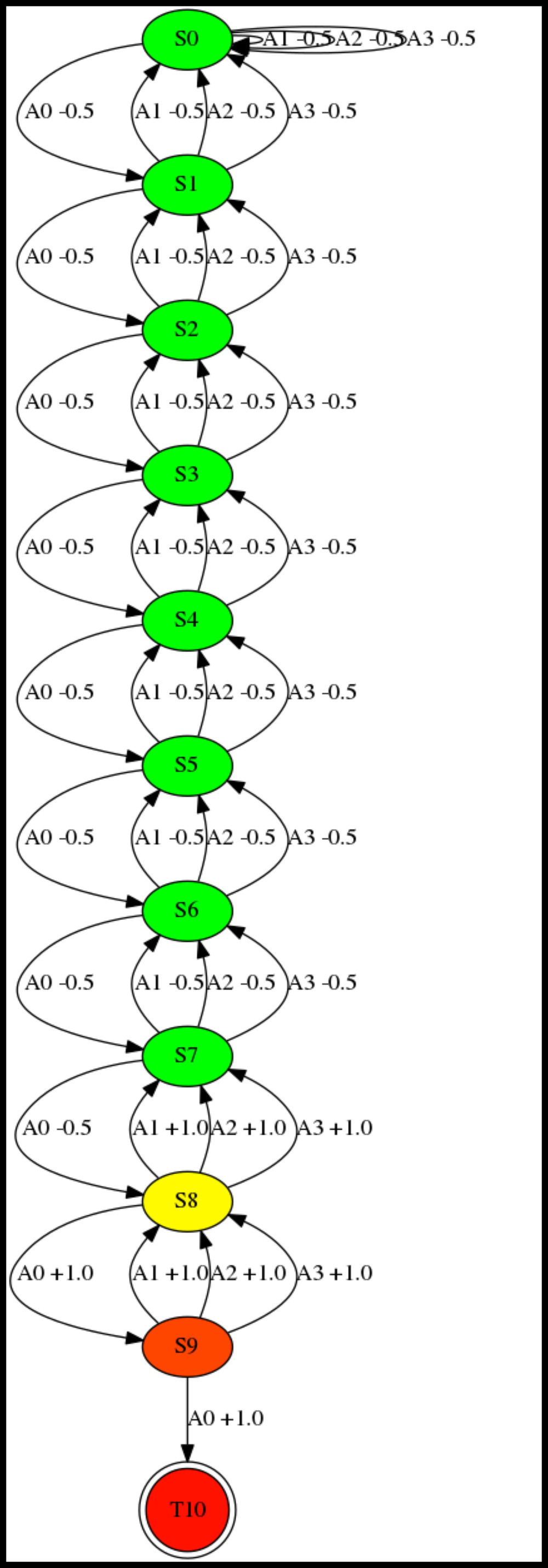}};
\path (C4_7) node[scale=\scale]{\includegraphics[width=\widthfig]{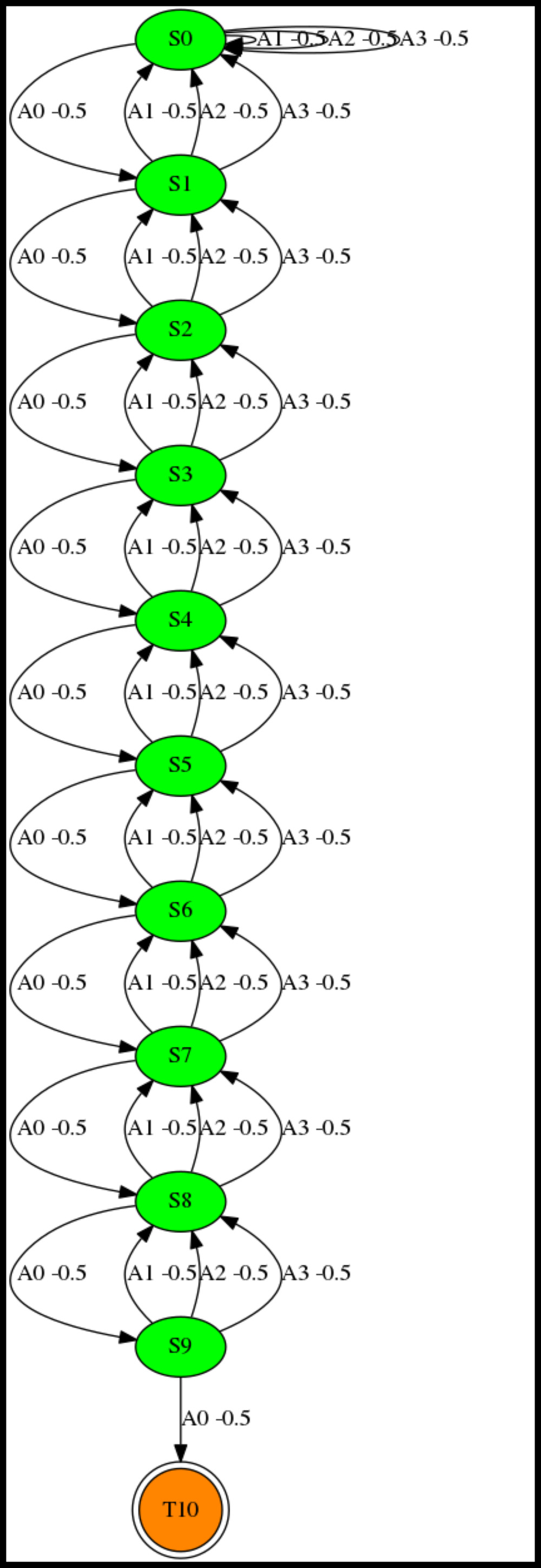}};
\path (C5_7) node[scale=\scale]{\includegraphics[width=\widthfig]{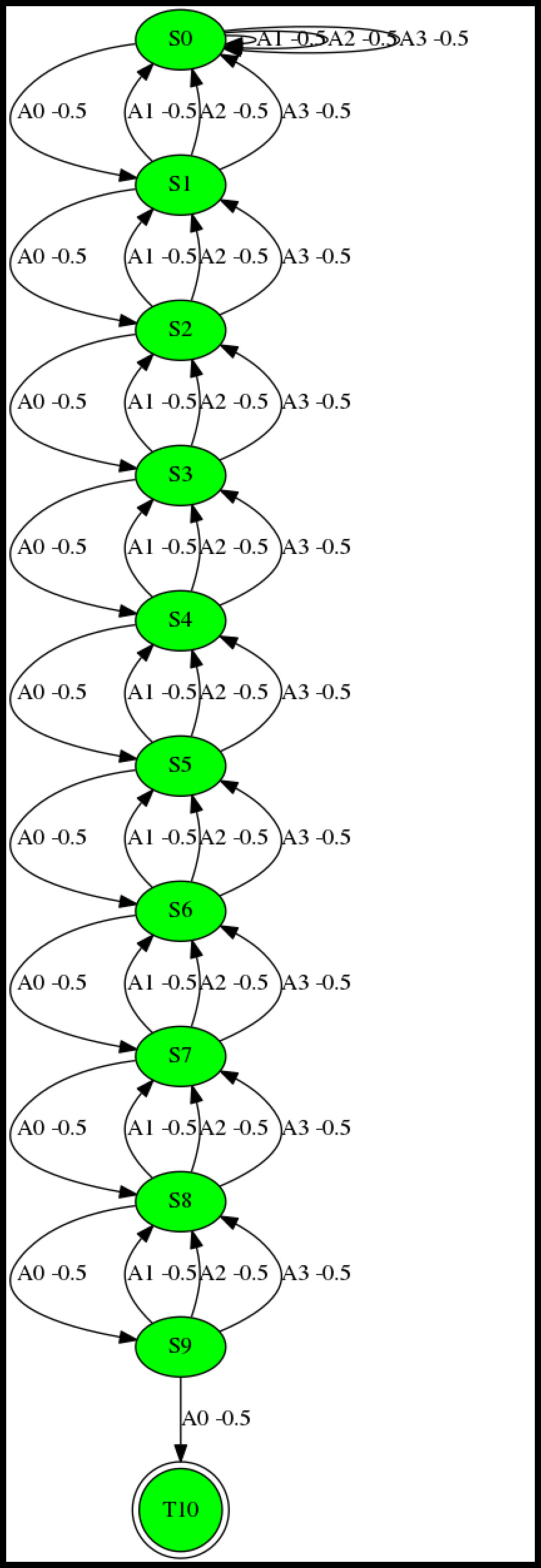}};

\path (C1_17) node[scale=\scale]{\includegraphics[width=\widthfig]{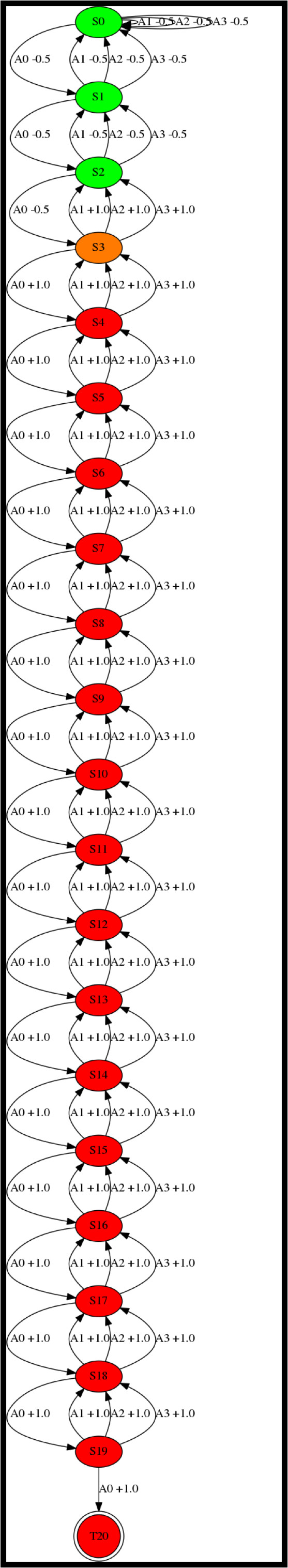}};
\path (C2_17) node[scale=\scale]{\includegraphics[width=\widthfig]{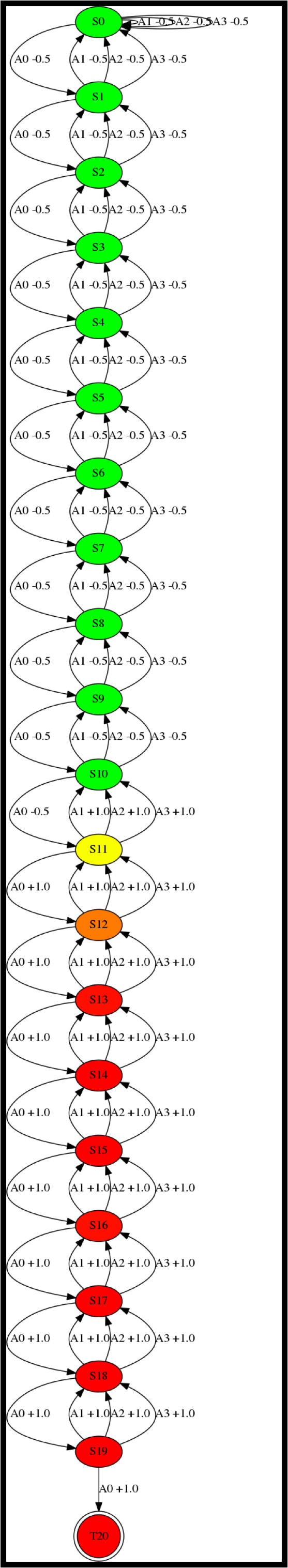}};
\path (C3_17) node[scale=\scale] (central3) {\includegraphics[width=\widthfig]{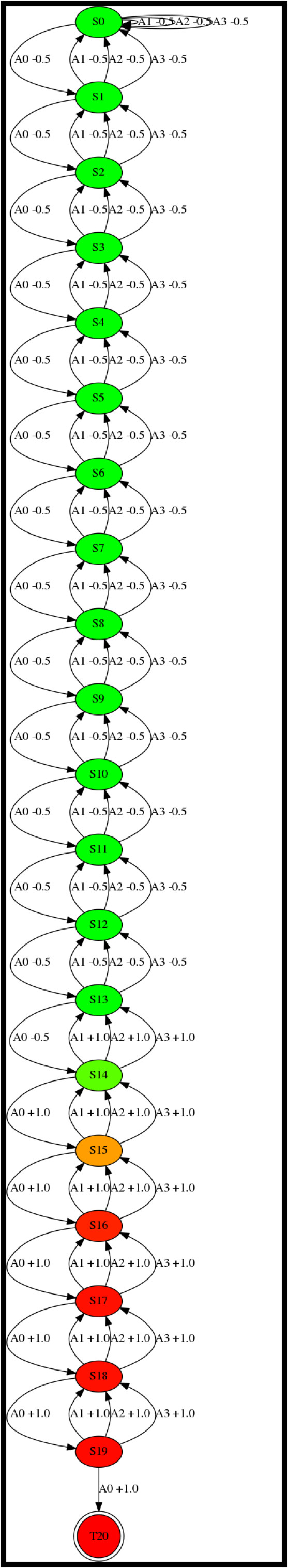}};
\path (C4_17) node[scale=\scale]{\includegraphics[width=\widthfig]{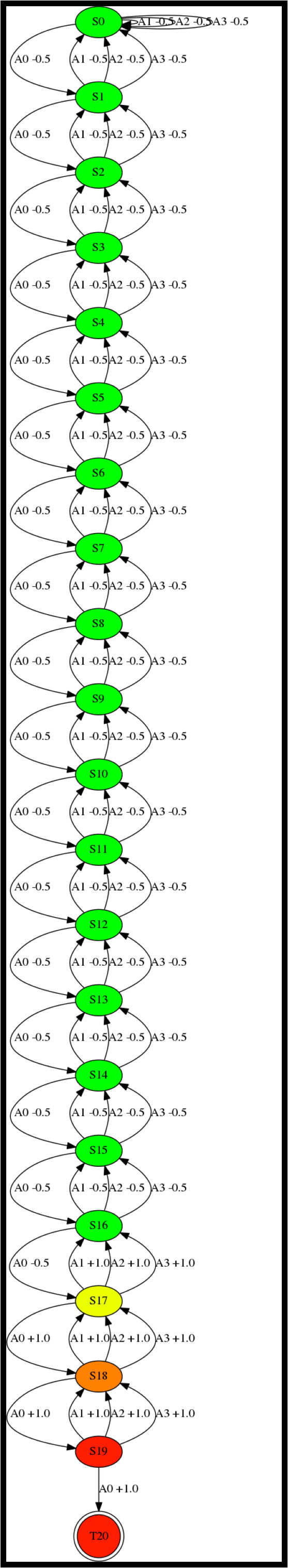}};
\path (C5_17) node[scale=\scale]{\includegraphics[width=\widthfig]{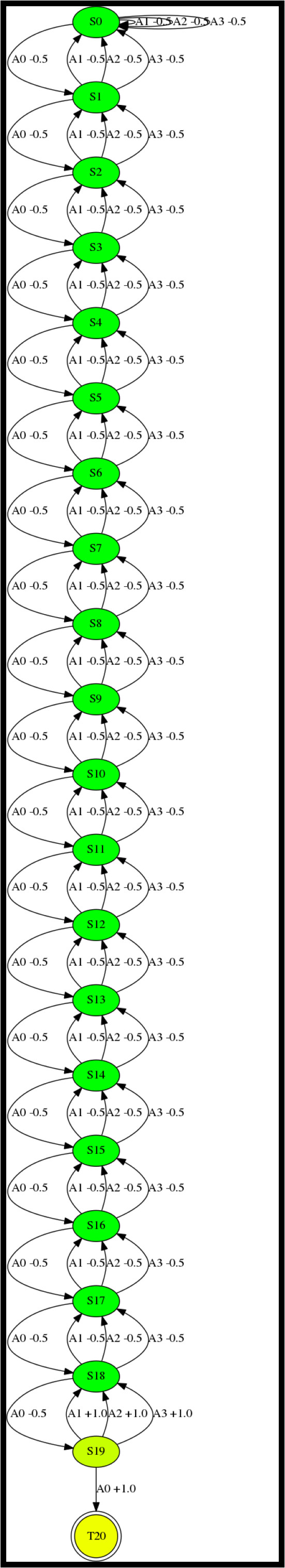}};

\path (central1.north) + (0,-\vcaption) node[format]{\texttt{CE($s_0$,TRPO)} on Consecutive Crossroad Traps, depth $d=5$, iteration $n=1,2,3,4,5$ from left to right.};
\path (central2.north) + (0,-\vcaption) node[format]{\texttt{CE($s_0$,TRPO)} on Consecutive Crossroad Traps, depth $d=10$, iteration $n=1,4,6,8,10$ from left to right.};
\path (central3.north) + (0,-\vcaption) node[format]{\texttt{CE($s_0$,TRPO)} on Consecutive Crossroad Traps, depth $d=20$, iteration $n=1,4,7,10,14$ from left to right.};
\end{tikzpicture}

\clearpage

\begin{tikzpicture}[overlay] %[node distance=3cm,auto]
\label{page:evolution_trpodcl}
\tikzstyle{format} = [anchor = south]
\path (.5\columnwidth,-.5\textheight) coordinate (center);
\def\vstep{1.25cm}
\def\hstep{4cm}
\def\widthfig{10cm}
\def\vcaption{-0.1cm}
\path (-5cm,0cm) coordinate (UL);
\foreach\t in {0,...,10} \foreach\s in {0,...,20} \path (UL) + (\t*\hstep,-\s*\vstep) coordinate (C\t_\s);
%	\foreach\t in {0,...,10} \foreach\s in {0,...,6} \path (C\t_\s) node {(\t,-\s)};

\def\xscale{0.3}
\def\yscale{0.6}
\path (C1_1) node[xscale=\xscale,yscale=\yscale]{\includegraphics[width=\widthfig]{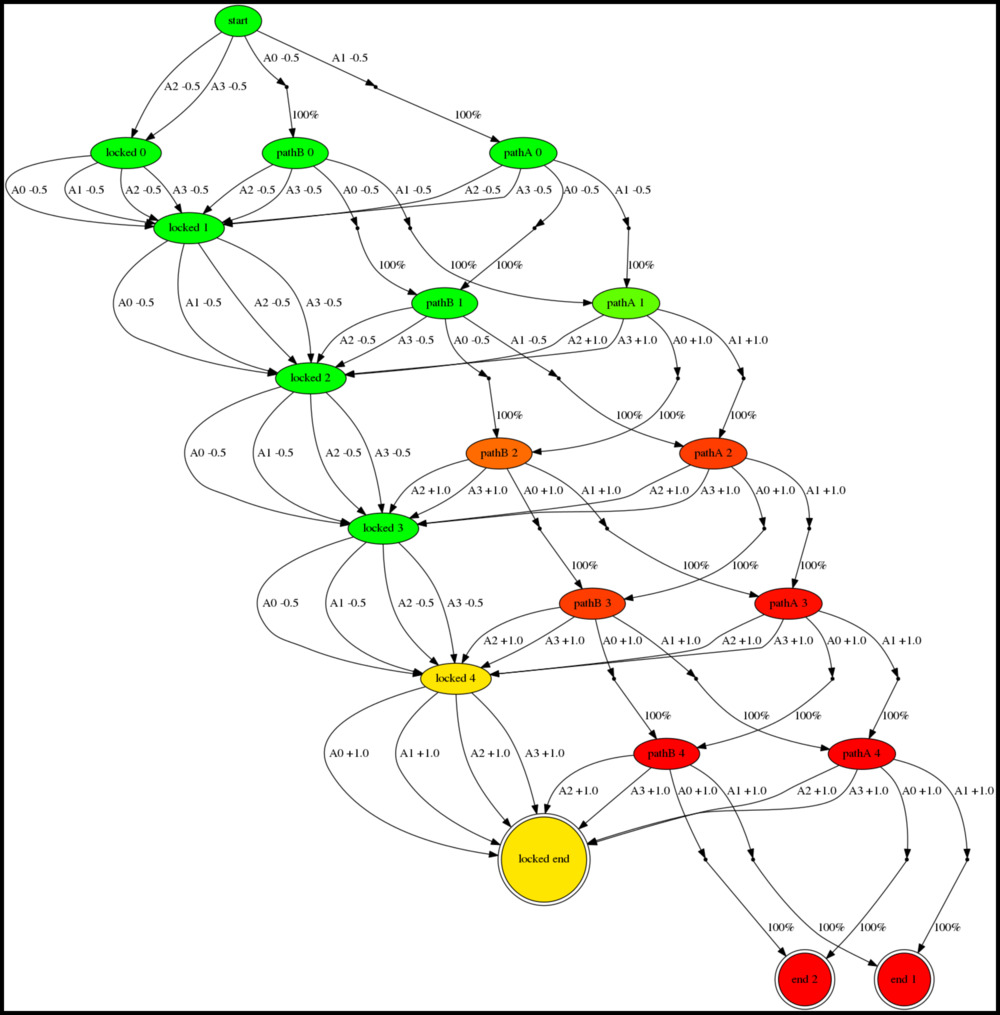}};
\path (C2_1) node[xscale=\xscale,yscale=\yscale]{\includegraphics[width=\widthfig]{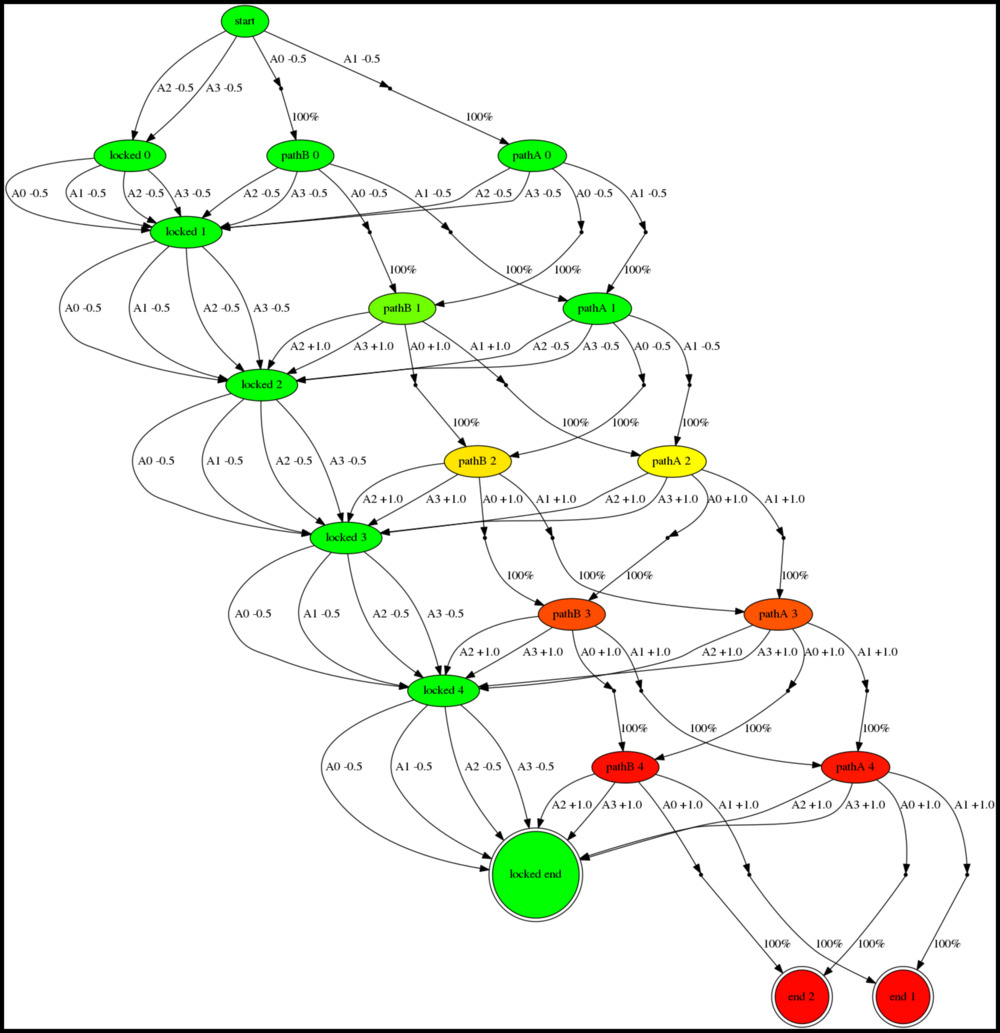}};
\path (C3_1) node[xscale=\xscale,yscale=\yscale] (central1) {\includegraphics[width=\widthfig]{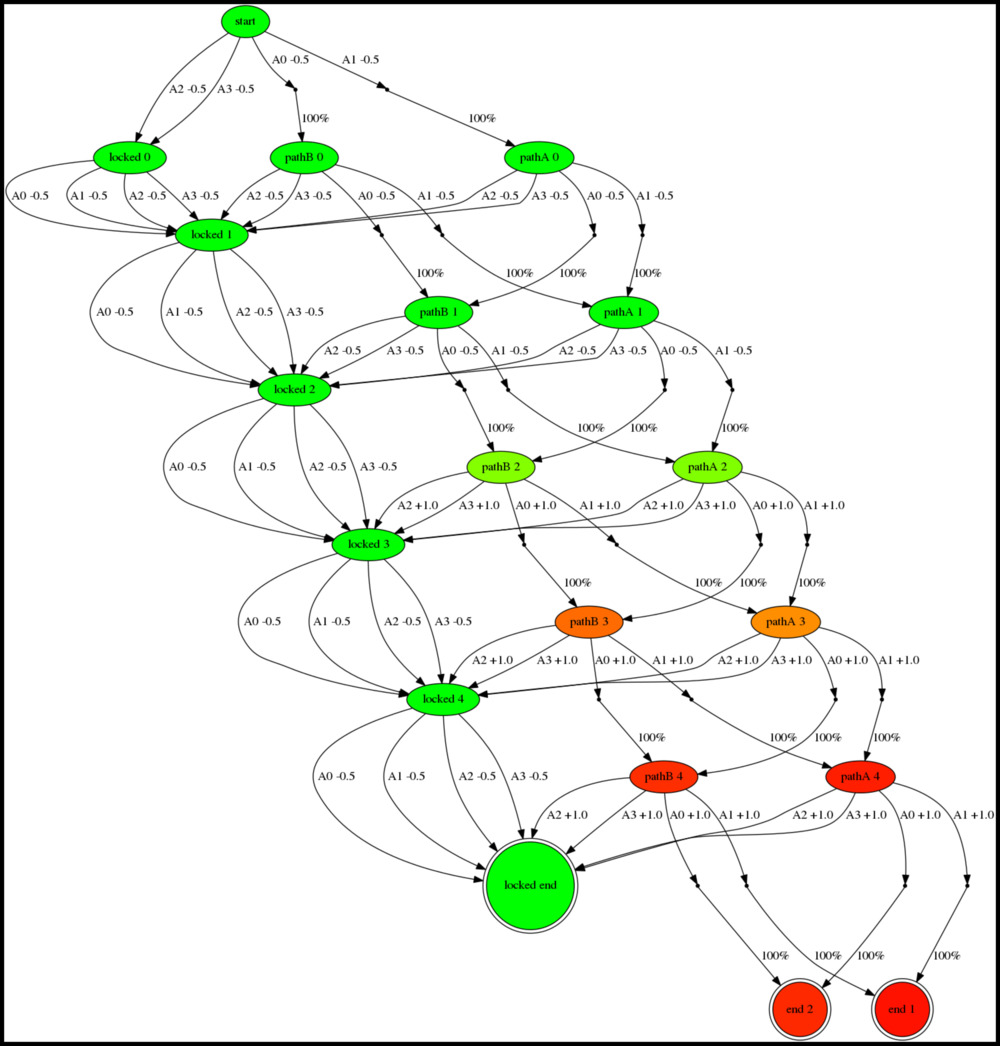}};
\path (C4_1) node[xscale=\xscale,yscale=\yscale]{\includegraphics[width=\widthfig]{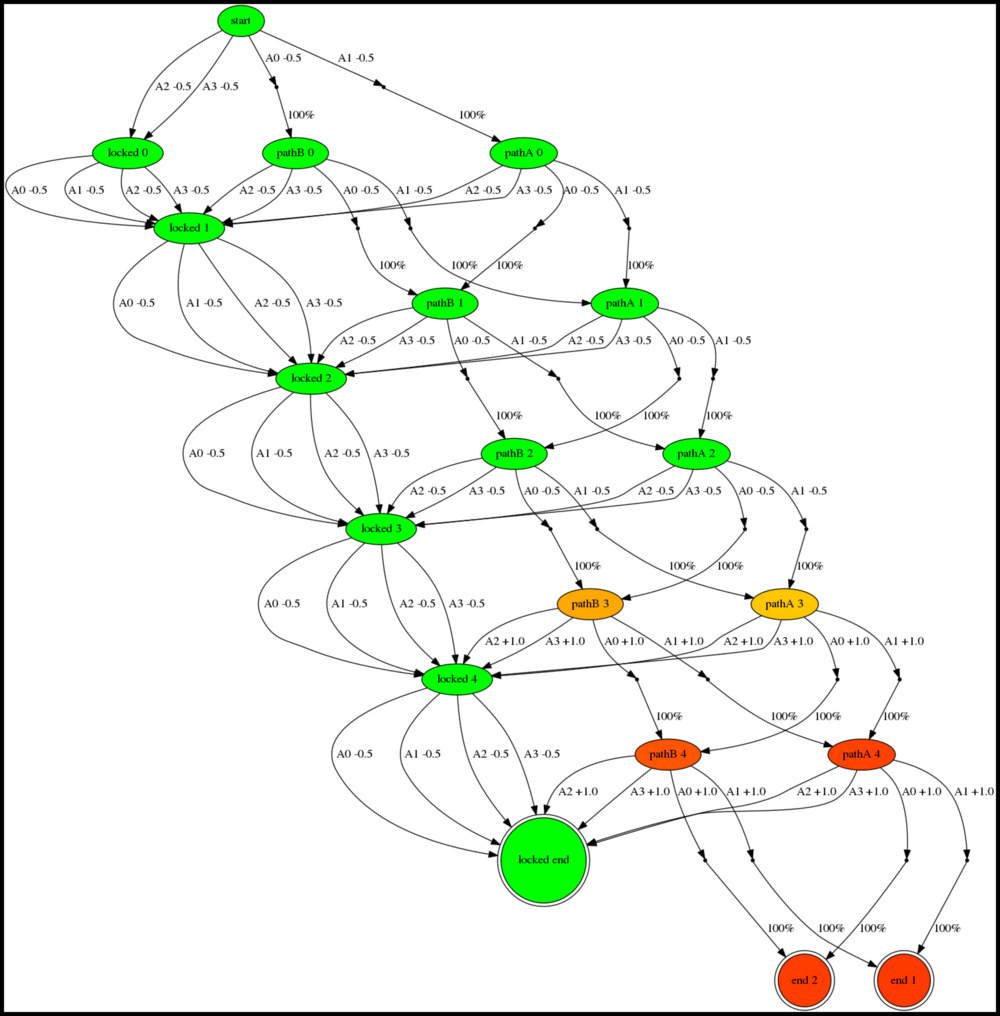}};
\path (C5_1) node[xscale=\xscale,yscale=\yscale]{\includegraphics[width=\widthfig]{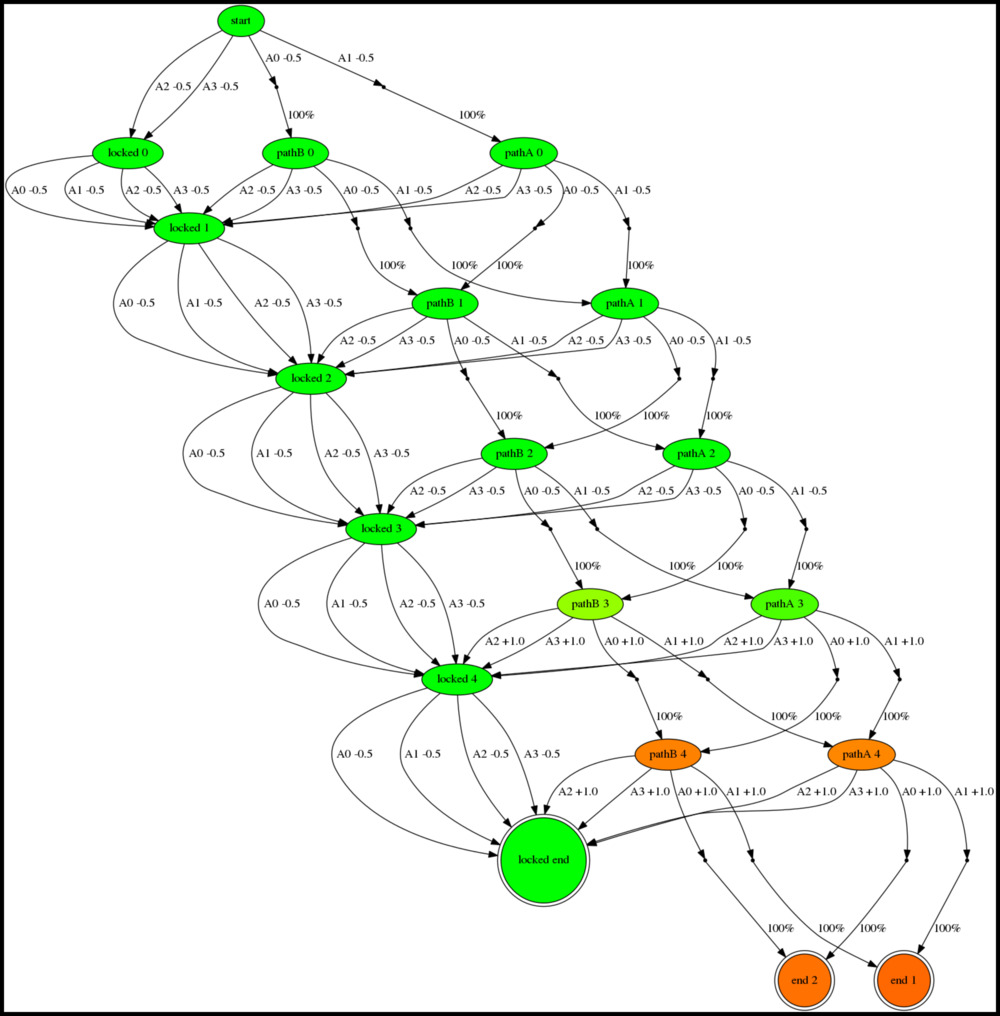}};

\path (C1_8) node[xscale=\xscale,yscale=\yscale]{\includegraphics[width=\widthfig]{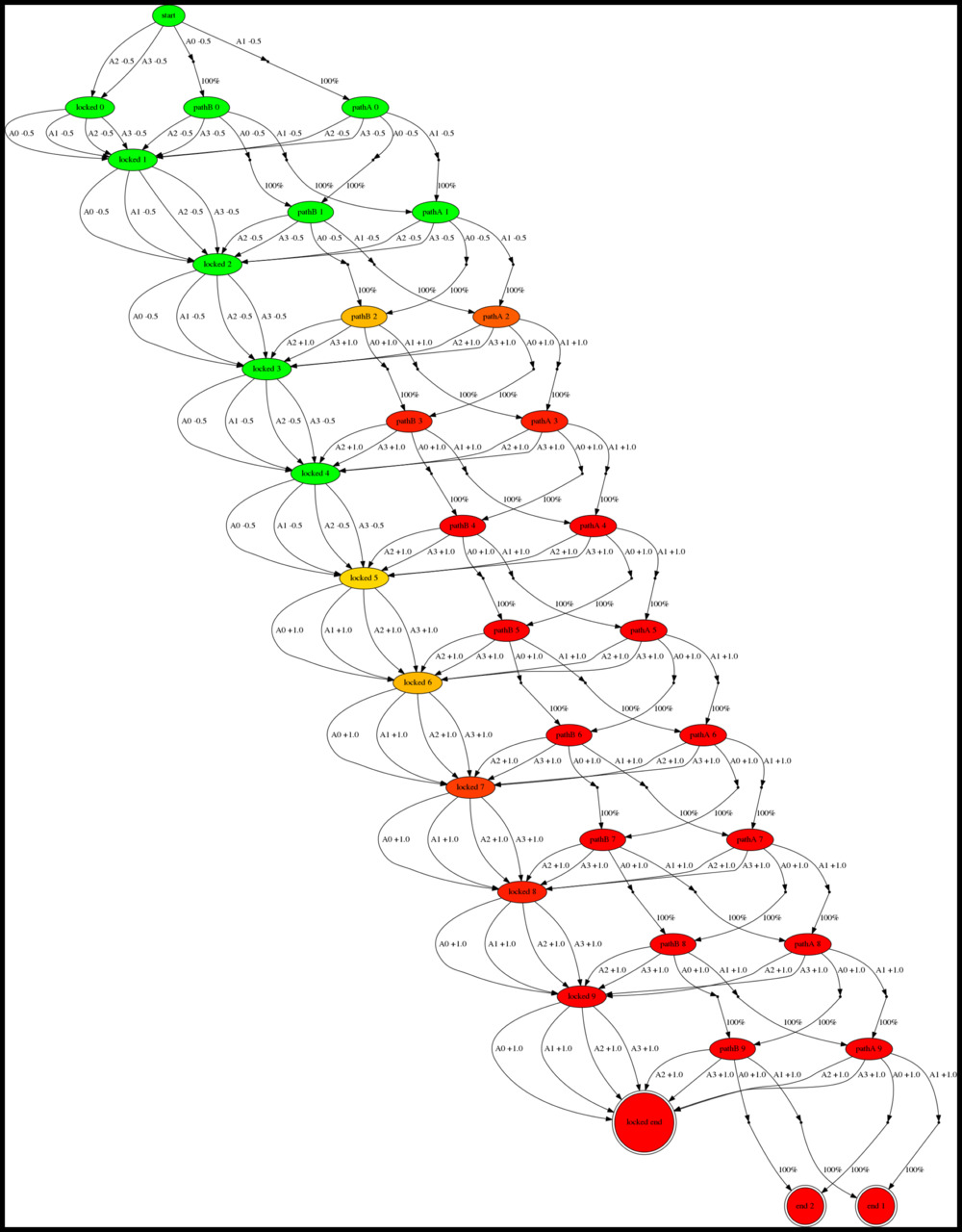}};
\path (C2_8) node[xscale=\xscale,yscale=\yscale]{\includegraphics[width=\widthfig]{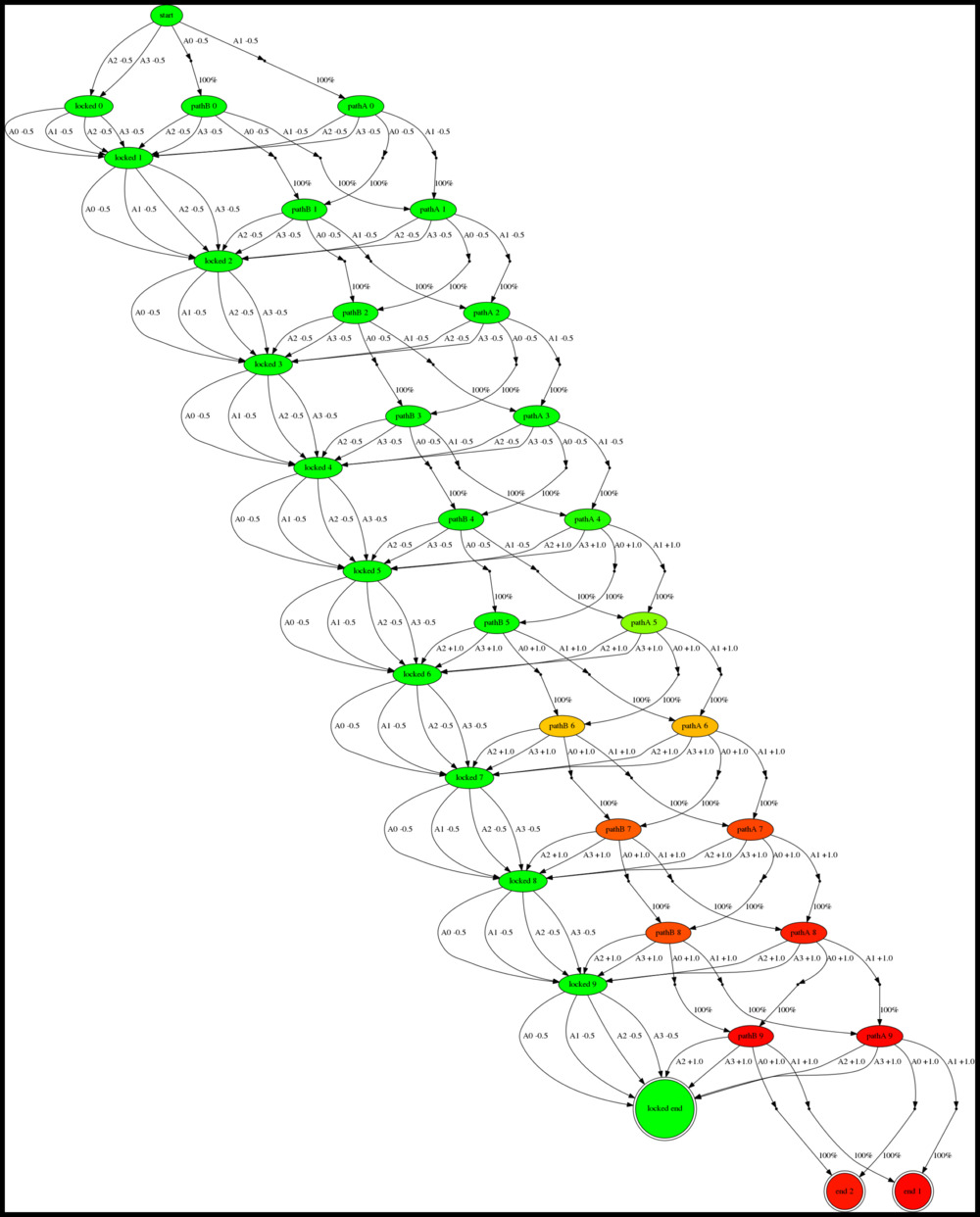}};
\path (C3_8) node[xscale=\xscale,yscale=\yscale] (central2) {\includegraphics[width=\widthfig]{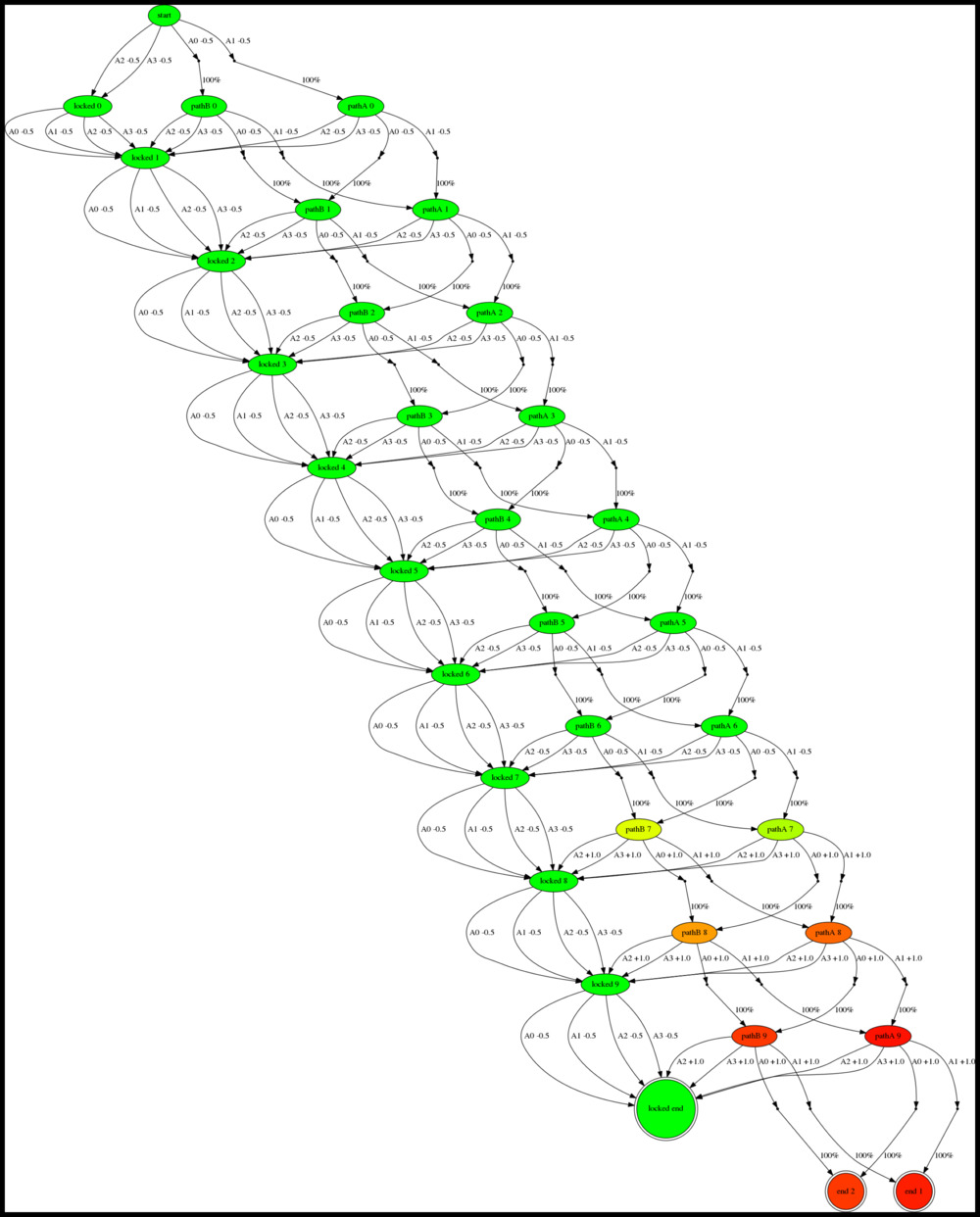}};
\path (C4_8) node[xscale=\xscale,yscale=\yscale]{\includegraphics[width=\widthfig]{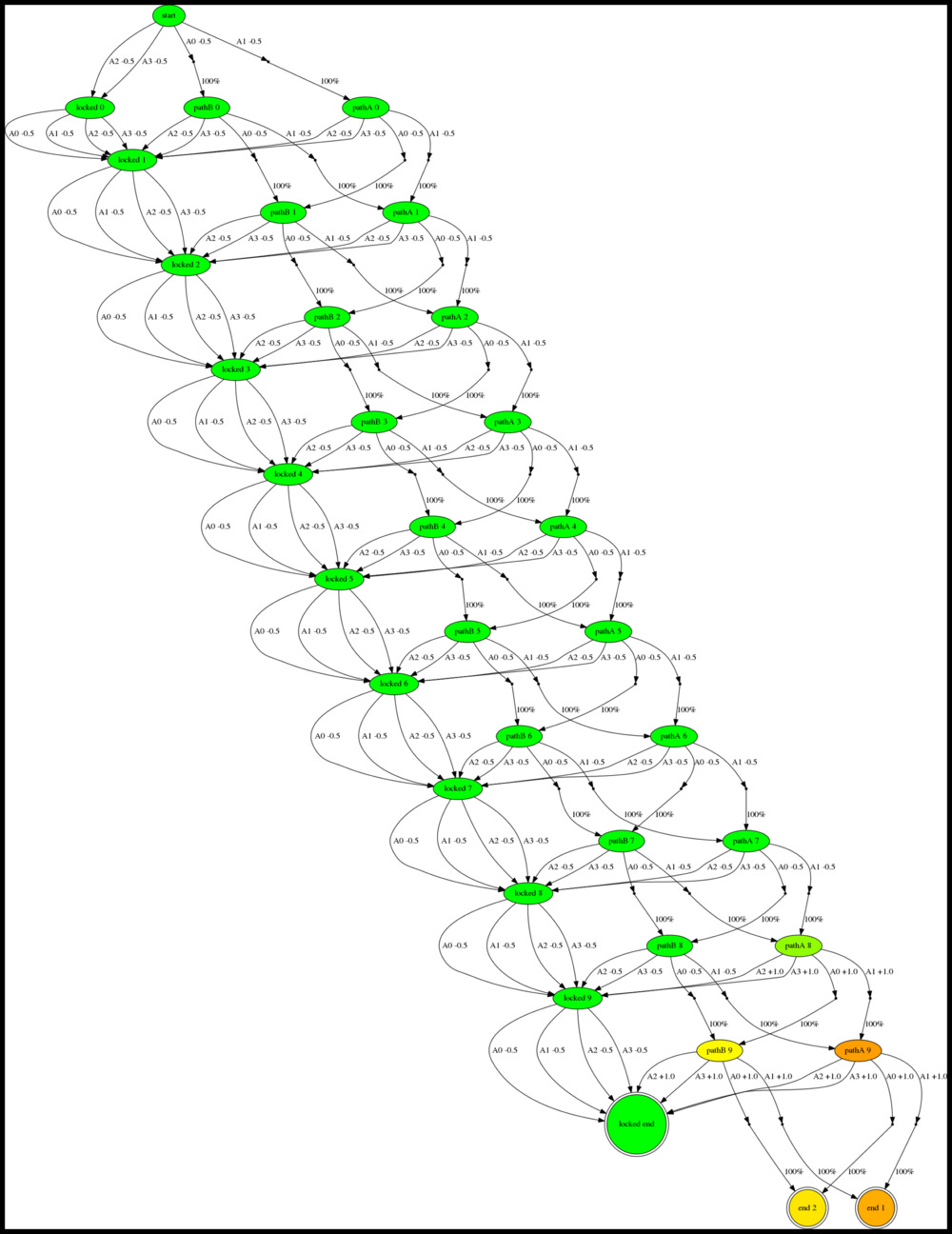}};
\path (C5_8) node[xscale=\xscale,yscale=\yscale]{\includegraphics[width=\widthfig]{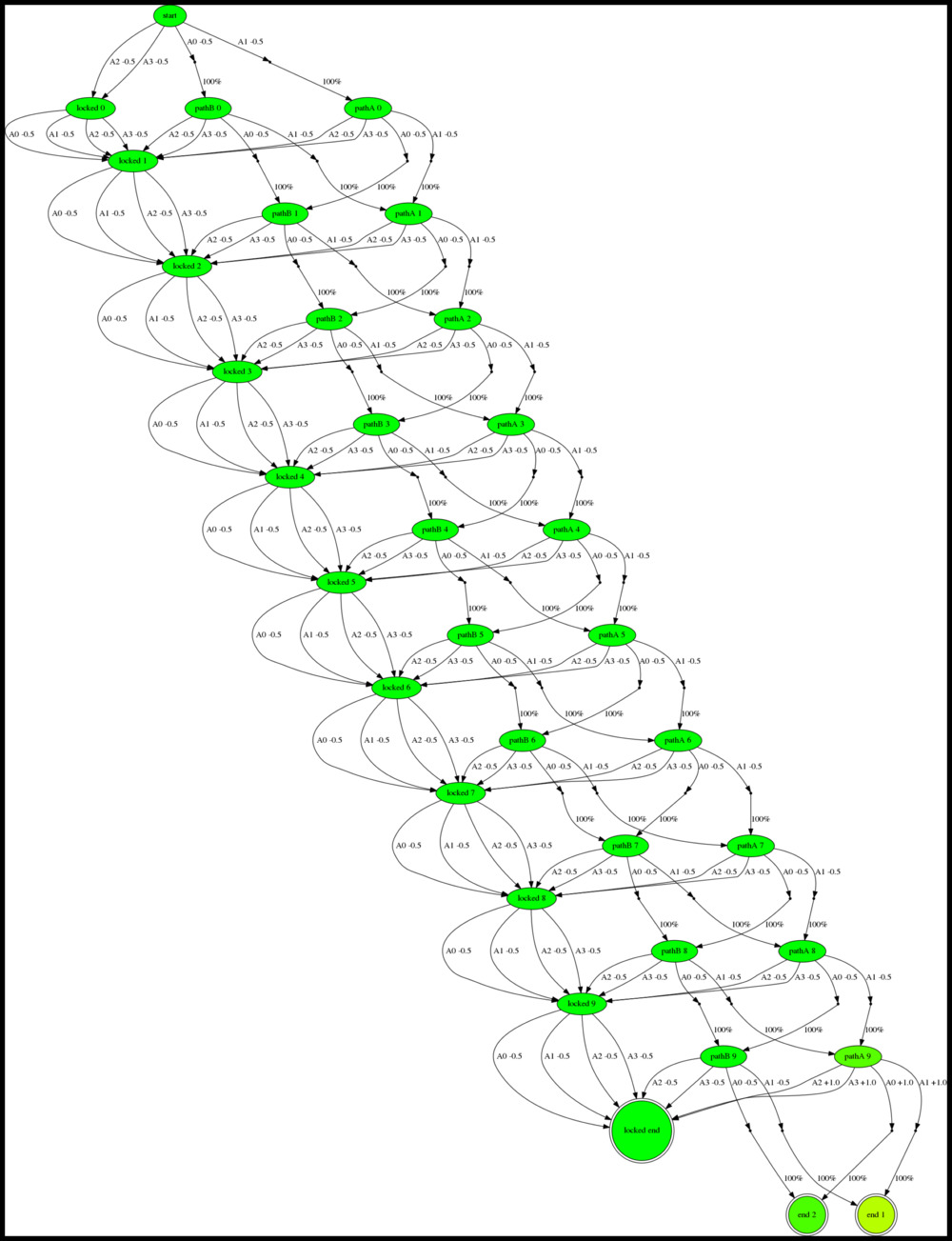}};

\path (C1_16) node[xscale=\xscale,yscale=\yscale]{\includegraphics[width=\widthfig]{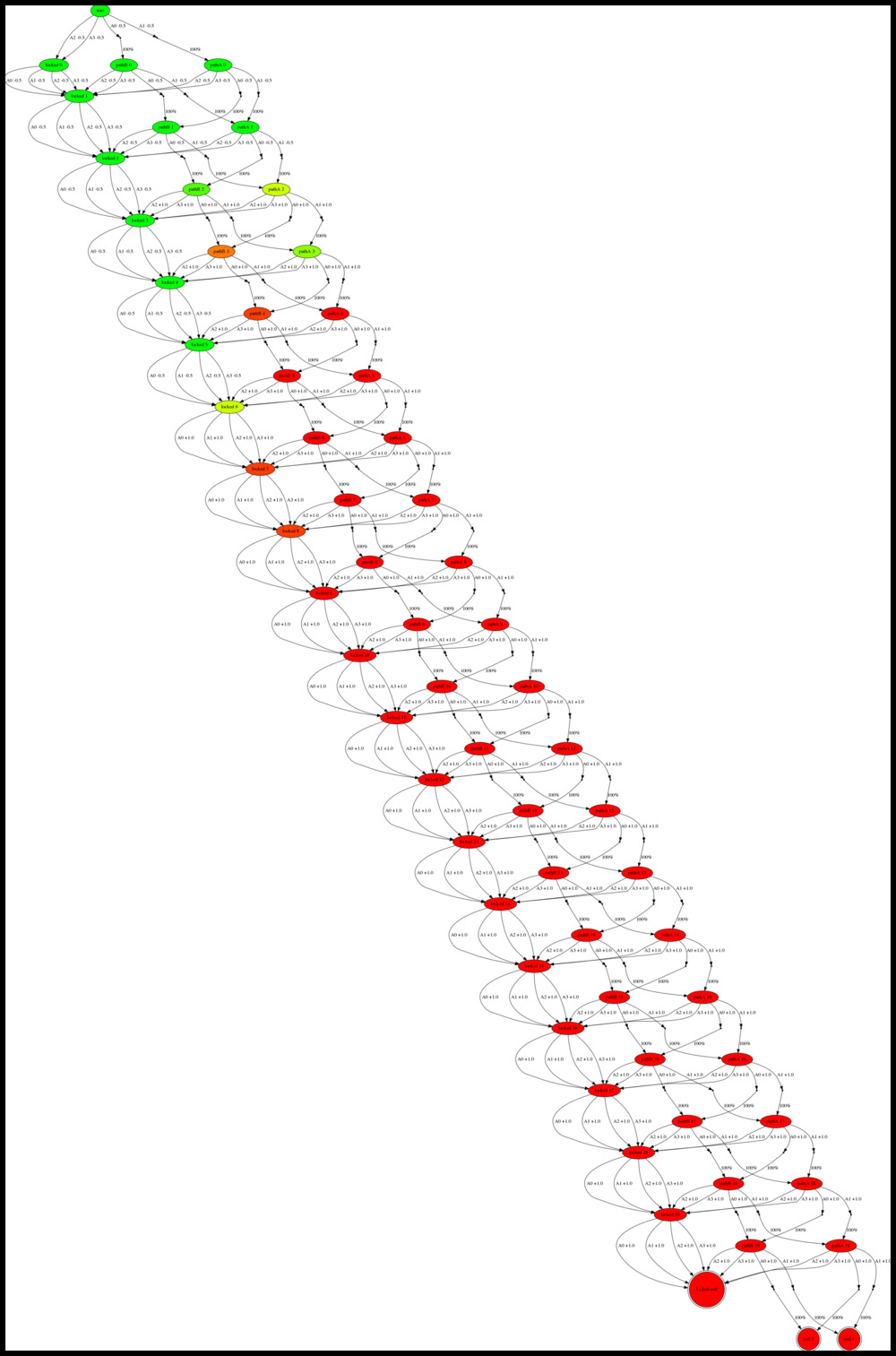}};
\path (C2_16) node[xscale=\xscale,yscale=\yscale]{\includegraphics[width=\widthfig]{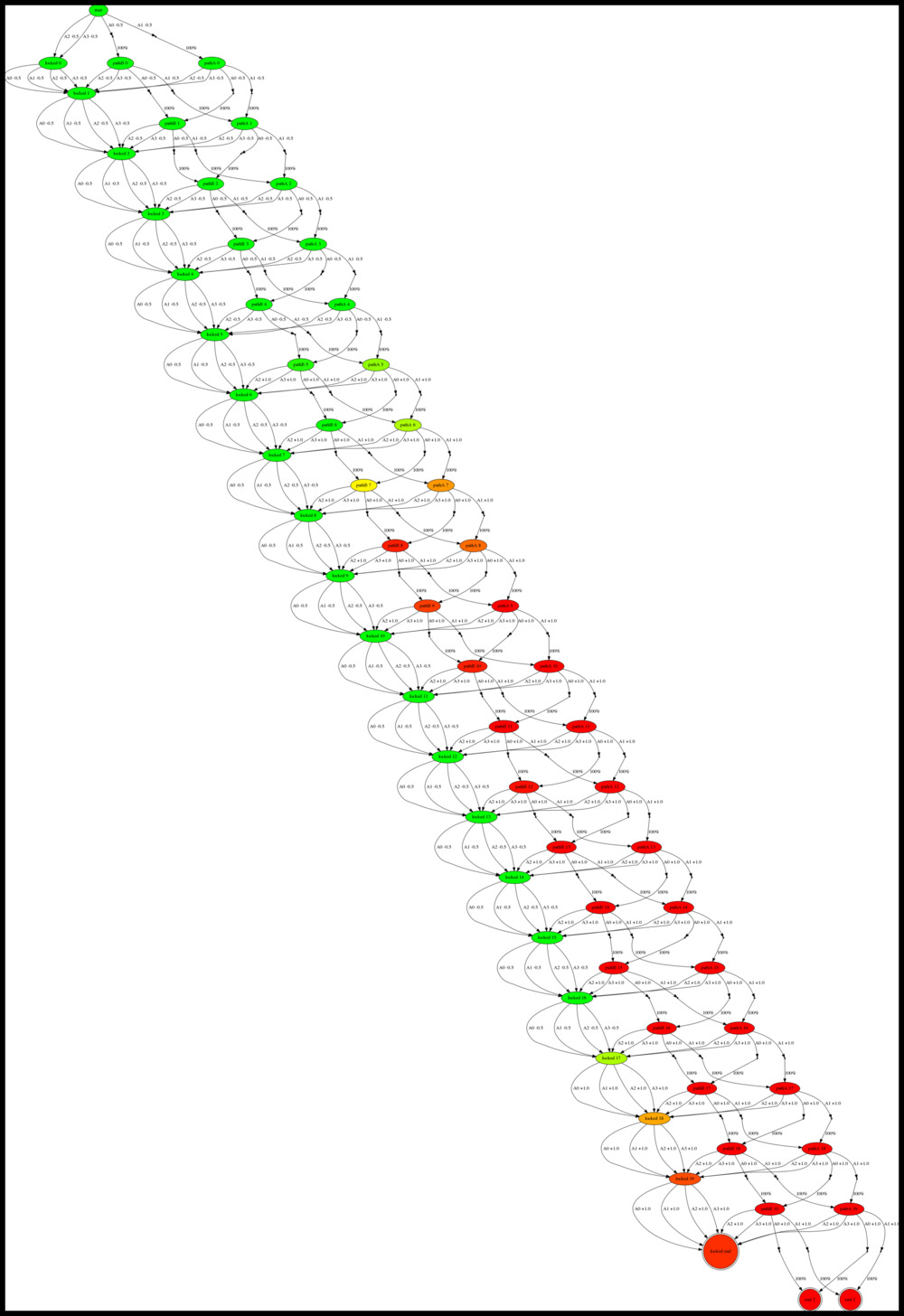}};
\path (C3_16) node[xscale=\xscale,yscale=\yscale] (central3) {\includegraphics[width=\widthfig]{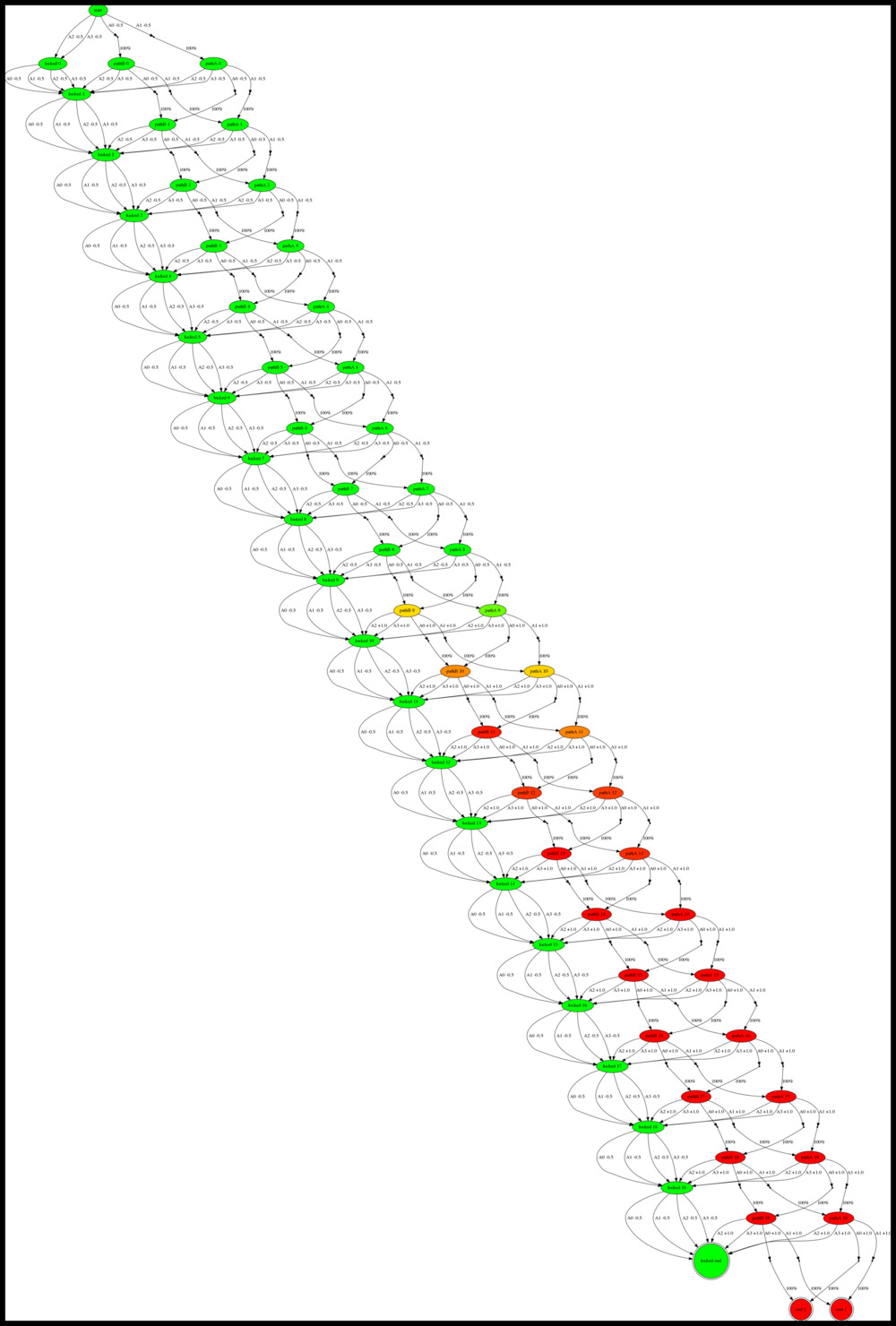}};
\path (C4_16) node[xscale=\xscale,yscale=\yscale]{\includegraphics[width=\widthfig]{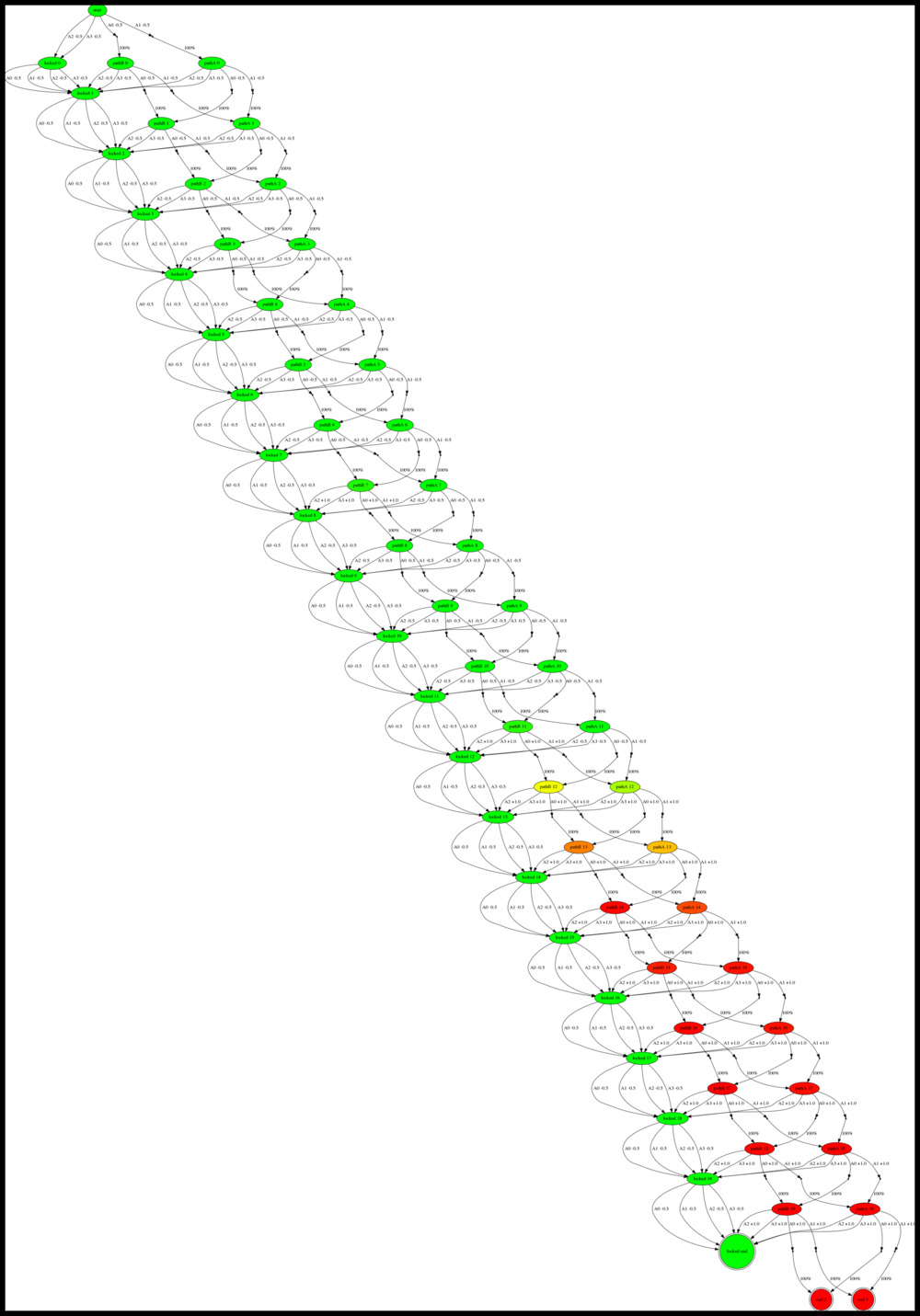}};
\path (C5_16) node[xscale=\xscale,yscale=\yscale]{\includegraphics[width=\widthfig]{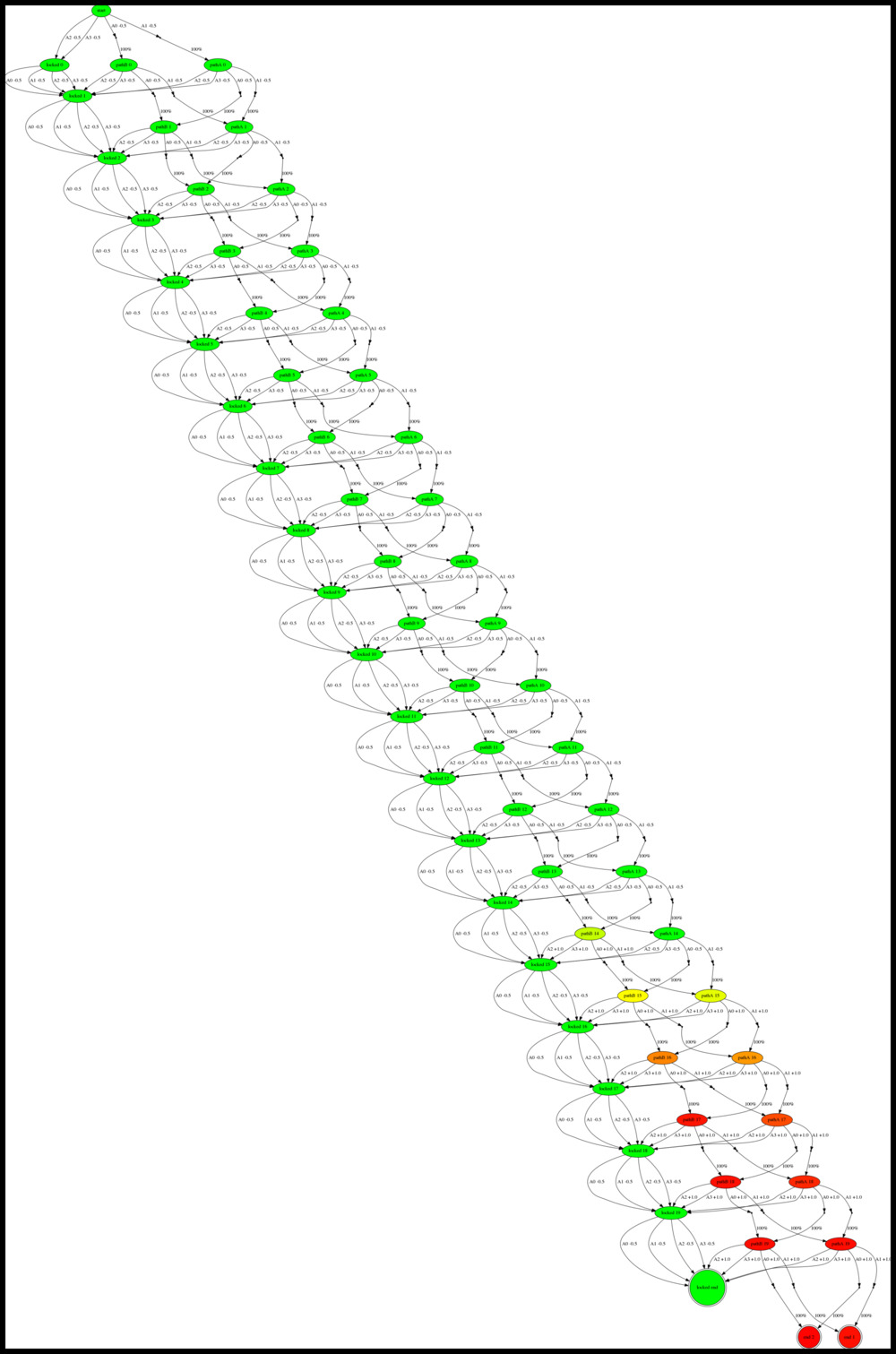}};

\path (central1.north) + (0,-\vcaption) node[format]{\texttt{CE($s_0$,TRPO)} on Diabolical Combination Lock, depth $d=5$, iteration $n=1,2,3,4,5$ from left to right.};
\path (central2.north) + (0,-\vcaption) node[format]{\texttt{CE($s_0$,TRPO)} on Diabolical Combination Lock, depth $d=10$, iteration $n=1,4,6,8,10$ from left to right.};
\path (central3.north) + (0,-\vcaption) node[format]{\texttt{CE($s_0$,TRPO)} on Diabolical Combination Lock, depth $d=20$, iteration $n=1,4,7,10,14$ from left to right.};
\end{tikzpicture}

\clearpage

%	\caption{Visit distribution $\mu_n$ simulated by CE for $n= 1, 3, 6, 10, 20$ steps respectively. Colors range from green for the most visited to red for the less visited, proportionally to the visit distribution value. In the supplementary material you can find a video showing the exploration evolution.}\label{fig:D20evolution}

\renewcommand{\dirfig}{./img/scatterplots/}

\begin{tikzpicture}[overlay] %[node distance=3cm,auto]
\label{page:scatterplots_reinforcecct}
\tikzstyle{format} = [anchor = south]
\path (.5\columnwidth,-.5\textheight) coordinate (center);
\def\vstep{1.12cm}
\def\hstep{4cm}
\def\widthfig{1.1\linewidth}
\def\caption{0.4cm}
\path (-5cm,1.3cm) coordinate (UL);
\foreach\t in {0,...,10} \foreach\s in {0,...,20} \path (UL) + (\t*\hstep,-\s*\vstep) coordinate (C\t_\s);
%	\foreach\t in {0,...,10} \foreach\s in {0,...,6} \path (C\t_\s) node {(\t,-\s)};

\def\scale{1}
\path (C3_3) node[scale=\scale]{\includegraphics[width=\widthfig]{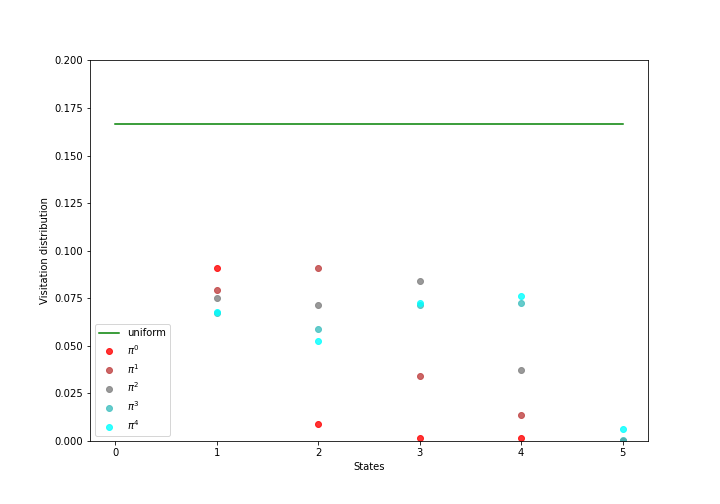}};
\path (C3_11) node[scale=\scale] (central) {\includegraphics[width=\widthfig]{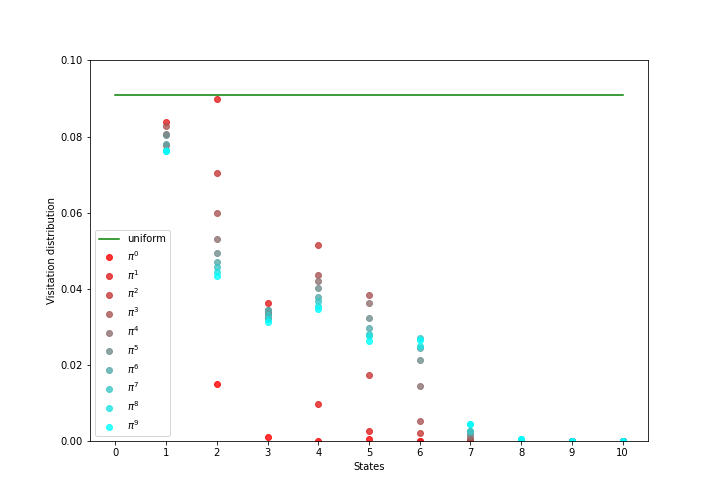}};
\path (C3_19) node[scale=\scale]{\includegraphics[width=\widthfig]{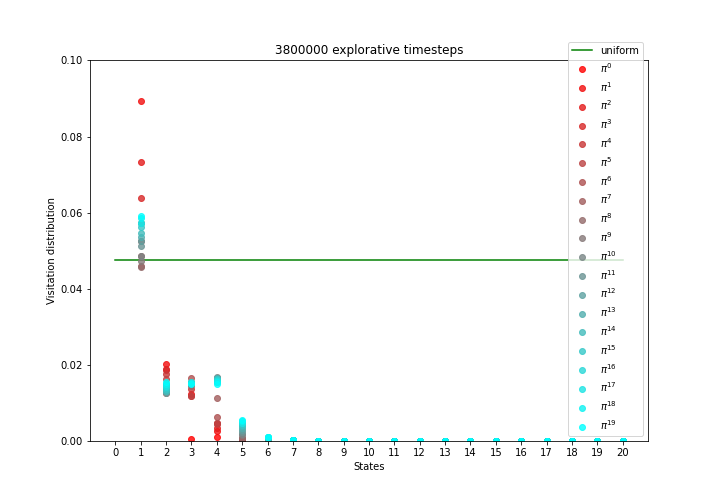}};

\path (central.east) + (-\caption,0) node[scale=1.7,format,rotate=-90,text width=12.3cm]{Scatterplots of visit distributions of \texttt{CE($s_0$,\REINFORCE)} on Consecutive Crossroad Traps of depths $d=5$ (upper), $d=10$ (middle), $d=20$ (lower). States $s$ on the $x$-axis, values of $\mu_n(s)$ on the $y$-axis, colors red $\rightarrow$ blue for steps $n=0\rightarrow d-1$.};
%\path (central.east) + (-\caption,0) node[scale=1.7,format,rotate=-90,text width=12.3cm]{Scatterplots of visit distributions of \texttt{CE($s_0$,\REINFORCE)} on Consecutive Crossroad Traps of depths $d=5$ (upper), $d=10$ (middle). States $s$ on the $x$-axis, values of $\mu_n(s)$ on the $y$-axis, colors red $\rightarrow$ blue for steps $n=0\rightarrow d-1$.};
\end{tikzpicture}

\clearpage

\begin{tikzpicture}[overlay] %[node distance=3cm,auto]
\label{page:scatterplots_reinforcedcl}
\tikzstyle{format} = [anchor = south]
\path (.5\columnwidth,-.5\textheight) coordinate (center);
\def\vstep{1.12cm}
\def\hstep{4cm}
\def\widthfig{1.1\linewidth}
\def\caption{0.4cm}
\path (-5cm,1.3cm) coordinate (UL);
\foreach\t in {0,...,10} \foreach\s in {0,...,20} \path (UL) + (\t*\hstep,-\s*\vstep) coordinate (C\t_\s);
%	\foreach\t in {0,...,10} \foreach\s in {0,...,6} \path (C\t_\s) node {(\t,-\s)};

\def\scale{1}
\path (C3_3) node[scale=\scale]{\includegraphics[width=\widthfig]{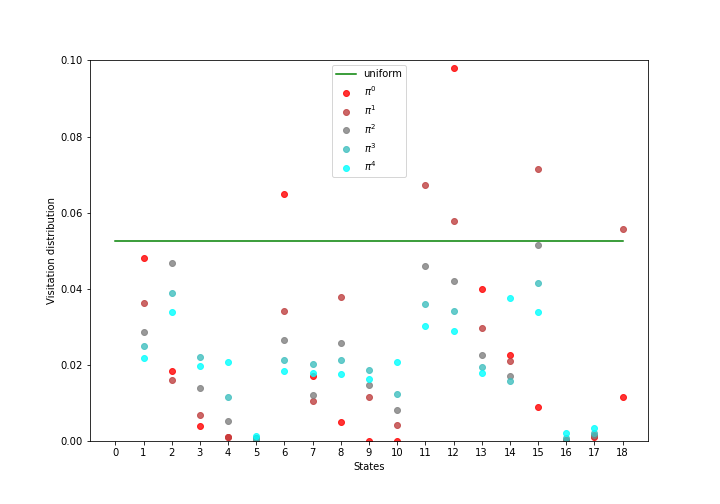}};
\path (C3_11) node[scale=\scale] (central) {\includegraphics[width=\widthfig]{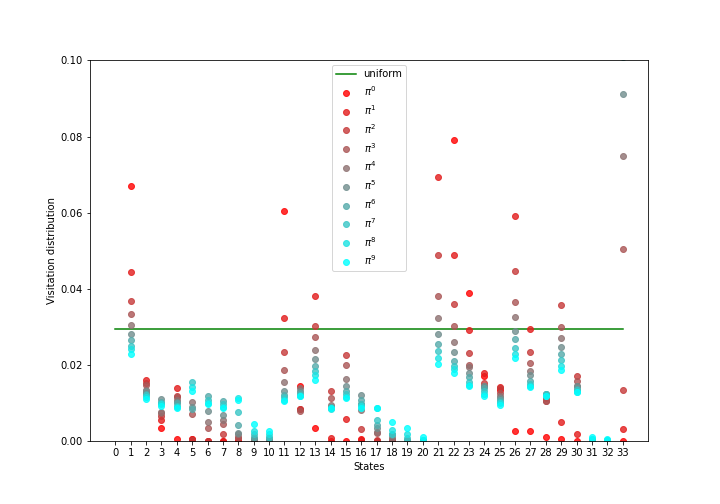}};
\path (C3_19) node[scale=\scale]{\includegraphics[width=\widthfig]{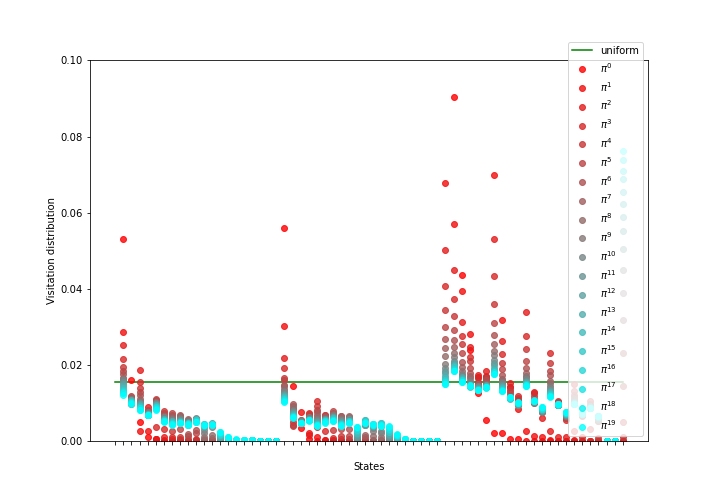}};

\path (central.east) + (-\caption,0) node[scale=1.7,format,rotate=-90,text width=12.3cm]{Scatterplots of visit distributions of \texttt{CE($s_0$,\REINFORCE)} on Diabolical Combination Lock of depths $d=5$ (upper), $d=10$ (middle), $d=20$ (lower). States $s$ on the $x$-axis, values of $\mu_n(s)$ on the $y$-axis, colors red $\rightarrow$ blue for steps $n=0\rightarrow d-1$.};
\end{tikzpicture}

\clearpage

\begin{tikzpicture}[overlay] %[node distance=3cm,auto]
\label{page:scatterplots_trpocct}
\tikzstyle{format} = [anchor = south]
\path (.5\columnwidth,-.5\textheight) coordinate (center);
\def\vstep{1.12cm}
\def\hstep{4cm}
\def\widthfig{1.1\linewidth}
\def\caption{0.4cm}
\path (-5cm,1.3cm) coordinate (UL);
\foreach\t in {0,...,10} \foreach\s in {0,...,20} \path (UL) + (\t*\hstep,-\s*\vstep) coordinate (C\t_\s);
%	\foreach\t in {0,...,10} \foreach\s in {0,...,6} \path (C\t_\s) node {(\t,-\s)};

\def\scale{1}
\path (C3_3) node[scale=\scale]{\includegraphics[width=\widthfig]{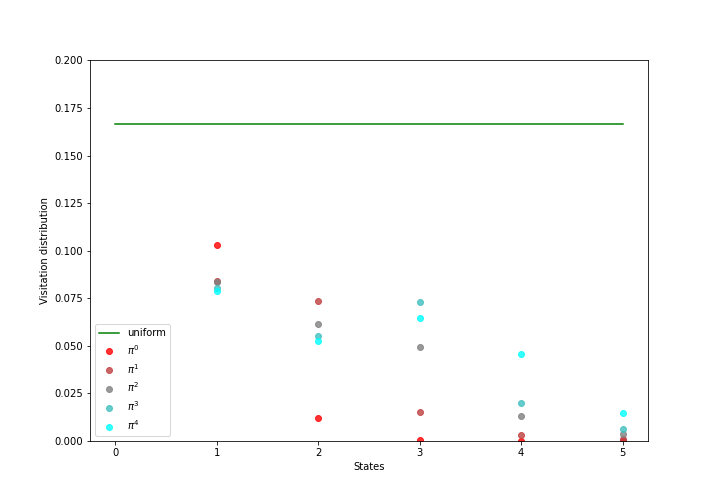}};
\path (C3_11) node[scale=\scale] (central) {\includegraphics[width=\widthfig]{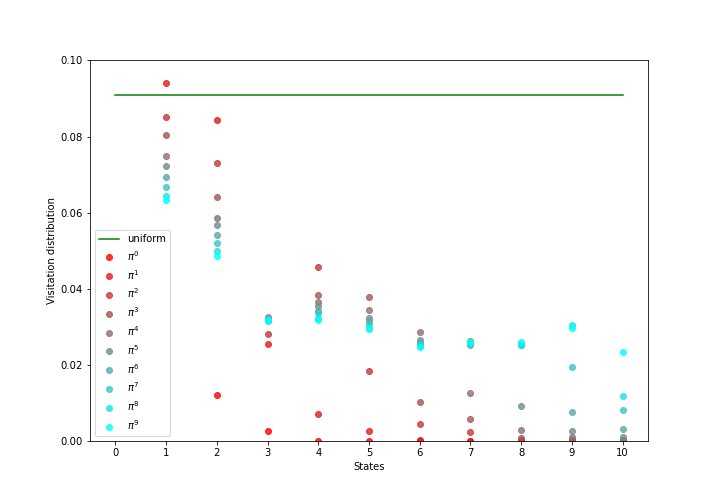}};
\path (C3_19) node[scale=\scale]{\includegraphics[width=\widthfig]{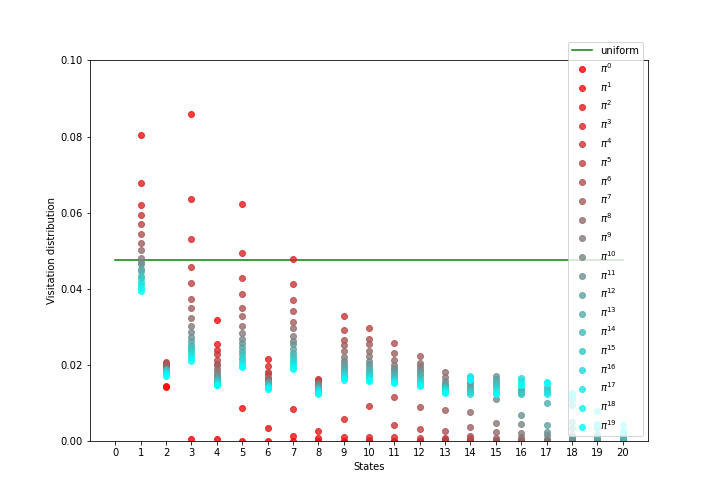}};

\path (central.east) + (-\caption,0) node[scale=1.7,format,rotate=-90,text width=12.3cm]{Scatterplots of visit distributions of \texttt{CE($s_0$,TRPO)} on Consecutive Crossroad Traps of depths $d=5$ (upper), $d=10$ (middle), $d=20$ (lower). States $s$ on the $x$-axis, values of $\mu_n(s)$ on the $y$-axis, colors red $\rightarrow$ blue for steps $n=0\rightarrow d-1$.};
%\path (central.east) + (-\caption,0) node[scale=1.7,format,rotate=-90,text width=12.3cm]{Scatterplots of visit distributions of \texttt{CE($s_0$,TRPO)} on Consecutive Crossroad Traps of depths $d=5$ (upper), $d=10$ (middle). States $s$ on the $x$-axis, values of $\mu_n(s)$ on the $y$-axis, colors red $\rightarrow$ blue for steps $n=0\rightarrow d-1$.};
\end{tikzpicture}

\clearpage

\begin{tikzpicture}[overlay] %[node distance=3cm,auto]
\label{page:scatterplots_trpodcl}
\tikzstyle{format} = [anchor = south]
\path (.5\columnwidth,-.5\textheight) coordinate (center);
\def\vstep{1.12cm}
\def\hstep{4cm}
\def\widthfig{1.1\linewidth}
\def\caption{0.4cm}
\path (-5cm,1.3cm) coordinate (UL);
\foreach\t in {0,...,10} \foreach\s in {0,...,20} \path (UL) + (\t*\hstep,-\s*\vstep) coordinate (C\t_\s);
%	\foreach\t in {0,...,10} \foreach\s in {0,...,6} \path (C\t_\s) node {(\t,-\s)};

\def\scale{1}
\path (C3_3) node[scale=\scale]{\includegraphics[width=\widthfig]{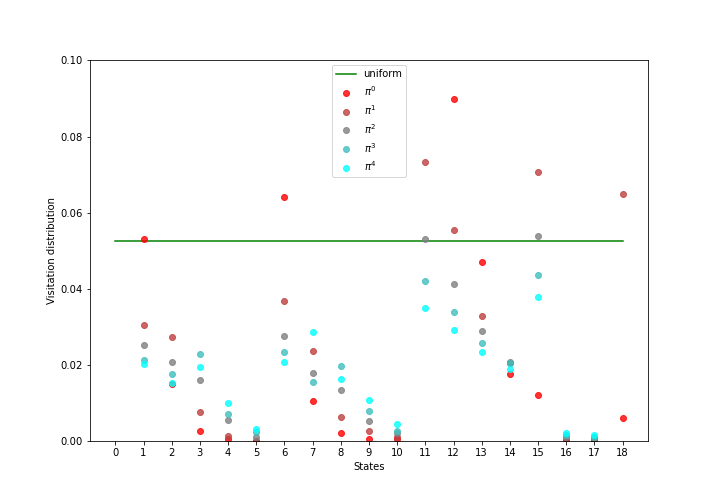}};
\path (C3_11) node[scale=\scale] (central) {\includegraphics[width=\widthfig]{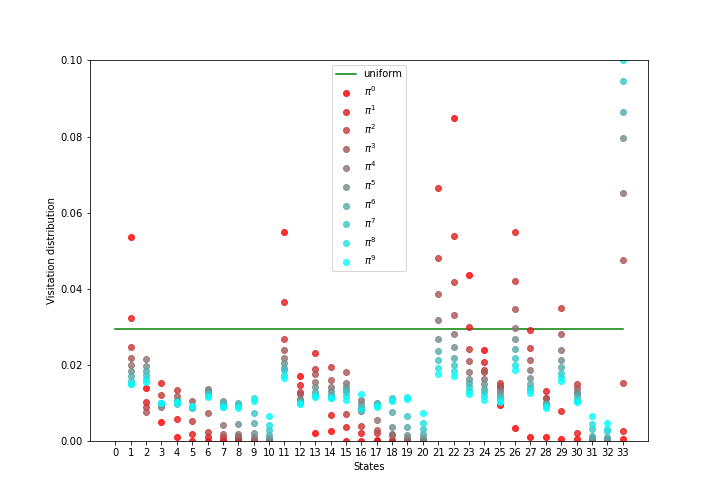}};
\path (C3_19) node[scale=\scale]{\includegraphics[width=\widthfig]{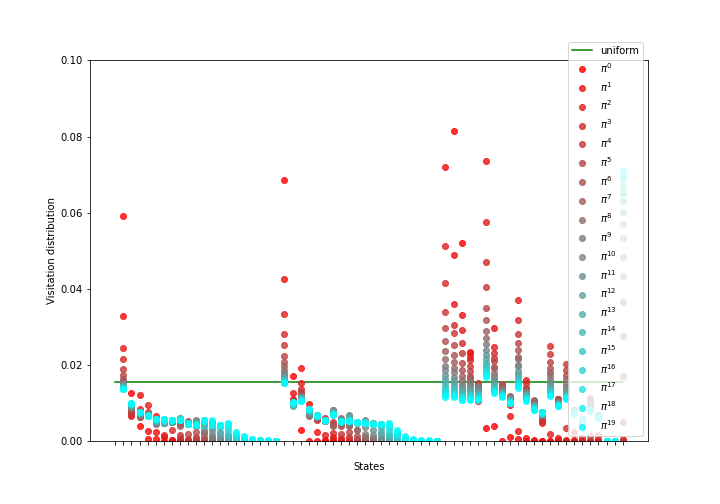}};

\path (central.east) + (-\caption,0) node[scale=1.7,format,rotate=-90,text width=12.3cm]{Scatterplots of visit distributions of \texttt{CE($s_0$,TRPO)} on Diabolical Combination Lock of depths $d=5$ (upper), $d=10$ (middle), $d=20$ (lower). States $s$ on the $x$-axis, values of $\mu_n(s)$ on the $y$-axis, colors red $\rightarrow$ blue for steps $n=0\rightarrow d-1$.};
\end{tikzpicture}

\clearpage

\renewcommand{\dirfig}{./img/COMP/}

\begin{tikzpicture}[overlay] %[node distance=3cm,auto]
\label{page:COMP_reinforcecct}
\tikzstyle{format} = [anchor = south]
\path (.5\columnwidth,-.5\textheight) coordinate (center);
\def\vstep{1.12cm}
\def\hstep{4cm}
\def\widthfig{1.1\linewidth}
\def\caption{0.4cm}
\path (-5cm,1.3cm) coordinate (UL);
\foreach\t in {0,...,10} \foreach\s in {0,...,20} \path (UL) + (\t*\hstep,-\s*\vstep) coordinate (C\t_\s);
%	\foreach\t in {0,...,10} \foreach\s in {0,...,6} \path (C\t_\s) node {(\t,-\s)};

\def\scale{1}
\path (C3_3) node[scale=\scale]{\includegraphics[width=\widthfig]{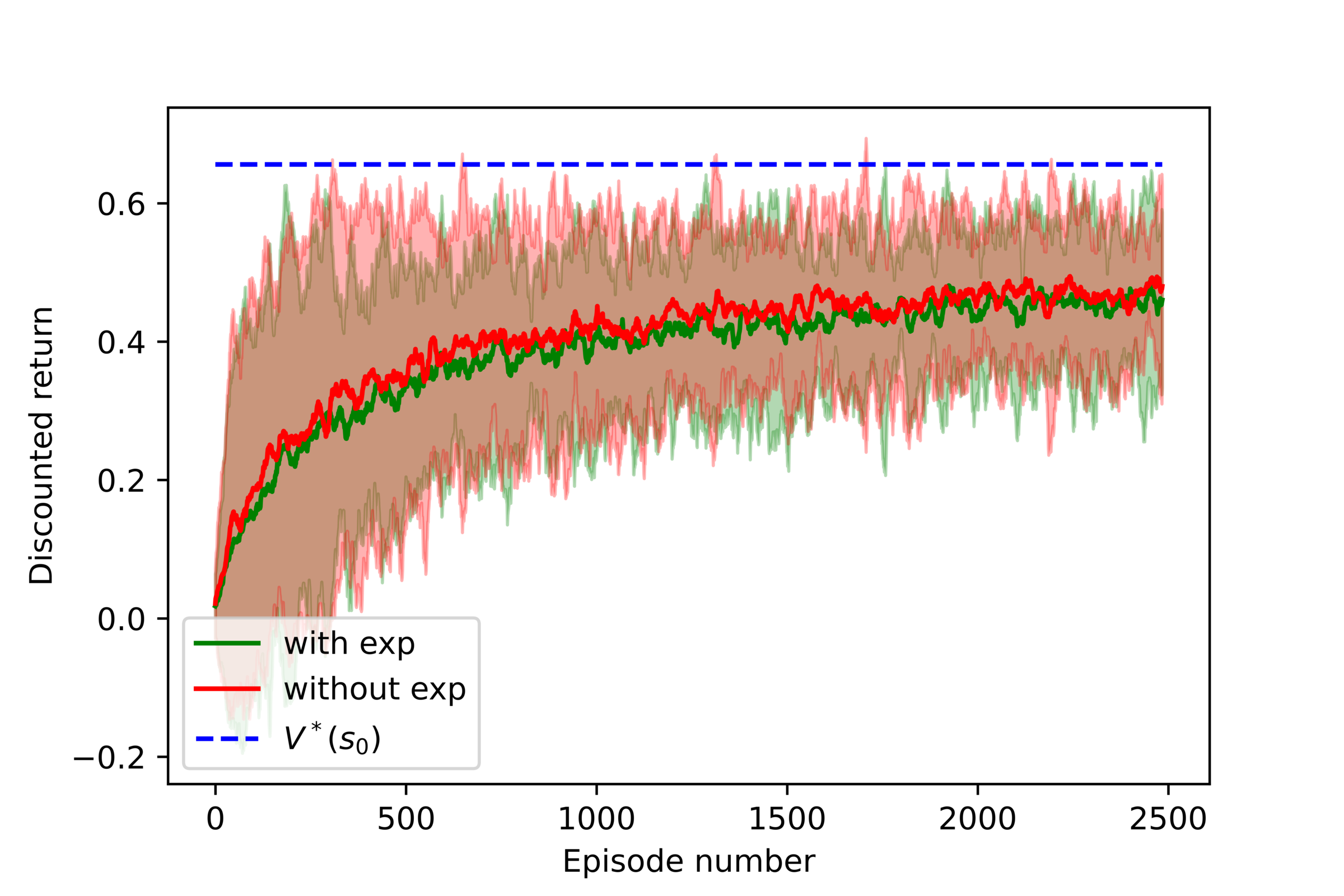}};
\path (C3_11) node[scale=\scale] (central) {\includegraphics[width=\widthfig]{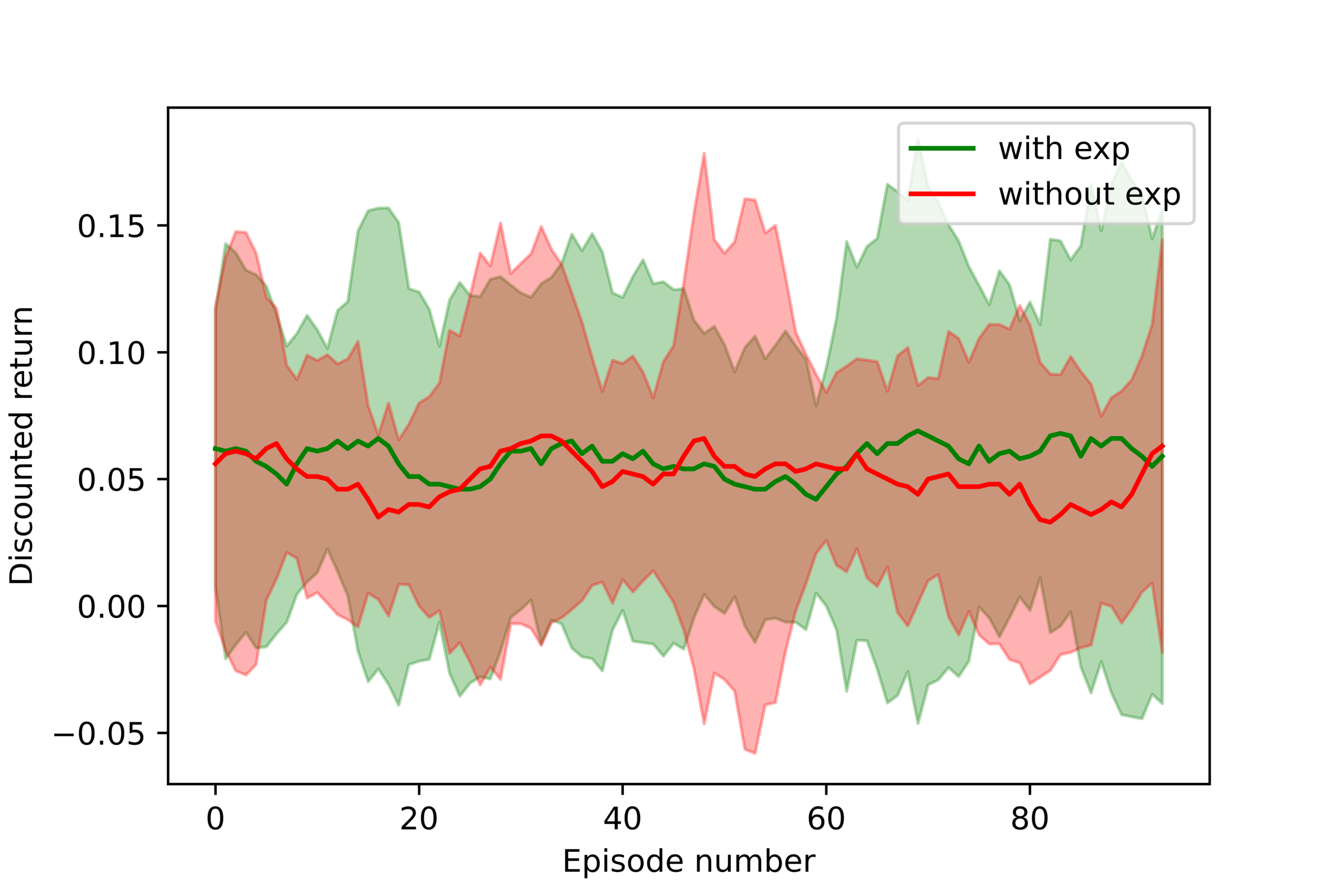}};
\path (C3_19) node[scale=\scale]{\includegraphics[width=\widthfig]{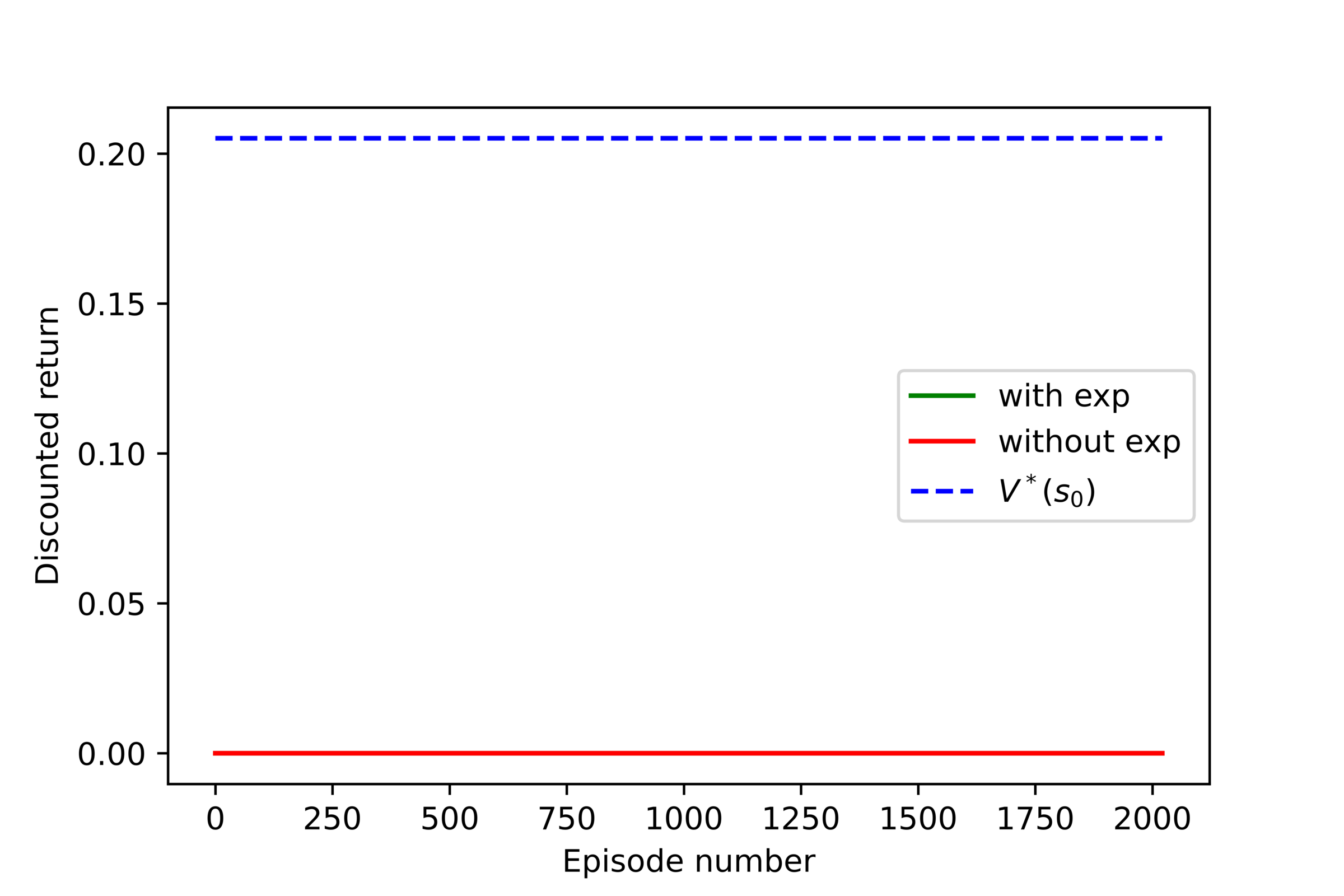}};

\path (central.east) + (-\caption,0) node[scale=1.7,format,rotate=-90,text width=12.5cm]{Average of the discounted return accumulated during episodes of \REINFORCE on CCT5, CCT10 and CCT20, smoothed by an average over 10 different independent runs. Red is \REINFORCE alone, green is \REINFORCE with \texttt{CE($s_0$,\REINFORCE)}-reset model. The solid line is the average, the light area is one standard deviation.};
\end{tikzpicture}

\clearpage

\begin{tikzpicture}[overlay] %[node distance=3cm,auto]
\label{page:COMP_reinforcedcl}
\tikzstyle{format} = [anchor = south]
\path (.5\columnwidth,-.5\textheight) coordinate (center);
\def\vstep{1.12cm}
\def\hstep{4cm}
\def\widthfig{1.1\linewidth}
\def\caption{0.4cm}
\path (-5cm,1.3cm) coordinate (UL);
\foreach\t in {0,...,10} \foreach\s in {0,...,20} \path (UL) + (\t*\hstep,-\s*\vstep) coordinate (C\t_\s);
%	\foreach\t in {0,...,10} \foreach\s in {0,...,6} \path (C\t_\s) node {(\t,-\s)};

\def\scale{1}
\path (C3_3) node[scale=\scale]{\includegraphics[width=\widthfig]{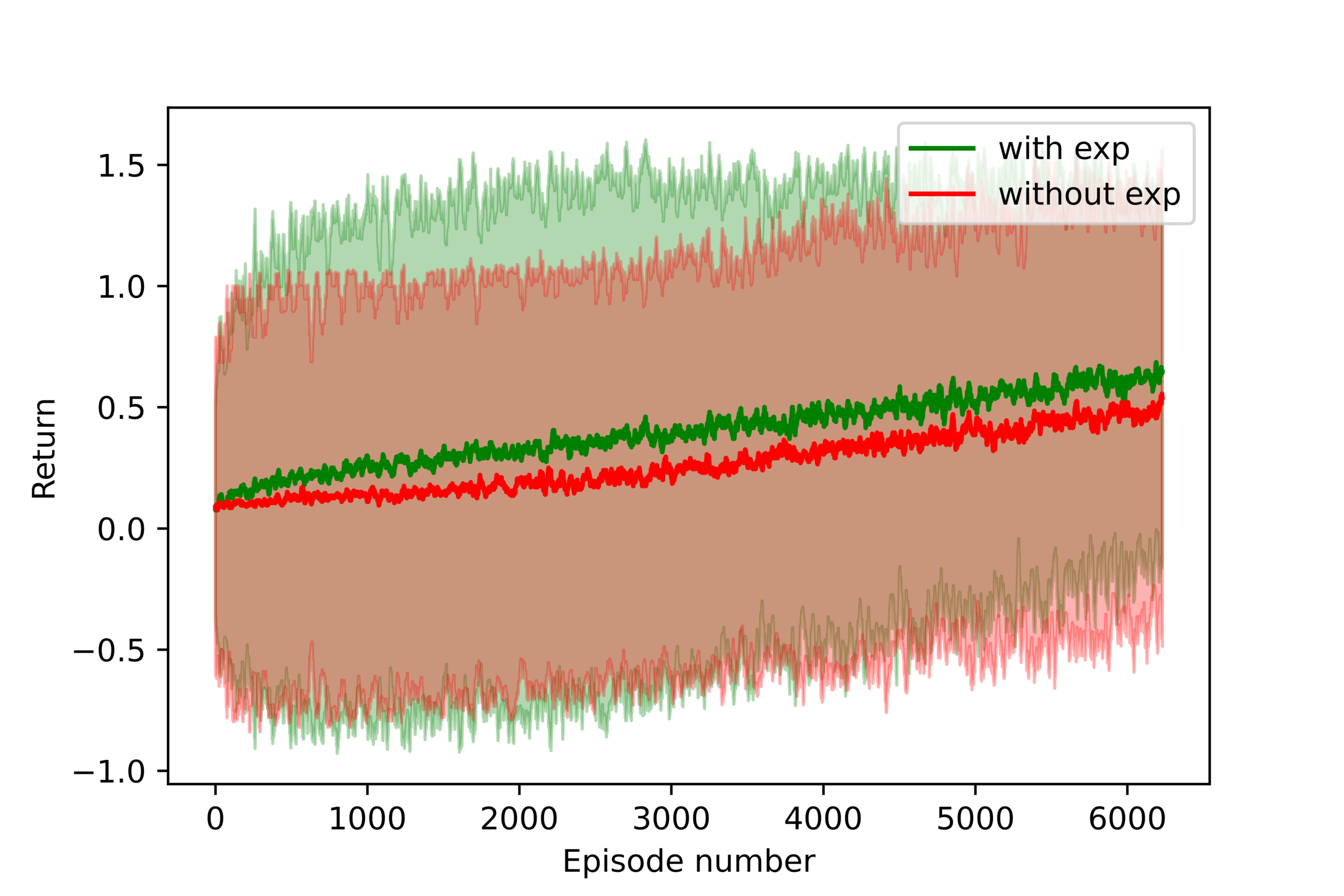}};
\path (C3_11) node[scale=\scale]{\includegraphics[width=\widthfig]{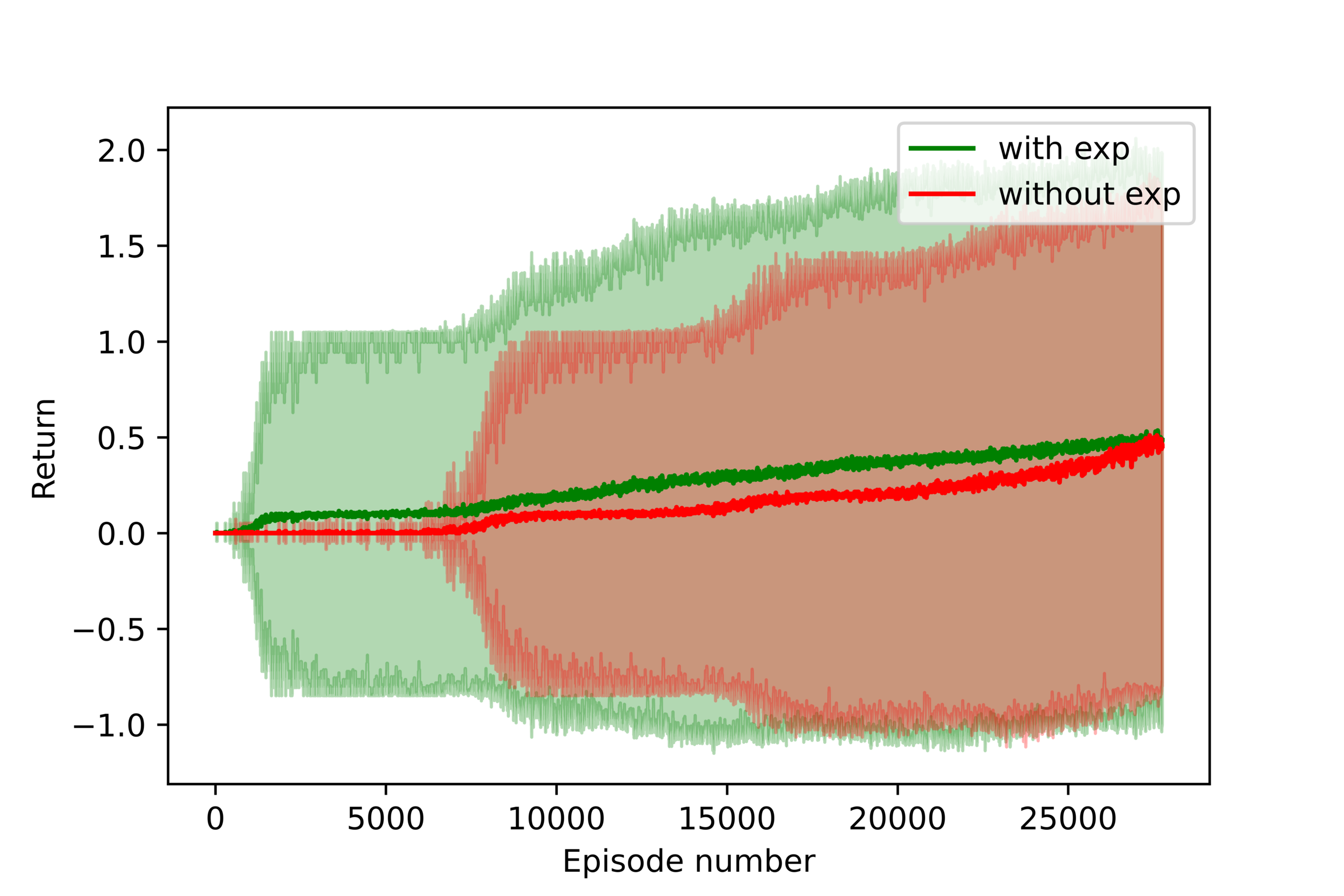}};
\path (C3_19) node[scale=\scale]{\includegraphics[width=\widthfig]{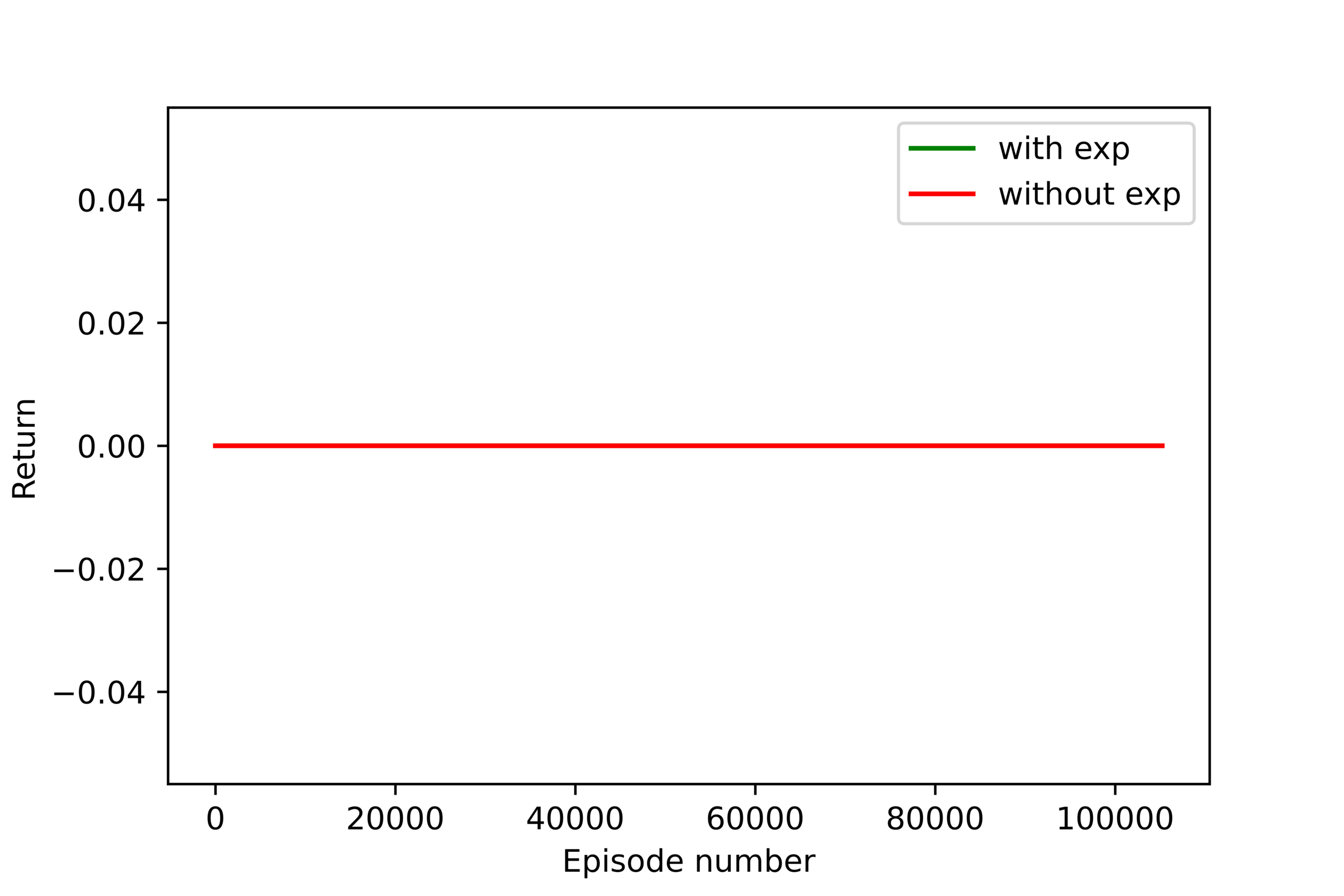}};

\path (central.east) + (-\caption,0) node[scale=1.7,format,rotate=-90,text width=12.5cm]{Average of the (undiscounted) return accumulated during episodes of \REINFORCE on DCL5, DCL10 and DCL20, smoothed by an average over 10 different independent runs. Red is \REINFORCE alone, green is \REINFORCE with \texttt{CE($s_0$,\REINFORCE)}-reset model. The solid line is the average, the light area is one standard deviation.};
\end{tikzpicture}

\clearpage

\begin{tikzpicture}[overlay] %[node distance=3cm,auto]
\label{page:COMP_trpocct}
\tikzstyle{format} = [anchor = south]
\path (.5\columnwidth,-.5\textheight) coordinate (center);
\def\vstep{1.12cm}
\def\hstep{4cm}
\def\widthfig{1.1\linewidth}
\def\caption{0.4cm}
\path (-5cm,1.3cm) coordinate (UL);
\foreach\t in {0,...,10} \foreach\s in {0,...,20} \path (UL) + (\t*\hstep,-\s*\vstep) coordinate (C\t_\s);
%	\foreach\t in {0,...,10} \foreach\s in {0,...,6} \path (C\t_\s) node {(\t,-\s)};

\def\scale{1}
\path (C3_3) node[scale=\scale]{\includegraphics[width=\widthfig]{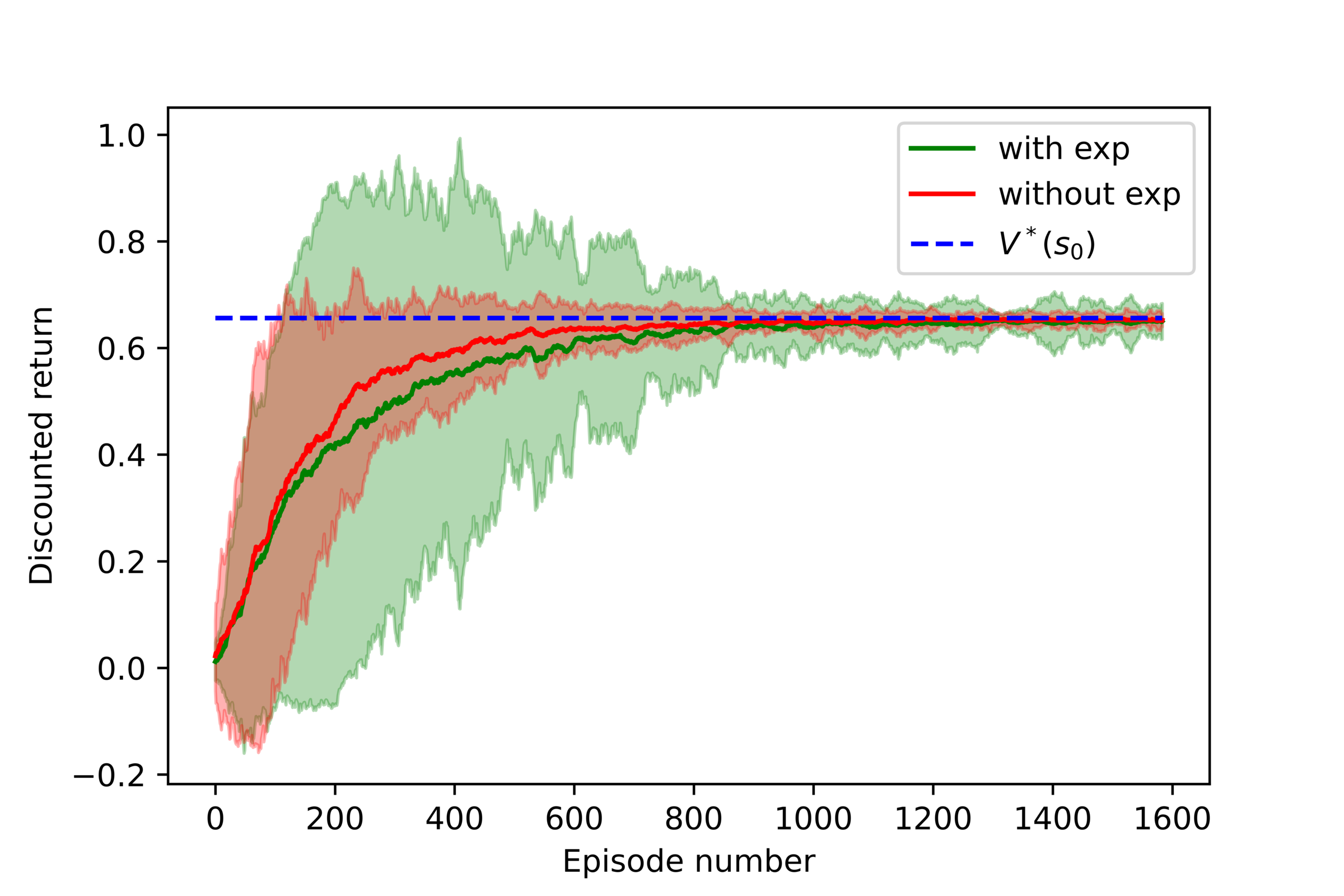}};
\path (C3_11) node[scale=\scale] (central) {\includegraphics[width=\widthfig]{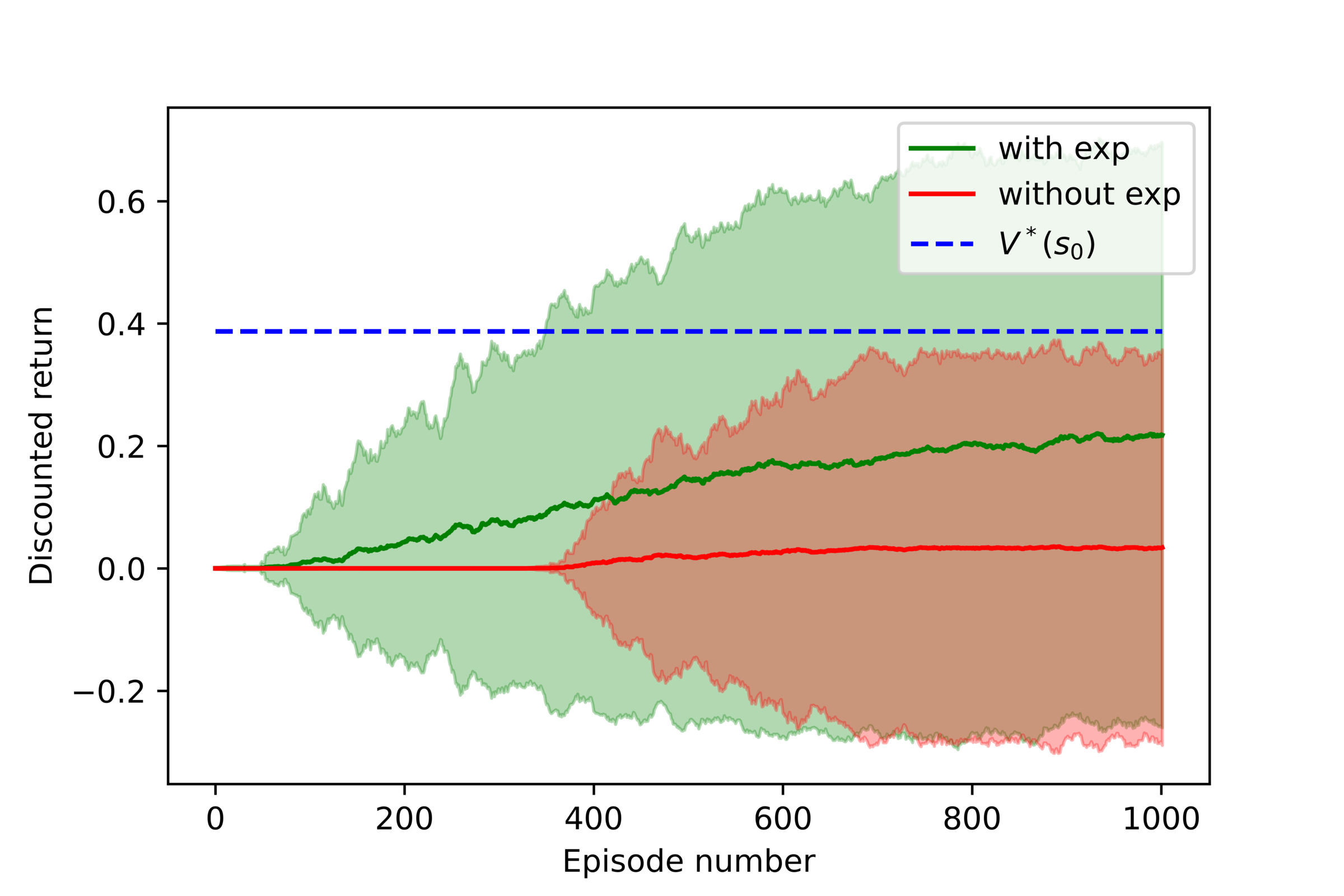}};
\path (C3_19) node[scale=\scale]{\includegraphics[width=\widthfig]{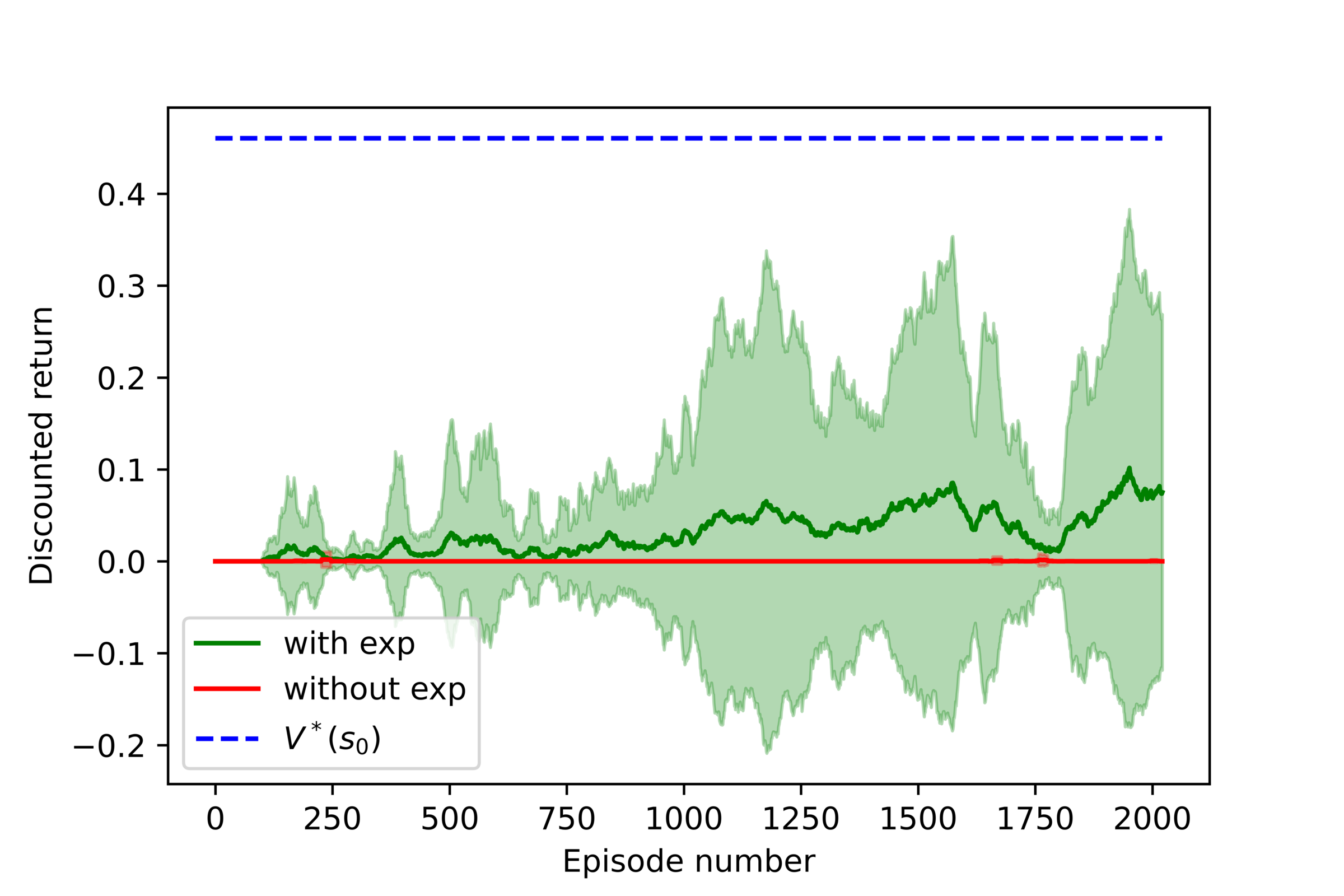}};

%\path (central.east) + (-\caption,0) node[scale=1.7,format,rotate=-90,text width=12.3cm]{Scatterplots of visit distributions of \texttt{CE($s_0$,TRPO)} on Consecutive Crossroad Traps of depths $d=5$ (upper), $d=10$ (middle), $d=20$ (lower). States $s$ on the $x$-axis, values of $\mu_n(s)$ on the $y$-axis, colors red $\rightarrow$ blue for steps $n=0\rightarrow d-1$.};
\path (central.east) + (-\caption,0) node[scale=1.7,format,rotate=-90,text width=12.5cm]{Average of the discounted return accumulated during episodes of \texttt{TRPO} on CCT5, CCT10 and CCT20, smoothed by an average over 10 different independent runs. Red is \texttt{TRPO} alone, green is \texttt{TRPO} with \texttt{CE($s_0$,TRPO)}-reset model. The solid line is the average, the light area is one standard deviation.};
\end{tikzpicture}

\clearpage

\begin{tikzpicture}[overlay] %[node distance=3cm,auto]
\label{page:COMP_trpodcl}
\tikzstyle{format} = [anchor = south]
\path (.5\columnwidth,-.5\textheight) coordinate (center);
\def\vstep{1.12cm}
\def\hstep{4cm}
\def\widthfig{1.1\linewidth}
\def\caption{0.4cm}
\path (-5cm,1.3cm) coordinate (UL);
\foreach\t in {0,...,10} \foreach\s in {0,...,20} \path (UL) + (\t*\hstep,-\s*\vstep) coordinate (C\t_\s);
%	\foreach\t in {0,...,10} \foreach\s in {0,...,6} \path (C\t_\s) node {(\t,-\s)};

\def\scale{1}
\path (C3_3) node[scale=\scale]{\includegraphics[width=\widthfig]{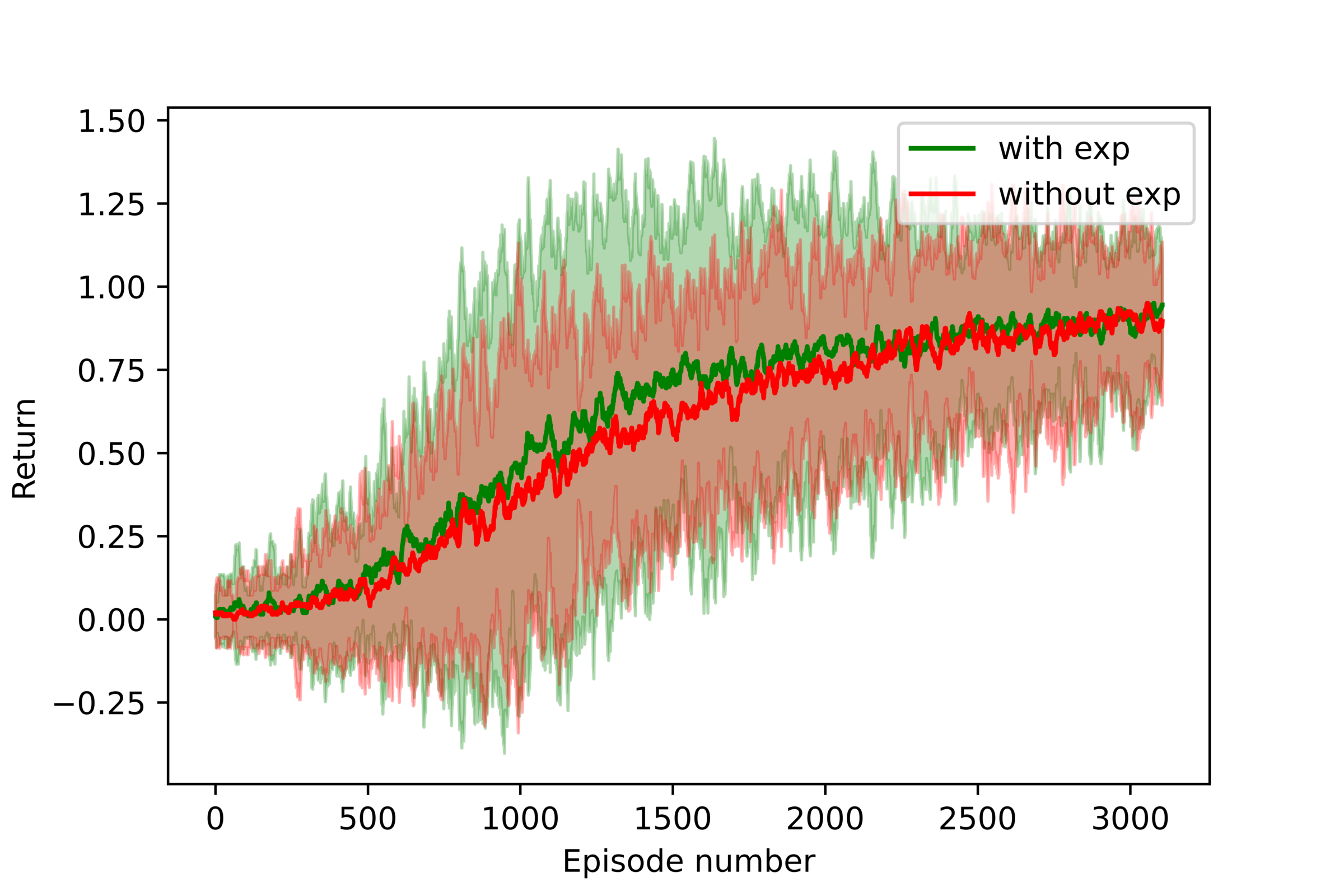}};
\path (C3_11) node[scale=\scale] (central) {\includegraphics[width=\widthfig]{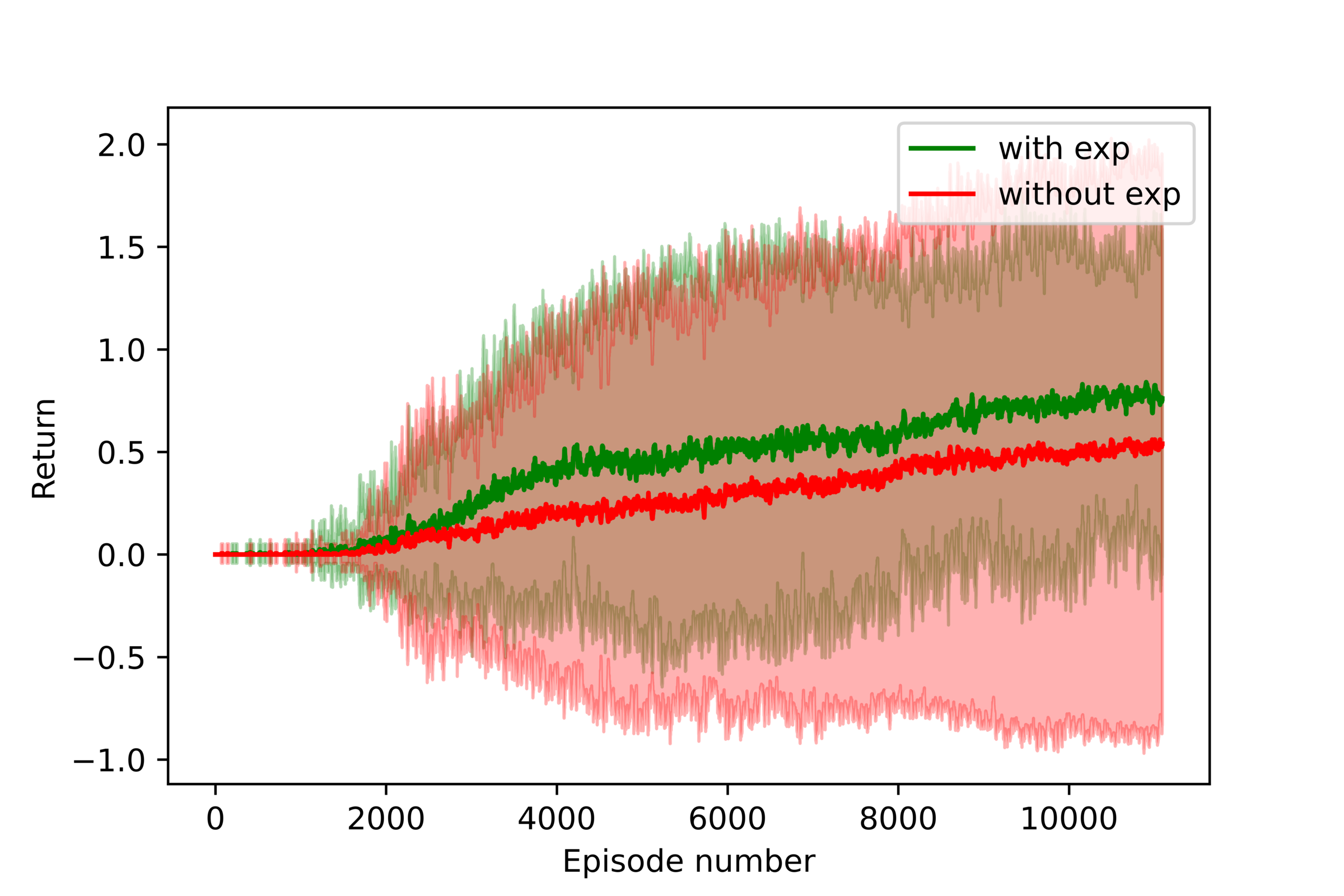}};
\path (C3_19) node[scale=\scale]{\includegraphics[width=\widthfig]{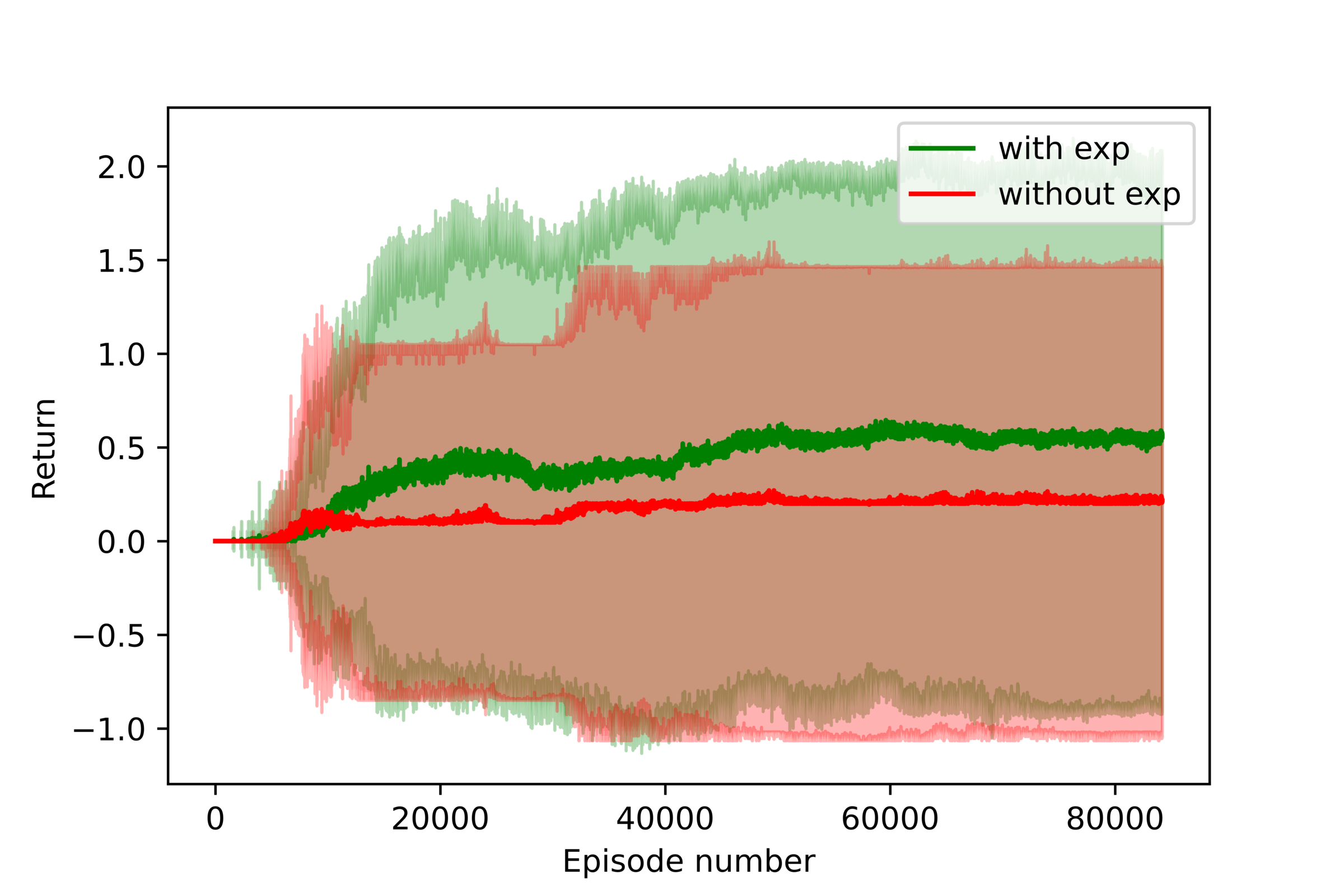}};

%\path (central.east) + (-\caption,0) node[scale=1.7,format,rotate=-90,text width=12.3cm]{Scatterplots of visit distributions of \texttt{CE($s_0$,TRPO)} on Consecutive Crossroad Traps of depths $d=5$ (upper), $d=10$ (middle), $d=20$ (lower). States $s$ on the $x$-axis, values of $\mu_n(s)$ on the $y$-axis, colors red $\rightarrow$ blue for steps $n=0\rightarrow d-1$.};
\path (central.east) + (-\caption,0) node[scale=1.7,format,rotate=-90,text width=12.5cm]{Average of the (undiscounted) return accumulated during episodes of \texttt{TRPO} on DCL5, DCL10 and DCL20, smoothed by an average over 10 different independent runs. Red is \texttt{TRPO} alone, green is \texttt{TRPO} with \texttt{CE($s_0$,TRPO)}-reset model. The solid line is the average, the light area is one standard deviation.};
\end{tikzpicture}

\clearpage

%\begin{figure}[ht] 
%	\centering
%	\includegraphics[width=0.7\linewidth]{img/D20_comp.jpg}
%	\caption{Policy value improvement with number of episodes, with and without \texttt{CE($s_0$,TRPO)}-reset. Values are averaged over $10$ different independent runs. Solid line is the average, light area is one standard deviation. In red, the value improvement when starting from $\texttt{start}$. The sparsity of the reward makes learning very difficult. In green, the value improvement when starting from $s\sim\texttt{CE(\texttt{start},TRPO)}$.}\label{fig:learning}
%\end{figure}

\bibliography{curious_explorer}
\bibliographystyle{alpha}

\end{document}